\documentclass[11pt]{article}

\usepackage{microtype}
\usepackage{graphicx}
\usepackage{subfigure} 
\usepackage{booktabs} 
\usepackage{natbib}
\usepackage{epsfig,epsf,fancybox}
\usepackage{amsmath}
\usepackage{mathrsfs}
\usepackage{amssymb}
\usepackage{color}
\usepackage{multirow}
\usepackage{paralist}
\usepackage{verbatim}
\usepackage{algorithm}
\usepackage{algorithmic}
\usepackage{galois}
\usepackage{boxedminipage}
\usepackage{accents}
\usepackage{stmaryrd}
\usepackage[bottom]{footmisc}

\usepackage[normalem]{ulem}
\usepackage{xcolor}

\usepackage{natbib}

\usepackage{url}
\usepackage[colorlinks,linkcolor=magenta,citecolor=blue, pagebackref=true,backref=true]{hyperref}
\renewcommand*{\backrefalt}[4]{%
    \ifcase #1 \footnotesize{(Not cited.)}%
    \or        \footnotesize{(Cited on page~#2.)}%
    \else      \footnotesize{(Cited on pages~#2.)}%
    \fi}

\newtheorem{theorem}{Theorem}[section]

\newtheorem{lemma}[theorem]{Lemma}
\newtheorem{proposition}[theorem]{Proposition}

\newtheorem{definition}{Definition}[section]

\newtheorem{remark}[theorem]{Remark}
\newtheorem{assumption}[theorem]{Assumption}

\usepackage{hyperref}

\newcommand{\BB}{\mathbb{B}}
\newcommand{\EE}{\mathbb{E}}

\newcommand{\st}{\textnormal{s.t.}}

\newcommand{\sign}{\textnormal{sign}}

\newcommand{\x}{\mathbf x}
\newcommand{\y}{\mathbf y}

\newcommand{\g}{\mathbf g}

\newcommand{\z}{\mathbf z}
\newcommand{\w}{\mathbf w}
\newcommand{\su}{\mathbf u}
\newcommand{\sv}{\mathbf v}

\newcommand{\DCal}{\mathcal{D}}

\newcommand{\br}{\mathbb{R}}

\newcommand{\ba}{\begin{array}}
\newcommand{\ea}{\end{array}}

\newcommand{\ZCal}{\mathcal{Z}}

\newcommand{\PCal}{\mathcal{P}}
\newcommand{\RCal}{\mathcal{R}}
\newcommand{\WCal}{\mathcal{W}}
\newcommand{\XCal}{\mathcal{X}}

\newcommand{\PP}{\mathbb{P}}

\newcommand{\NCal}{\mathcal{N}}

\newcommand{\mb}{\mathbb}

\begin{document}


\begin{center}

{\bf{\LARGE{Fast Distributionally Robust Learning with 
\vspace*{.05in}
\\ Variance Reduced Min-Max Optimization}}}

\vspace*{.2in}
{\large{
\begin{tabular}{c}
Yaodong Yu$^{\diamond, \star}$ \quad 
Tianyi Lin$^{\diamond, \star}$ \quad  
Eric Mazumdar$^{\diamond, \star}$ \quad 
Michael I. Jordan$^{\diamond, \dagger}$ \\
\end{tabular}
}}

\vspace*{.2in}

\begin{tabular}{c}
Department of Electrical Engineering and Computer Sciences$^\diamond$\\
Department of Statistics$^\dagger$ \\ 
University of California, Berkeley
\end{tabular}

\vspace*{.2in}

\today

\vspace*{.2in}

\begin{abstract}
Distributionally robust supervised learning (DRSL) is emerging as a key paradigm for building reliable machine learning systems for real-world applications---reflecting the need for classifiers and predictive models that are robust to the distribution shifts that arise from phenomena such as selection bias or nonstationarity. Existing algorithms for solving Wasserstein DRSL--- one of the most popular DRSL frameworks based around robustness to perturbations in the Wasserstein distance---have serious limitations that limit their use in large-scale problems---in particular they involve solving complex subproblems and they fail to make use of stochastic gradients. We revisit Wasserstein DRSL through the lens of min-max optimization and derive scalable and efficiently implementable stochastic extra-gradient algorithms which provably achieve faster convergence rates than existing approaches. We demonstrate their effectiveness on synthetic and real data when compared to existing DRSL approaches. Key to our results is the use of variance reduction and random reshuffling to accelerate stochastic min-max optimization, the analysis of which may be of independent interest.
\end{abstract}
\let\thefootnote\relax\footnotetext{$^\star$ Yaodong Yu, Tianyi Lin and Eric Mazumdar contributed equally to this work.}
\end{center}

\section{Introduction}
With machine learning systems increasingly being deployed in real-world settings, there is an urgent need for machine learning approaches that can adapt to --- or are robust to---changes in the environment. Despite this, the dominant paradigm for supervised learning~\citep{Hastie-2009-Elements} remains that of empirical risk minimization (ERM)~\citep{Vapnik-2013-Nature}, wherein a model is trained by minimizing a loss over a fixed set of training data. A key assumption underlying this approach is that the training data is from the same distribution as the test data --- i.e., the distribution of the data does not change between training time and deployment.

Such assumptions are known to rarely hold in practice. Indeed, the distribution may change due to selection bias, nonstationarity in the environment~\citep{Quionero-2009-Dataset}, or even adversarial perturbations~\citep{Szegedy-2013-Intriguing, Madry-2018-Towards}, leaving machine learning models trained through ERM particularly susceptible to adversarial attacks~\citep{Szegedy-2013-Intriguing, Carlini-2017-Towards} or to degraded performance from distribution shifts. 

Distributionally robust supervised learning (DRSL) seeks to address this issue by explicitly optimizing for solutions that are robust to adversarial distribution shifts. Note that DRSL is not a new terminology in the context of machine learning but has been frequently used in the existing works (see~\citet{Hu-2018-Does} for an example). Common approaches to DRSL formulate training as a zero-sum game between the machine learning model and an adversary that perturbs the training data within a ball in either an $f$-divergence~\citep{Bagnell-2005-Robust, Ben-2013-Robust, Namkoong-2016-Stochastic, Namkoong-2017-Variance, Hu-2018-Does} or the Wasserstein distance~\citep{ Gao-2017-Distributional, Sinha-2018-Certifiable, Blanchet-2019-Quantifying, Blanchet-2019-Robust, Shafieezadeh-2019-Regularization}. For an overview of distributionally robust optimization and relevant applications, we refer the reader to a recent survey~\citep{Rahimian-2019-Distributionally}. 

Despite its appealing premise, DRSL suffers in practice from a relative dearth of algorithms for solving large constrained min-max optimization problems. Indeed, existing approaches to DRSL reformulate the problem as a non-smooth convex optimization problem~\citep{Wiesemann-2014-Distributionally,Abadeh-2015-Distributionally}, and then apply optimization algorithms which are deterministic or require solving complex sub-problems which quickly become computationally intractable as the number of data points increases~\citep{Li-2019-First,Li-2020-Fast,Liu-2017-Primal,Lee-2015-Distributionally, Luo-2019-Decomposition}. 

Recent years have seen an emerging interest in developing stochastic algorithms for solving large-scale DRSL. In particular, ~\citet{Namkoong-2016-Stochastic} proposed stochastic algorithms for solving convex reformulations of DRSL with $f$-divergences. For Wasserstein DRSL, ~\citet{Blanchet-2018-Optimal}  avoided the non-smooth reformulation and proved locally strong convexity of the dual problem under certain conditions, where standard stochastic gradient descent algorithms can be applied. In practice, Wasserstein distances have been found to be more favorable than $f$-divergences since they can be used to compare distributions with disjoint supports.  Indeed since $f$ divergences can only be used to compare distributions on the same support, the end result of DRSL with $f$-divergences is a classifier that it robust to reweighting of empirical distribution. Wasserstein DRSL, on the other hand, provides robustness to an adversary who perturbs the data in more general ways.

This paper focuses on providing stochastic gradient-based algorithms for DRSL with Wasserstein metrics (known as Wasserstein DRSL or WDRSL) which can handle large datasets common in supervised learning, while being provably faster than existing approaches. Key to our approach is that we refrain from reformulating the problem as a convex minimization problem, and instead reformulate it as a smooth, finite-sum, convex-concave min-max optimization problem which can be efficiently solved with specially constructed variance-reduced stochastic extragradient algorithms.

\paragraph{Contributions.} We develop algorithms for solving WDRSL problems with generalized linear losses---a class of problems that includes WDRSL with logistic losses. Our contributions can be summarized as follows. 
\begin{enumerate}
\item We reformulate WDRSL with generalized linear losses as a constrained, smooth, convex-concave min-max optimization problem where the constraint sets for the maximization and minimization problems are the $\ell_\infty$-ball and the second-order cone, respectively. We then define the notion of optimality for the min-max formulation (cf.\ Definition~\ref{Def:opt-min-max-pred}), and show that it is consistent with the standard one (Definition~\ref{Def:opt-robust-pred}).

\item We develop two simple stochastic algorithms for solving WDRSL, which we refer to as \textit{Stochastic Extragradient with Variance Reduction} (SEVR) and \textit{Stochastic Proximal Point with Random Reshuffling} (SPPRR). By a careful complexity analysis, we prove that these two algorithms are able to exploit the smoothness and finite-sum structure of our reformulation to converge faster than the existing off-the-shelf solvers and many deterministic first-order algorithmic frameworks for solving WDRSL considered in this paper.

\item We present experimental results that demonstrate  that our proposed algorithms are computationally efficient and significantly improve upon (i) existing methods for solving WDRSL, and (ii) other min-max optimization algorithms not tailored to the problem structure on real datasets.
\end{enumerate}
\paragraph{Organization.} The paper is organized as follows. In Section~\ref{sec:related} we place our work in the context of recent advances in stochastic min-max optimization. In Section~\ref{sec:prelim}, we provide the basic setup for the WDRSL problems. In Section~\ref{subsec:model}, we prove that the WDRSL problems can be reformulated as a smooth, convex-concave and finite-sum min-max optimization problem and specify the optimality criterion. In Section~\ref{subsec:alg_SEVR} and~\ref{subsec:alg_SPPRR}, we propose and analyze the SEVR and SPPRR algorithms for solving WDRSL problems and show that both algorithms achieve fast finite-time rates. In Section~\ref{sec:experiment}, we present empirical results on both synthetic and real datasets that illustrate that our algorithms outperform off-the-shelf solvers and other competing deterministic and stochastic first-order algorithms. Finally, we conclude this paper in Section~\ref{sec:conclu} and provide detailed proofs and further numerical results in the appendix.  

\paragraph{Notation.} We use bold lower-case letters (e.g., $\x$) to denote vectors, upper-case letters (e.g., $X$) to denote matrices, and calligraphic upper case letters (e.g., $\XCal$) to denote sets. The notion $[n]$ refers to $\{1, 2, \ldots, n\}$ for some integer $n > 0$. The notions $\EE[\cdot \mid \xi]$ and $\EE[\cdot]$ refer to the expectation conditioned on the random variable $\xi$ and over all the randomness. We let $a \wedge b$ denote $\min\{a, b\}$. For a differentiable function $f: \br^d \rightarrow \br$, we let $\nabla f(\x)$ denote the gradient of $f$ at $\x$. For a vector $\x \in \br^d$, we denote $\|\x\|$ as its $\ell_2$-norm and $\langle\x, \y\rangle = x^\top y$ as the inner product between two vectors $\x, \y \in \br^d$. For a constraint set $\XCal \subseteq \br^d$, we let $D_\XCal$ denote its diameter, where $D_\XCal = \max_{\x, \x'\in \XCal} \|\x-\x'\|$ and let $\PCal_\XCal$ denote the orthogonal projection onto this set. Lastly, given the accuracy $\epsilon > 0$, the notation $a=O(b(\epsilon))$ stands for the upper bound $a \leq C \cdot b(\epsilon)$ where a constant $C>0$ is independent of $\epsilon$. Similarly, $a=\tilde{O}(b(\epsilon))$ indicates the previous inequality may depend on the logarithmic function of $\epsilon$, and where a constant $C>0$ is also independent of $\epsilon$.

\section{Related Work}\label{sec:related}
Our work comes amid a surge of interest in gradient-based algorithms for a large class of emerging min-max optimization problems ~\citep{Daskalakis-2018-Training,Mazumdar-2020-On,Mokhtari-2020-Unified}. Despite this activity, existing stochastic-gradient algorithms\footnote{For brevity, we focus on stochastic algorithms and leave deterministic algorithms~\citep[e.g.,][]{Korpelevich-1976-Extragradient, Nemirovski-2004-Prox, Nesterov-2007-Dual, Chambolle-2011-First,Mazumdar-2019-On, Liang-2019-Interaction, Thekumparampil-2019-Efficient, Mokhtari-2020-Convergence, Lin-2020-Near} out of our discussion.}---e.g., stochastic mirror-prox and its variants~\citep{Nemirovski-2009-Robust, Juditsky-2011-Solving, Chen-2014-Optimal, Chen-2017-Accelerated}, single-call stochastic extragradient~\citep{Hsieh-2019-Convergence}, and epoch-wise stochastic gradient descent ascent~\citep{Yan-2020-Optimal}---do not fully exploit the underlying structure of WDRSL and suffer from slow rates of convergence. In particular, current stochastic min-max optimization algorithms for solving finite-sum convex-concave min-max optimization problems~\citep[e.g.,][]{Balamurugan-2016-Stochastic, Shi-2017-Bregman, Iusem-2017-Extragradient, Du-2019-Linear, Chavdarova-2019-Reducing, Cui-2019-Analysis, Yang-2020-Catalyst, Kotsalis-2020-Simple} often require large batches of data at each iteration or only achieve convergence guarantees if the objective function is strongly convex-concave. The objective function in WDRSL, however, is only convex-concave, motivating us to further tailor simple variance-reduced stochastic gradient algorithms~\citep[such as, e.g.,][]{Balamurugan-2016-Stochastic, Chavdarova-2019-Reducing} to effectively take advantage of the smoothness and finite-sum structure of WDRSL problems. In particular, our algorithms make use of two schemes commonly used in optimization to handle minimization problems that can be decomposed as finite sums --- variance reduction and random reshuffling of the data. 

\paragraph{Variance reduction (VR).} Variance-reduced (VR) algorithms were originally proposed for solving finite-sum minimization problems~\citep{Johnson-2013-Accelerating, Allen-2016-Improved,Reddi-2016-Stochastic, Allen-2016-Variance, Fang-2018-Near, Zhou-2018-Stochastic} and have recently been extended to finite-sum min-max optimization and variational inequality (VI) problems~\citep{Balamurugan-2016-Stochastic, Chavdarova-2019-Reducing, Carmon-2019-Variance, Alacaoglu-2021-Stochastic}. The stochastic variance-reduced gradient (SVRG) algorithm~\citep{Balamurugan-2016-Stochastic, Chavdarova-2019-Reducing} in particular, is the focus of our analysis since remains it computationally tractable in large problems. Previous works have shown that it achieves linear convergence rates in strongly convex-concave finite-sum problems, but its convergence in  convex-concave~\citep{Xie-2020-Lower} problems remains unknown---an issue that we address in Section~\ref{subsec:alg_SEVR} by introducing a provably efficient algorithm which we call stochastic extragradient with variance reduction (SEVR). Recently, the concurrent work has appeared~\citep{Alacaoglu-2021-Stochastic} that provides a new complexity bound for stochastic variance-reduced algorithms for solving finite-sum monotone VIs and thus convex-concave min-max optimization problems. Our algorithms are different from theirs and the theoretical results in two papers complement each other.

\paragraph{Random reshuffling (RR).} Algorithms that employ random reshuffling (RR) of the data are increasingly popular among practitioners in ERM problems, but, until recently, have proven challenging to analyze theoretically. Indeed,  RR algorithms converge faster than the stochastic gradient algorithms on many problems~\citep{Bottou-2009-Curiously, Recht-2013-Parallel}, though the benefit comes at the cost of a significant complication to its theoretical analysis---the gradient estimators are now biased. Recent work in the optimization literature has established convergence rates for RR algorithms~\citep{Shamir-2016-Without, Gurbuzbalaban-2019-Random, Haochen-2019-Random, Nagaraj-2019-Sgd, Rajput-2020-Closing, Safran-2020-Good, Nguyen-2020-Unified, Mishchenko-2020-Random}. In particular, they have been shown to converge faster than stochastic gradient descent on quadratic or strongly convex objectives~\citep{Gurbuzbalaban-2019-Random, Haochen-2019-Random, Nguyen-2020-Unified, Mishchenko-2020-Random} and more recently for general smooth and convex objectives~\citep{Nagaraj-2019-Sgd}.  However, to date, random reshuffling has not been proposed for min-max optimization problems. 

In this work, we introduce a min-max optimization algorithm that uses random reshuffling and provides performance guarantees. We show both theoretically and empirically that the algorithm retains the impressive performance of RR methods when compared to other stochastic schemes in min-max optimization.

\section{Preliminaries}\label{sec:prelim}
In this section, we present a setup for the Wasserstein distributionally robust supervised learning (WDRSL) problem which covers a wide range of use cases.

\paragraph{WDRSL.} In the statistical learning literature, it is common to assume that the pairs of a data sample and its label are independent, identically distributed, and drawn from some distribution $\PP$ supported on $\br^d \times \{-1, 1\}$.  In a \textit{generalized linear problem} the parameter $\beta \in \br^d$ can be estimated by solving the following stochastic optimization problem:
\begin{equation}\label{prob:GLM}
\inf_{\beta \in \br^d} \ \left\{
\int_{\br^d \times \{-1, 1\}} \left(\Psi(\langle\x, \beta\rangle) - y\langle\x, \beta\rangle\right) \PP(d(\x, y))\right\}, 
\end{equation}
where $\Psi: \br \rightarrow \br$ is a smooth and nonlinear function, $y \in \{-1, 1\}$ denotes the response variable, $\x \in \br^d$ denotes the predictor (or covariate), and the expectation is with respect to the distribution $\PP$. Problem~\eqref{prob:GLM} is referred to a generalized linear problem in the canonical form, and covers a wide range of application problems; see~\citet{Hardin-2007-Generalized} and~\citet{Dobson-2018-Introduction} for the details. 

Distributionally robust optimization~\citep[see, e.g.,][]{Wiesemann-2014-Distributionally},  seeks to minimize the worst-case expectation of the generalized linear loss function over an ambiguity set $\PCal$, defined as a neighborhood of the distribution $\PP$. The resulting problem has the form: 
\begin{equation}\label{prob:DRSL-abstract}
\inf_{\beta \in \br^d} \sup_{\PP \in \PCal} \EE\left[\Psi(\langle\x, \beta\rangle) - y\langle\x, \beta\rangle\right]. 
\end{equation}
The set $\PCal$ can be thought of as the set of distributions in which an adversary could choose the data distribution, and it is common for it to be defined through moment constraints~\citep{Delage-2010-Distributionally} or with metrics between probability distributions~\citep{Namkoong-2016-Stochastic, Namkoong-2017-Variance, Gao-2017-Distributional, Esfahani-2018-Data, Blanchet-2019-Quantifying, Blanchet-2019-Robust}. 

In practice, since we only have access to samples from $\PP$, the ambiguity set is often defined as a neighborhood of the empirical distribution $\widehat{\PP}_n$, with the most popular characterization making use of the Wasserstein distance between probability distributions~\citep{Abadeh-2015-Distributionally, Gao-2017-Distributional, Sinha-2018-Certifiable, Blanchet-2019-Quantifying, Blanchet-2019-Robust, Shafieezadeh-2019-Regularization} due to computational tractability and theoretical properties~\citep{Esfahani-2018-Data}. 
\begin{definition}[Wasserstein distance]
Let $\mu$ and $\nu$ be two probability distributions supported on $\ZCal=\br^d \times \{-1, 1\}$ and let $\Pi(\mu, \nu)$ denote the set of all couplings (joint distributions) between $\mu$ and $\nu$. The Wasserstein distance between $\mu$ and $\nu$ is defined by 
\begin{equation*}
\WCal(\mu, \nu) = \inf_{\pi \in \Pi(\mu, \nu)} \int_{\ZCal \times \ZCal} c(\z, \z') \pi(d\z, d\z'),  
\end{equation*}
where $\z = (\x, y) \in \br^d \times \{-1, 1\}$ and $c(\cdot, \cdot)$ is a well-defined metric on $\br^d \times \{-1, 1\}$. 
\end{definition}
This results in the \emph{Wasserstein distributionally robust generalized linear problem}, which takes the form of Eq.~\eqref{prob:WDRSL-abstract}, where the ambiguity set is defined as the $\delta$-ball in the Wasserstein distance centered at the empirical distribution $\widehat{\PP}_n$, $\PCal=\BB_\delta(\widehat{\PP}_n)=\{\PP \mid \WCal(\PP, \widehat{\PP}_n) \leq \delta\}$:
\begin{equation}\label{prob:WDRSL-abstract}
\inf_{\beta \in {\mathbb{R}}^d}  \sup_{\PP \in \BB_\delta(\widehat{\PP}_n)} \EE^\PP\left[\Psi(\langle\x, \beta\rangle) - y\langle\x, \beta\rangle\right]. 
\end{equation}
We remark that if $\delta=0$ (no perturbation), the problem in Eq.~\eqref{prob:WDRSL-abstract} is equivalent to solving the classical empirical risk minimization problem.  

\section{Main Results}
In this section, we begin by deriving a new structured min-max optimization formulation of WDRSL with generalized linear losses. We then develop two simple stochastic iterative algorithms that exploit this structure to provably achieve fast finite-time rates. We defer proofs to the supplementary materials.

\subsection{WDRSL as Structured Min-max Optimization}\label{subsec:model}
Before stating our main results, we recall some basic definitions for smooth functions.
\begin{definition}
A function $f$ is $L$-Lipschitz if for $\forall \x, \x'$: $|f(\x)-f(\x')| \leq L\|\x-\x'\|$.
\end{definition}
\begin{definition}
A function $f$ is $\ell$-smooth if for $\forall \x, \x'$: $\|\nabla f(\x)-\nabla f(\x')\| \leq \ell\|\x-\x'\|$.
\end{definition}
We also require a notion of $\epsilon$-optimality for the WDRSL problem described in Eq.~\eqref{prob:WDRSL-abstract}. 
\begin{definition}\label{Def:opt-robust-pred}
A  point $\widehat{\beta}$ is an $\epsilon$-optimal solution ($\epsilon \geq 0$) of Eq.~\eqref{prob:WDRSL-abstract} if 
\begin{equation*}
\sup_{\PP \in \BB_\delta(\widehat{\PP}_n)} \EE^\PP\left[\Psi(\langle\x, \widehat{\beta}\rangle) - y\langle\x, \widehat{\beta}\rangle\right] - \inf_{\beta \in \br^d} \sup_{\PP \in \BB_\delta(\widehat{\PP}_n)} \EE^\PP\left[\Psi(\langle\x, \beta\rangle) - y\langle\x, \beta\rangle\right] \leq \epsilon. 
\end{equation*}
If $\epsilon=0$, then $\widehat{\beta}$ is an optimal solution of WDRSL. 
\end{definition}
Definition~\ref{Def:opt-robust-pred} applies, inter alia, to the solutions to WDRSL with logistic loss functions~\citep{Abadeh-2015-Distributionally}, though in this paper, we consider a more general setting and make the following assumption throughout.   
\begin{assumption}\label{Assumption:main}
The function $\Psi$ is convex, $L$-Lipschitz and $\ell$-smooth for some $L, \ell > 0$ and each data point $(\x_i, y_i)$ satisfies $\|\x_i\| \leq 1$ for all $i \in [n]$. 
\end{assumption}
\begin{remark}
Assumption~\ref{Assumption:main} is mild and covers the WDRSL with logistic and generalized logistic functions~\citep{Stukel-1988-Generalized, Aljarrah-2020-Generalized}. In particular, \textbf{the second condition is only assumed for the simplicity and can be relaxed to $\|\x_i\| \leq G$ for all $i \in [n]$}. The Lipschitz objective function is necessary and is standard for analyzing stochastic algorithms either without access to full gradients~\citep{Nemirovski-2009-Robust} or with random reshuffling~\citep{Nagaraj-2019-Sgd}.   
\end{remark}
Given these definitions, we now show that the WDRSL problem defined in Eq.~\eqref{prob:WDRSL-abstract} is equivalent to a structured min-max optimization problem under Assumption~\ref{Assumption:main} by~\citet{Abadeh-2015-Distributionally}. We provide the proofs in our work for completeness; see Appendix~\ref{app:model} for the details on the min-max reformulation. We then provide the explicit form of our min-max optimization reformulation of the WDRSL problem\footnote{We refer to~\citet{Gao-2017-Distributional} and~\citet{Shafieezadeh-2019-Regularization} for more general loss functions in WDRSL.}. 

\paragraph{Min-max reformulation.} It is worth mentioning that the linear model $\Psi(\langle \x, \cdot\rangle)$ is generally assumed in the context of distributionally robust optimization. Under Assumption~\ref{Assumption:main} and for an empirical distribution $\widehat{\PP}_n = (1/n)\sum_{i=1}^n \delta_{(\widehat{\x}_i, \widehat{y}_i)}$, the WDRSL problem in Eq.~\eqref{prob:WDRSL-abstract} is equivalent to the following structured min-max optimization model:
\begin{equation}\label{prob:WDRSL-min-max}
\begin{aligned} 
\min_{(\lambda, \beta) \in \br^d\times \br} \max_{\gamma \in \mb{R}^n}  \quad &\big\{\lambda(\delta - \kappa) +  \frac{1}{n}\sum_{i=1}^n \Psi(\langle\widehat{\x}_i, \beta\rangle) +  \frac{1}{n}\sum_{i=1}^n \gamma_i\left(\widehat{y}_i\langle\widehat{\x}_i, \beta\rangle - \lambda\kappa\right)\big\}  \\
\st &\quad \|\beta\| \leq \lambda/(L+1), \ \|\gamma\|_{\infty} \leq 1,
\end{aligned}
\end{equation}
where $\kappa>0$ is a positive constant associated with the metric $c(\z, \z')=\|\x-\x'\| + \kappa|y-y'|$, for points $\z=(\x, y)$ and $\z'=(\x', y')$ in  $ \br^d$.
\begin{remark}
We remark that the min-max model in Eq.~\eqref{prob:WDRSL-min-max} has nice properties: (i) It enjoys a finite-sum structure, (ii) The objective function is Lipschitz, smooth, convex in $(\lambda, \beta) \in \br \times \br^d$ and concave in $\gamma \in \br^n$, and (iii) The constraint sets are in the form of second-order cones and the $\ell_\infty$-ball, such that the orthogonal projections can be computed analytically. For the general finite-sum convex-concave optimization,~\citet{Alacaoglu-2021-Stochastic} provided a generic algorithm with finite-time convergence guarantee. Compared to their algorithm, our algorithms achieve better performance when $n$ is huge and $\epsilon > 0$ is medium. 
\end{remark}
\begin{remark}
We assume that the existence of an optimal saddle point $\su^\star=(\lambda^\star, \beta^\star, \gamma^\star) \in \Lambda \times \Gamma$ in Eq.~\eqref{prob:WDRSL-min-max} and do not need to assume that this saddle point $\su^\star \in \Lambda \times \Gamma$ is unique. Nonetheless, the objective function in any optimal saddle point is the same due to Sion's min-max theorem~\citep{Sion-1958-General}. 
\end{remark}
Given our formulation we can define an $\epsilon$-optimal solution for solving the WDRSL in Eq.~\eqref{prob:WDRSL-min-max} and show that it is consistent with the optimality condition in Definition~\ref{Def:opt-robust-pred}. To do so, we denote $\su = (\lambda, \beta, \gamma)$ and let $L(\su)$ be the objective function of the smooth min-max optimization model in Eq.~\eqref{prob:WDRSL-min-max} with an optimal saddle point $\su^\star=(\lambda^\star, \beta^\star, \gamma^\star) \in \Lambda \times \Gamma$. We define the operator $F: \br \times \br^d \times \br^n \rightarrow \br \times \br^d \times \br^n$ as:
\begin{equation}\label{Def:operator-main}
F(\su) = \begin{pmatrix} \nabla_{(\lambda, \beta)} L(\su) \\ - \nabla_\gamma L(\su) \end{pmatrix},
\end{equation}
where:
\begin{eqnarray*}
\nabla_{(\lambda, \beta)} L(\su) &= & 
\begin{pmatrix}
\delta - \kappa\left(1 + \frac{1}{n}\sum_{i=1}^n \gamma_i\right) \\
\frac{1}{n}\sum_{i=1}^n \Psi'(\langle\widehat{\x}_i, \beta\rangle)\widehat{\x}_i + \frac{1}{n}\sum_{i=1}^n \gamma_i\widehat{y}_i\widehat{\x}_i \end{pmatrix},
\\
\nabla_\gamma L(\su) & = & \begin{pmatrix}
\frac{1}{n}(\widehat{y}_1\langle\widehat{\x}_1, \beta\rangle - \lambda\kappa) \\ 
\vdots \\
\frac{1}{n}(\widehat{y}_n\langle\widehat{\x}_n, \beta\rangle - \lambda\kappa) \end{pmatrix}.
\end{eqnarray*}
Finally, we write the constraint sets as $\Lambda=\{(\lambda, \beta) \in \br \times \br^d \mid \|\beta\| \leq \lambda/(L+1)\}$ and $\Gamma=\{\gamma \in \br^n \mid \|\gamma\|_\infty \leq 1\}$. 
Given this notation, we define the $\epsilon$-optimal solution to the min-max formulation of the WDRSL problem.
\begin{definition}\label{Def:opt-min-max-pred}
A point $\widehat{\su} = (\widehat{\lambda}, \widehat{\beta}, \widehat{\gamma})$ is an $\epsilon$-optimal saddle-point solution of the WDRSL in Eq.~\eqref{prob:WDRSL-min-max} if $\widehat{\su} \in \Lambda \times \Gamma$ and $\Delta(\widehat{\su}) = L(\widehat{\lambda}, \widehat{\beta}, \gamma^\star) - L(\lambda^\star, \beta^\star, \widehat{\gamma}) \leq \epsilon$. 
\end{definition}
Note that Definition~\ref{Def:opt-min-max-pred} measures the optimality via appeal to the duality gap which is different from the objective function gap in optimization. Before introducing our algorithms for solving this problem, we first show that the two definitions of optimality, Definition~\ref{Def:opt-min-max-pred} and Definition~\ref{Def:opt-robust-pred}, are equivalent.
\begin{proposition}\label{Prop:equivalence}
For sufficiently small $\epsilon>0$, a point $\widehat{\beta}$ is an $\epsilon$-optimal solution of the WDRSL in Eq.~\eqref{prob:WDRSL-abstract} if a point $\widehat{\su} = (\widehat{\lambda}, \widehat{\beta}, \widehat{\gamma})$ is an $\epsilon$-optimal saddle-point solution of the WDRSL in Eq.~\eqref{prob:WDRSL-min-max} for some $\widehat{\lambda}$ and $\widehat{\gamma}$. Conversely, there exist $\widehat{\lambda}$ and $\widehat{\gamma}$ such that a point $\widehat{\su} = (\widehat{\lambda}, \widehat{\beta}, \widehat{\gamma})$ is an $\epsilon$-optimal saddle-point solution of the WDRSL in Eq.~\eqref{prob:WDRSL-min-max} if a point $\widehat{\beta}$ is an an $\epsilon$-optimal solution of the WDRSL in Eq.~\eqref{prob:WDRSL-abstract}. 
\end{proposition}
Proposition~\ref{Prop:equivalence} shows that we can solve the original WDRSL problem posed in Eq.~\eqref{prob:WDRSL-abstract} by applying min-max optimization algorithms to the min-max reformulation described in Eq.~\eqref{prob:WDRSL-min-max}. We remark that this result does not hold in general, and is a consequence of special structure of the WDRSL problem in Eq.~\eqref{prob:WDRSL-min-max}. Most importantly, the gap function $\Delta(\widehat{\su})$ allows for a finite-time analysis of stochastic min-max optimization algorithms with variance reduction and random reshuffling.   
\begin{remark}
It is important to remark the difference between the desired accuracy $\epsilon > 0$ and the perturbation level $\delta > 0$.  Indeed, we use $\epsilon$ to characterize the optimality of the optimization problems. Furthermore, we can increase $\delta$ for $\WCal(\PP, \widehat{\PP}_n) \leq \delta$ defined in Eq.~\eqref{prob:WDRSL-abstract} in order to improve model robustness, which will not affect the computation of an $\epsilon$-optimal saddle point solution of Eq.~\eqref{prob:WDRSL-min-max} when $\epsilon$ is sufficiently small. Given an $\epsilon$-optimal saddle point solution, we can convert it to an $\epsilon$-optimal solution of Eq.~\eqref{prob:WDRSL-abstract}. 
\end{remark}

\subsection{Extragradient Meets Variance Reduction}\label{subsec:alg_SEVR}
We present a simple iterative algorithm, \textit{Stochastic Extragradient with Variance Reduction} (SEVR), for solving the min-max reformulation of the WDRSL problem, the pseudo-code of which is presented in Algorithm~\ref{Algorithm:SEVR}. For simplicity, we denote the component operator $F_i: \br \times \br^d \times \br^n \rightarrow \br \times \br^d \times \br^n$ by
\begin{equation}\label{Def:operator-component}
F_i(\su) = 
\begin{pmatrix}
\delta - \kappa\left(1 + \gamma_i\right) \\
\Psi'(\langle\widehat{\x}_i, \beta\rangle)\widehat{\x}_i + \gamma_i\widehat{y}_i\widehat{\x}_i \\
0 \\ 
\vdots \\ 
-(\widehat{y}_i\langle\widehat{\x}_i, \beta\rangle - \lambda\kappa) \\
\vdots \\ 
0 
\end{pmatrix}.   
\end{equation}  
In SEVR, we start each epoch $s$ of the SEVR algorithm with a reference point $\tilde{\su}^s$ and compute $F(\tilde{\su}^s)$ (cf. Eq.~\eqref{Def:operator-main}). We then proceed with $k_s$ projected extragradient updates of the form:
\begin{align}
\g_t^s &= F(\tilde{\su}^s) + (F_{i_t}(\su_t^s) - F_{i_t}(\tilde{\su}^s)), \nonumber\\
\bar{\su}_{t+1}^s &= \PCal_{\Lambda \times \Gamma}(\su_t^s - \eta_{t+1}^s\g_t^s), \label{Update:SEVR-extrapolation} \\
\bar{\g}_t^s &= F(\tilde{\su}^s) + (F_{j_t}(\bar{\su}_{t+1}^s) - F_{j_t}(\tilde{\su}^s)), \nonumber\\
\su_{t+1}^s &= \PCal_{\Lambda \times \Gamma}(\su_t^s - \eta_{t+1}^s\bar{\g}_t^s), \label{Update:SEVR-main} 
\end{align}
and set the next reference point $\tilde{\su}^{s+1}$ to be the average of the iterates in the epoch: $(1/k_s)\sum_{t=1}^{k_s} \su_t^s$. 
\begin{algorithm}[!t]
\begin{algorithmic}
\STATE \textbf{Input:} $\tilde{\su}^0 = \su_0^0 \in \br \times \br^d \times \br^n$, epoch length $k_0$, step size $\eta>0$ and the number of epochs $S$.  
\STATE \textbf{Initialization:} $l \leftarrow 0$ and $T \leftarrow k_02^S - k_0$.    
\FOR{$s=0,1,\ldots, S-1$}
\STATE Compute and store $F(\tilde{\su}^s)$. 
\STATE $k_s \leftarrow k_0 2^s$. 
\FOR{$t = 0, 1, \ldots, k_s - 1$}
\STATE Set $l \leftarrow l+1$ and $\eta_{t+1}^s \leftarrow \frac{\eta\sqrt{T}}{\sqrt{2T-l}}$.  
\STATE Uniformly sample with replacement two indices $i_t, j_t \in [n]$. Note that $i_t$ and $j_t$ are independent and identically distributed, and update $\su_{t+1}^s$ by Eq.~\eqref{Update:SEVR-extrapolation} and Eq.~\eqref{Update:SEVR-main}. 
\ENDFOR
\STATE $\tilde{\su}^{s+1} \leftarrow \frac{1}{k_s}\sum_{t=1}^{k_s} \su_t^s$.
\STATE $\su_0^{s+1} \leftarrow \su_{k_s}^s$.  
\ENDFOR
\STATE \textbf{Output:} $\tilde{\su}^S$. 
\end{algorithmic} \caption{Stochastic Extragradient with Variance Reduction (SEVR)} \label{Algorithm:SEVR}
\end{algorithm}

SEVR is inspired by the SVRG++ algorithm~\citep{Allen-2016-Improved} to min-max optimization problems.  It is related to the variance-reduced extragradient algorithms for saddle point problems (SVRG) proposed by ~\citet{Balamurugan-2016-Stochastic} and the stochastic variance-reduced extragradient (SVRE) algorithm proposed by~\citet{Chavdarova-2019-Reducing}. A key difference between our approach, SEVR, and both SVRG and SVRE however, is that in SVRG and SVRE, the number of inner loops is a constant or is distributed according to a geometric random variable, while the proposed SEVR algorithm performs exponentially more inner loops as the iterates approach the optimal set. This yields improved complexity bounds by reducing the number of times the full gradient is computed. Further, both SVRG and SVRE have only been analyzed in the context of strongly convex-concave min-max problems while in the following theorem we provide complexity bounds for SEVR under the weaker assumption of convex-concave functions.
\begin{theorem}\label{Theorem:SEVR}
Let $L(\su)$ be the objective function for the smooth min-max optimization problem in Eq.~\eqref{prob:WDRSL-min-max} with an optimal saddle point $\su^\star=(\lambda^\star, \beta^\star, \gamma^\star) \in \Lambda \times \Gamma$. Under the conditions of Assumption~\ref{Assumption:main}, assume that the initial vector $\tilde{\su}^0 = \su_0^0$ satisfies $\|\tilde{\su}^0 - \su^\star\| \leq D_\su$ and $\Delta(\tilde{\su}^{0}) \leq D_L$ for parameters $D_\su, D_L \geq 0$, and let $\epsilon \in (0, 1)$ denote a desired accuracy. Then for a number of epochs $S$, initial epoch length $k_0$, and step size $\eta$, satisfying:
\begin{align*}
S &= \,O\left(\log_2\left(\frac{D_L}{\epsilon}\right)\right), \quad k_0 = \, \frac{D_\su^2}{\eta D_L},  \\
\eta &= \, O\left(\min\left\{\frac{1}{\ell+\kappa+1}, \frac{\epsilon}{(\ell+\kappa+1)^2D_\su^2}, \frac{D_\su^2}{D_L}\right\}\right),
\end{align*}
the iterates generated by the SEVR algorithm (cf.\ Algorithm~\ref{Algorithm:SEVR}) satisfy $\EE[\Delta(\tilde{\su}^{S})] \leq \epsilon$. Furthermore, the total number of gradient evaluations $N$ satisfies:
\begin{equation*}
N=O\left(n\log\left(\frac{1}{\epsilon}\right) + \frac{D_L}{\epsilon} + \frac{(\ell+\kappa+1)^2D_\su^4}{\epsilon^2}\right). 
\end{equation*}
\end{theorem}
The proof of~\ref{Theorem:SEVR} can be found in Appendix~\ref{app:alg_SEVR}. Through Theorem~\ref{Theorem:SEVR}, we immediately observe the benefits of our approach over deterministic algorithms for WDRSL that achieve complexity bounds of $O(n/\epsilon)$~\citep{Liu-2017-Primal, Zhang-2019-Efficient, Li-2019-First}. Indeed, our results show that SEVR achieves a better dependence on the number of samples $n$, in that the total number of gradient evaluations to find an $\epsilon$-optimal stationary point is $O(n\log(1/\epsilon))$.\footnote{Recent work, \cite{Alacaoglu-2021-Stochastic}, proved rates for a different stochastic variance-reduced algorithm under similar assumptions and derived a bound of $O(n + \sqrt{n}/\epsilon)$ which has a worse dependence on $\epsilon$ and $n$ than that in Theorem~\ref{Theorem:SEVR} (when $n$ is sufficiently large).}  We also remark that SEVR improves over general-purpose interior-point methods~\citep{Lee-2015-Distributionally, Luo-2019-Decomposition} which can be used to solve convex reformulations of WDRSL. Such methods have per-iteration costs at least quadratic in the problem dimension $d$ while the per-iteration complexity of SEVR is only linear in $d$. 
\begin{remark}
We remark that in contrast to many stochastic gradient algorithms the SEVR algorithm benefits from a sequence of adaptive yet \emph{non-decaying} step sizes $\{\eta_t^s\}_{s \geq 0, t \geq 0}$, which is important in practice when a very large number of iterations is commonly desired. 
\end{remark}
\begin{remark}
Algorithm~\ref{Algorithm:SEVR} is not a straightforward extension of~\citet{Allen-2016-Improved}; indeed, it is an extragradient-based algorithm whereas the algorithm in~\citet{Allen-2016-Improved} is an variant of gradient descent (GD). The gradient descent ascent, which extends GD to min-max optimization, might diverge or converges to limit circles in convex-concave setting. To analyze our algorithms, we construct a different potential function based on the duality gap, which is unnecessary in~\citet{Allen-2016-Improved}. 
\end{remark}

\subsection{Proximal Point Meets Random Reshuffling}\label{subsec:alg_SPPRR}
\begin{algorithm}[!t]
\begin{algorithmic}
\STATE \textbf{Input:} $\su_0^0 \in \br \times \br^d \times \br^n$, step size $\eta>0$ and the number of epochs $S$.
\STATE \textbf{Initialization:} 
\FOR{$s=0,1,\ldots, S-1$} 
\STATE $\sigma^s \leftarrow$ a random permutation of the set $[n]$. 
\FOR{$t = 0, 1, \ldots, n - 1$}
\STATE $\su_{0, t}^s \leftarrow \su_t^s$. 
\FOR{$i = 0, 1, \ldots, M - 1$} 
\STATE Update the iterate $\su_{i+1, t}^s$ by Eq.~\eqref{Update:SPPRR-main}.
\ENDFOR 
\STATE $\su_{t+1}^s \leftarrow \su_{M, t}^s$. 
\ENDFOR
\STATE $\su_0^{s+1} \leftarrow \su_n^s$.  
\ENDFOR
\STATE \textbf{Output:} $\tilde{\su}^S = \frac{1}{nS}\sum_{s=0}^{S-1}\sum_{t=1}^n \su_t^s$. 
\end{algorithmic} \caption{Stochastic Proximal Point with Random Reshuffling (SPPRR)}\label{Algorithm:SPPRR}
\end{algorithm} 
\noindent
The second algorithm we propose, \textit{Stochastic Proximal Point with Random Reshuffling} (SPPRR), uses the idea of random reshuffling. In each epoch $s$ of the SPPRR algorithm, we sample a random permutation $\{\sigma_0, \sigma_1, \ldots, \sigma_{n-1}\}$ of the set $[n]$ and proceed with $n$ inexact proximal point updates of the form: $\su_{t+1}^s \ \approx \ \PCal_{\Lambda \times \Gamma}(\su_t^s - \eta F_{ \sigma_t^s}(\su_{t+1}^s))$, where $\eta>0$ is a step size and $\sigma = \sigma_t^s$. We then set $\su_0^{s+1}=\su^s_n$ and repeat the process for a total of $S$ epochs. Note that a new permutation/shuffling is only generated at the beginning of each epoch; see Algorithm~\ref{Algorithm:SPPRR}. 

To compute the inexact proximal point updates, we first note that the exact proximal point update is equivalent to the solution of the fixed-point problem $\su = T(\su)$, where the operator $T(\su)$ is defined as $T(\su):= \PCal_{\Lambda \times \Gamma}(\su_t^s - \eta F_{\sigma_t}(\su))$. If $\eta \leq 1/(2(\ell+\kappa+1))$, it is easy to verify that $T$ is a contraction. Thus, if we let $\su_{0, t}^s = \su_t^s$, we can define our inexact proximal point update to be the result of $M$ applications of the operator $T$ to $\su_{0, t}^s$ , which takes the following iterative form for $i=1,...,M$: 
\begin{equation}\label{Update:SPPRR-main}
\su_{i+1, t}^s \ = \ \PCal_{\Lambda \times \Gamma}(\su_t^s - \eta F_{\sigma_t^s}(\su_{i, t}^s)).   
\end{equation}
We set $\su_{t+1}^s = \su_{M, t}^s$ and proceed to the next inexact proximal point update. In our theoretical treatment we set $M = O(\log(1/\epsilon))$ to guarantee that $\su_{i+1, t}^s$ is sufficiently close to the fixed point, while in practice it suffices to have $M=2$ to observe fast convergence. 

To our knowledge, the SPPRR algorithm is the first stochastic first-order algorithm for solving the min-max optimization problem with random reshuffling. The following theorem shows that the complexity bound of our algorithm is independent of the number of samples $n$ up to a logarithmic factor.    

\begin{theorem}\label{Theorem:SPPRR}
Let $L(\su)$ be the objective function of the smooth min-max optimization problem in Eq.~\eqref{prob:WDRSL-min-max} with an optimal saddle point $\su^\star=(\lambda^\star, \beta^\star, \gamma^\star) \in \Lambda \times \Gamma$, and define $D_\su=\|\su_0^0 - \su^\star\|$. Under the conditions of Assumption~\ref{Assumption:main}, if each $F_i$ is bounded over $\Lambda \times \Gamma$ for all $i \in [n]$, then for a desired accuracy level $\epsilon \in (0, 1)$, number of epochs $S$, step size $\eta$, and number of fixed point-iterations $M$ that satisfy:
\begin{align*}
S  &= \, O\left(\frac{1}{n}\left(\frac{(\ell+\kappa+1)D_\su^2}{\epsilon} + \frac{G^2D_\su^2}{\epsilon^2}\right)\right), \\
\eta  &=  \, \min\left\{\frac{1}{2(\ell+\kappa+1)}, \frac{\epsilon}{4G^2}\right\}, \,\,
M  = \, 1 + \left\lfloor \log_2(10nS)\right\rfloor,
\end{align*} 
the iterates generated by the SPPRR algorithm satisfy $\EE[\Delta(\tilde{\su}^{S})] \leq \epsilon$. Further, the total number of gradient evaluations, $N$ satisfies:
\begin{equation*} N=O\left(\left(\frac{(\ell+\kappa+1)D_\su^2}{\epsilon} + \frac{G^2D_\su^2}{\epsilon^2}\right)\log\left(\frac{1}{\epsilon}\right)\right). 
\end{equation*}
\end{theorem}
Our results extend convergence guarantees for random reshuffling for SGD on smooth, bounded functions~\citep{Nagaraj-2019-Sgd} to smooth convex-concave min-max optimization regime. Note that $N$ does not depend on $n$; indeed, the random reshuffling algorithms can be viewed as stochastic gradient-based algorithms without replacement and the complexity bound is independent of $n$. Such result was known for minimization and our work extends it to min-max optimization. 
\begin{remark}
We remark that the boundedness of the component functions is a standard assumption for proving the convergence of stochastic gradient algorithms with random reshuffling in minimization problems~\citep{Shamir-2016-Without, Haochen-2019-Random, Nagaraj-2019-Sgd, Rajput-2020-Closing, Safran-2020-Good, Nguyen-2020-Unified, Mishchenko-2020-Random}. In the min-max regime, this translates to requiring the boundedness of each $F_i$ over the constraint set $\Lambda \times \Gamma$. 
\end{remark}
\begin{remark}
It is worth emphasizing the importance of random reshuffling-based algorithms in the min-max optimization; indeed, while the random reshuffling-based algorithms generally do not improve the convergence rate in convex optimization, they are extremely efficient in practice. This motivates us to consider their extension to min-max optimization and the impressive empirical performance confirm our conjecture.
\end{remark}
\begin{remark}
Algorithm~\ref{Algorithm:SPPRR} is not a straightforward extension of~\citet{Nagaraj-2019-Sgd}; indeed, it is a proximal point-based algorithm whereas the algorithm in~\citet{Nagaraj-2019-Sgd} is an variant of gradient descent (GD). The necessity of fixed-point sub-problem comes from the generality of min-max optimization. More specifically, ~\citet{Nagaraj-2019-Sgd} used the existence of \textit{convex objective functions} for proving their Lemma 2 and 4 which can not be extended to convex-concave min-max optimization problems in a obvious way. We resolve this issue by changing the subproblem to the proximal subproblem and solving the resulting fixed-point problem.  
\end{remark}
\begin{figure*}[!t]
\centering
\includegraphics[width=0.99\textwidth]{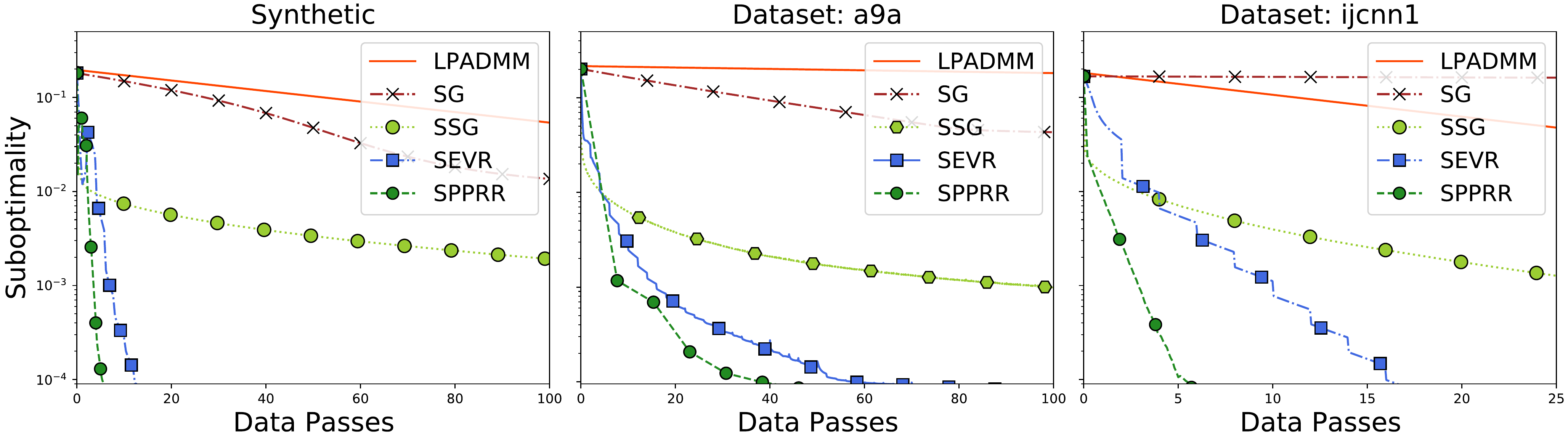}
\caption{\footnotesize{Comparison of SEVR and SPPRR with (minimization) baseline methods for solving convex reformulations of distributionally robust logistic regression on synthetic datasets and LIBSVM datasets. The horizontal axis represents the number of data passes, and the vertical axis represents the suboptimality $(f(\lambda^{(t)}, \beta^{(t)}) - f^{\star})$. We remark that LPADMM takes multiple data passes in every iteration as it needs to solve a nonlinear constrained problem in each iteration, which accounts for its poor performance.}}\label{fig:exp-comparison-with-min-baseline}
\end{figure*}


\begin{figure*}[t]
\centering
\includegraphics[width=0.99\textwidth]{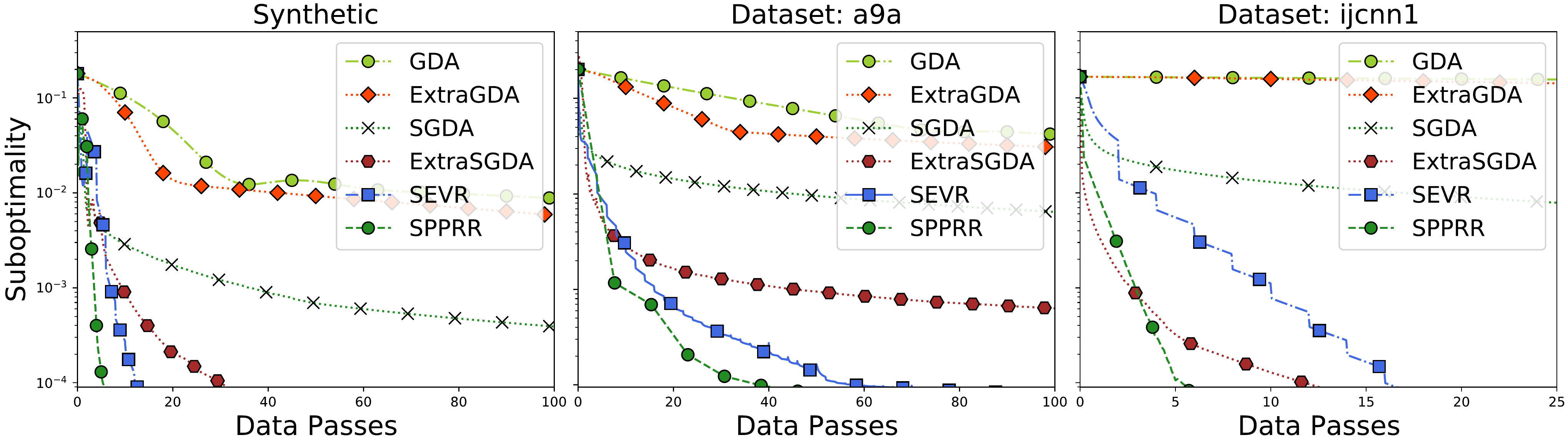}
\caption{\footnotesize{Comparison of SEVR and SPPRR with min-max baseline methods for solving distributionally robust logistic regression on synthetic datasets and LIBSVM datasets. The horizontal axis represents the number of data passes, and the vertical axis represents the suboptimality $(f(\lambda^{(t)}, \beta^{(t)}) - f^{\star})$.}}
\label{fig:exp-comparison-with-minmax-baseline}
\end{figure*}

\section{Experiments}\label{sec:experiment}
In this section, we empirically validate Algorithm~\ref{Algorithm:SEVR} and Algorithm~\ref{Algorithm:SPPRR}. Here we focus on distributionally robust logistic regression problems~\citep{Abadeh-2015-Distributionally} with Wasserstein-2 distance. Our experiments are implemented using Python 3.7 on a computer with a 2.6 GHz Intel Core i7 and 16GB of memory. We refer to Appendix~\ref{sec:appendix-exp} for further implementation details and more 
experimental results.

\paragraph{Baseline methods.} We compare our proposed algorithms, SEVR and SPPRR, against two classes of algorithms.\footnote{We exclude off-the-shelf solvers since~\citet{Li-2019-First} showed that LPADMM outperforms these solvers.} The first (Figure~\ref{fig:exp-comparison-with-min-baseline}) comprises algorithms that have been previously proposed for solving convex reformulations of WDRSL. In particular, we benchmark against linearized proximal alternating direction method of multipliers~(LPADMM)~\citep{Li-2019-First}, sub-gradient methods~(SG), and the stochastic sub-gradient method~(SSG). The second class (Figure~\ref{fig:exp-comparison-with-minmax-baseline}) is made up of general-purpose stochastic first-order algorithms for min-max optimization: gradient descent ascent~(GDA), extragradient descent ascent method~(ExtraGDA), stochastic gradient descent ascent~(SGDA), and single-call stochastic extragradient descent ascent~(ExtraSGDA)~\citep{Hsieh-2019-Convergence} to highlight the benefit of our algorithms over existing  general purpose min-max algorithms.
\begin{figure*}[t]
\centering
\subfigure[Effect of batch size on performance of SEVR ]{\includegraphics[width=0.495\textwidth]{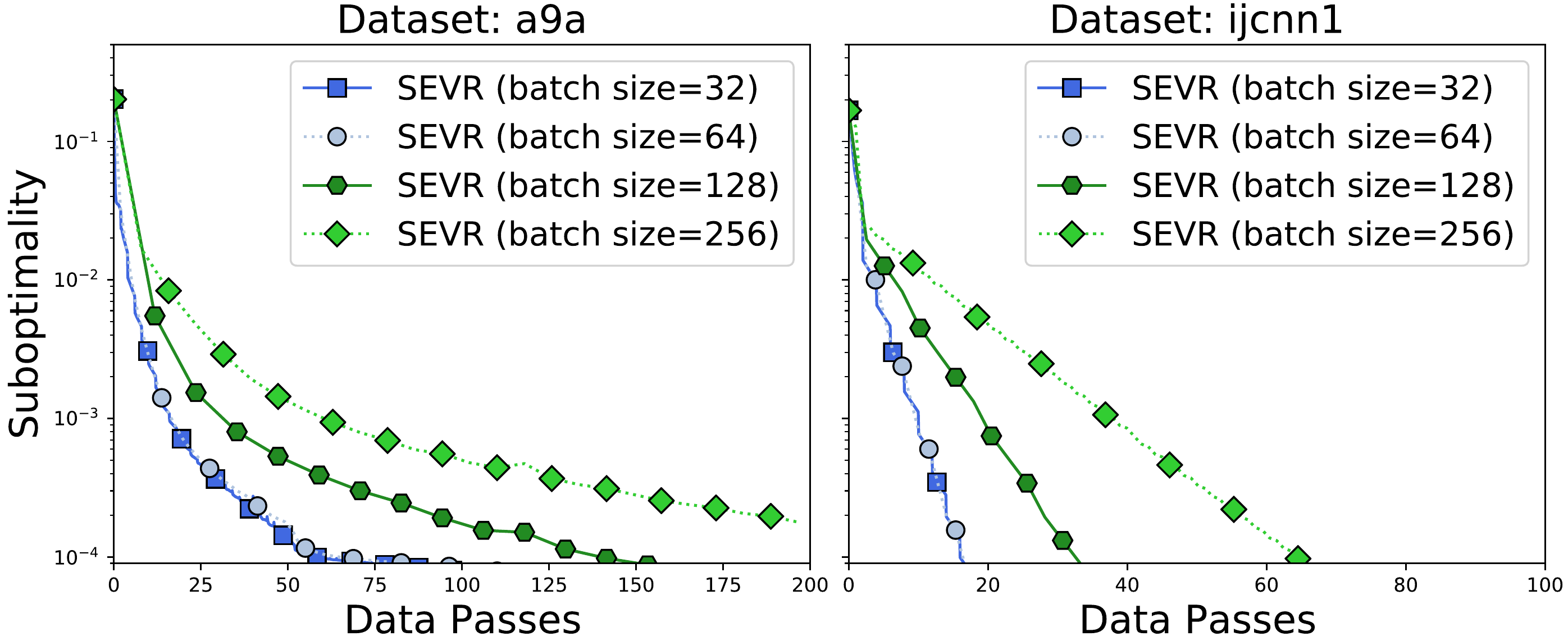}}
\subfigure[Number of fixed-point iterations for SPPRR]{\includegraphics[width=0.495\textwidth]{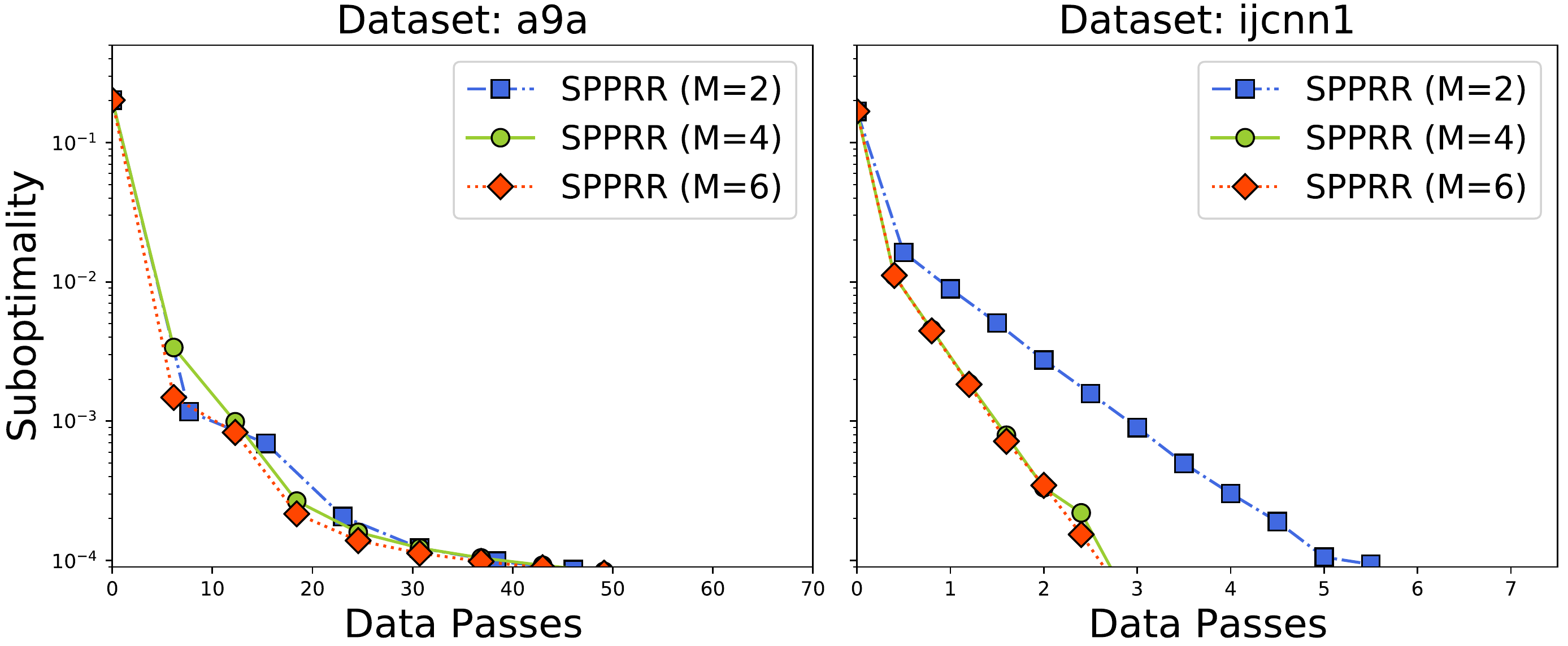}}
\caption{(\textbf{a})\,Comparing SEVR across batch sizes $B \in \{32, 64, 128, 256\}$ for solving WDRSL with logistic losses. (\textbf{b})\,Comparing SPPRR with a different number of fixed-point iterations $M \in \{2, 4, 6\}$ for solving WDRSL with logistic losses.}\label{fig:exp-comparison-batchsize-fix-point-iteration-M}
\end{figure*}
\paragraph{Experimental results.} For both the synthetic and real datasets, we find that both SEVR and SPPRR converge significantly faster than the baseline methods, especially when the  training dataset is large. In particular, we make two key observations: (1) As shown in Figure~\ref{fig:exp-comparison-with-min-baseline} our proposed algorithms show major improvements over existing algorithms for solving convex reformulations of WDRSL, and (2) Our algorithms improve upon existing stochastic gradient-based algorithms for min-max optimization (Figure~\ref{fig:exp-comparison-with-minmax-baseline}). This validates our approach of viewing WDRSL through the lens of min-max optimization and tailoring algorithms to the structure of a min-max reformulation. 

In both Figures~\ref{fig:exp-comparison-with-min-baseline} and~\ref{fig:exp-comparison-with-minmax-baseline}, we observe that SPPRR is able to outperform all the baselines across all datasets. This shows that random reshuffling maintains its impressive empirical performance even in min-max problems and highlights the need for a deeper theoretical understanding of the technique. Indeed, SPPRR is able to achieve $10^{-3}$ suboptimality within 20 data passes on all datasets, which corroborates our theoretical results on how its convergence rate is independent of the size of the dataset. We remark that ExtraSGDA achieves a faster convergence rate compared with other baselines, outperforming even SEVR on the a9a dataset. A possible explanation is that (1) stochastic gradients have low variance in early epochs, and (2) ExtraSGDA only performs one projection at each iteration while SEVR requires two.  In any case, we note that SPPRR outperforms ExtraSGDA across the board.

In Figure~\ref{fig:exp-comparison-batchsize-fix-point-iteration-M}(a) and Figure~\ref{fig:exp-comparison-batchsize-fix-point-iteration-M}(b), we present further numerical experiments showing that the fast convergence of SEVR remains robust to different batch sizes and that the performance of SPPRR also does not vary significantly with more fixed point iterations (suggesting that $M=2$ is an efficient default setting for the algorithm). Additionally, we also compare WDRSL and DRSL with $f$-divergence in Section~\ref{sec:robust-loss-compare-appendix}, where we verify that WDRSL achieves better performance than DRSL with $f$-divergence when evaluated on worst-case loss on the test dataset (within the ambiguity set $\PCal$ induced by the Wasserstein distance $\WCal$). Further numerical experiments are presented in Appendix~\ref{sec:appendix-exp}.

\section{Conclusions}\label{sec:conclu}
We studied Wasserstein distributionally robust supervised learning through the lens of min-max optimization. We first showed that WDRSL with generalized linear functions can be reformulated as constrained min-max optimization problems and proposed two simple and efficient algorithms which exploit the smoothness and finite-sum structure of the reformulation. We showed that the proposed stochastic min-max optimization algorithms are guaranteed to converge to the $\epsilon$-optimal solution of the WDRSL at fast rates. Experiments on synthetic and real datasets validate our theoretical findings and demonstrate the scalability and effectiveness of our approach to WDRSL. Indeed, they indicate that min-max formulations of robust optimization problems can be solved with significantly less computational overhead than standard optimization approaches based on convex reformulations.

\section*{Acknowledgements}
We would like to thank the AC and five reviewers for suggestions that improve the quality of this paper. This work was supported in part by the Mathematical Data Science program of the Office of Naval Research under grant number N00014-18-1-2764.

\bibliographystyle{plainnat}
\bibliography{ref}

\newpage\appendix
\section{Postponed Proofs in Subsection~\ref{subsec:model}}\label{app:model}
Before presenting the main theorem, we provide a key technical lemma which is a straightforward generalization of~\citet[Lemma~1]{Abadeh-2015-Distributionally}. The proof is based on a simple yet nontrivial generalization and we present all the details for the sake of completeness. 
\begin{lemma}\label{Lemma:formulation-constraint}
Under Assumption~\ref{Assumption:main} and let $(\widehat{\x}, \widehat{y}) \in \br^d \times \{-1, 1\}$ be a given pair of data sample. Then, for every $\lambda > 0$, we have
\begin{equation*}
\sup_{\x \in \br^d} \Psi(\langle\x, \beta\rangle) - \widehat{y}\langle\x, \beta\rangle - \lambda\|\widehat{\x} - \x\| = \left\{
\begin{array}{ll}
\Psi(\langle\widehat{\x}, \beta\rangle) - \widehat{y}\langle\widehat{\x}, \beta\rangle, & \textnormal{if } \|\beta\| \leq \lambda/(L+1), \\
-\infty, & \textnormal{otherwise}. 
\end{array}\right.  
\end{equation*}
\end{lemma}
\begin{proof}
Note that $\Psi(\langle\x, \beta\rangle) - \widehat{y}\langle\x, \beta\rangle - \lambda\|\widehat{\x} - \x\|$ is the difference of two different convex functions in $\x$. In order to maximize this function, we define its convex part as $h_\beta(\x)=\Psi(\langle\x, \beta\rangle)-\widehat{y}\langle\x, \beta\rangle$ and reformulate it using the conjugate function. 

We first consider one-dimensional function $f(t):=\Psi(t)-\widehat{y} t$. Since $\Psi$ is convex, the conjugate function of $f$ is well defined and expressed by $f^\star(\theta)=\sup_{t \in \br} \{\theta t - \Psi(t) + \widehat{y} t\}$. Then, we claim that $f^\star(\theta)=+\infty$ when $|\theta| > L+1$. Since $\Psi$ is $L$-Lipschitz (cf. Assumption~\ref{Assumption:main}), we have
\begin{equation*}
|\Psi(t) - \Psi(0)| = \left|\int_0^t \Psi'(s) \; ds\right| \leq L|t| \ \Longrightarrow \ \Psi(t) \leq \Psi(0) + L|t|. 
\end{equation*}
This implies that $f^*(\theta) \geq \sup_{t \in \br} \{\theta t - L|t| - \Psi(0) + \widehat{y} t\}$. Since $\widehat{y} \in \{-1, 1\}$, we have
\begin{equation*}
\theta t - L|t| - \Psi(0) + \widehat{y} t \longrightarrow \left\{\begin{array}{ll}
+\infty \textnormal{ as } t \rightarrow +\infty, & \textnormal{if } \theta > L + 1, \\ +\infty \textnormal{ as } t \rightarrow -\infty, & \textnormal{if } \theta < -L-1.\end{array}\right. 
\end{equation*} 
Putting these pieces together yields the desired result. Moreover, the conjugate of $h_\beta(\x)=\Psi(\langle\x, \beta\rangle) - \widehat{y}\langle\x, \beta\rangle$ is therefore given by 
\begin{equation*}
h_\beta^\star(\xi) = \left\{\begin{array}{ll}
\inf\limits_{|\theta| \leq L+1} f^\star(\theta), & \textnormal{if } \xi = \theta\beta, \\ +\infty, & \textnormal{otherwise}. 
\end{array} \right. 
\end{equation*} 
Since $h_\beta(\x)$ is convex and continuous, it is equivalent to its bi-conjugate function~\citep{Rockafellar-2015-Convex}. This implies the statement on the representation of the function $h_\beta$ as follows,  
\begin{equation}\label{inequality-formulation-constraint}
h_\beta(\x) = \sup_{\xi \in \br^d} \langle \xi, \x\rangle - h_\beta^\star(\xi) = \sup_{|\theta| \leq L+1} \langle \theta\beta, \x\rangle - f^\star(\theta). 
\end{equation}
The remaining steps are the same as that in~\citet[Lemma~1]{Abadeh-2015-Distributionally}. More specifically, Eq.~\eqref{inequality-formulation-constraint} implies that $h_\beta$ can be represented as the upper envelope of infinitely many linear functions. Using this representation, we have
\begin{eqnarray*}
\sup_{\x \in \br^d} \Psi(\langle\x, \beta\rangle) - \widehat{y}\langle\x, \beta\rangle - \lambda\|\widehat{\x} - \x\| & = & \sup_{\x \in \br^d}\sup_{|\theta| \leq L+1} \ \langle \theta\beta, \x\rangle - f^\star(\theta) - \lambda\|\widehat{\x} - \x\| \\
& & \hspace{-10em} = \ \sup_{\x \in \br^d}\sup_{|\theta| \leq L+1} \ \langle \theta\beta, \x\rangle - f^\star(\theta) - \sup_{\|q\| \leq \lambda} \langle q, \widehat{\x} - \x\rangle \\
& & \hspace{-10em} = \ \sup_{\x \in \br^d}\sup_{|\theta| \leq L+1}\inf_{\|q\| \leq \lambda} \ \langle \theta\beta, \x\rangle - f^\star(\theta) - \langle q, \widehat{\x} - \x\rangle. 
\end{eqnarray*}
By~\citet[Proposition~5.5.4]{Bertsekas-2009-Convex}, we have
\begin{eqnarray*}
\sup_{\x \in \br^d} \left\{\Psi(\langle\x, \beta\rangle) - \widehat{y}\langle\x, \beta\rangle - \lambda\|\widehat{\x} - \x\|\right\} & = & \sup_{|\theta| \leq L+1}\inf_{\|q\| \leq \lambda}\sup_{\x \in \br^d} \ \langle \theta\beta, \x\rangle - f^\star(\theta) - \langle q, \widehat{\x} - \x\rangle \\
& & \hspace{-10em} = \ \sup_{|\theta| \leq L+1}\inf_{\|q\| \leq \lambda}\sup_{\x \in \br^d} \ \langle \theta\beta + q, \x\rangle - f^\star(\theta) - \langle q, \widehat{\x}\rangle.  
\end{eqnarray*}
Explicitly evaluating the maximization over $\x \in \br^d$, we have
\begin{eqnarray*}
\sup_{\x \in \br^d} \left\{\Psi(\langle\x, \beta\rangle) - \widehat{y}\langle\x, \beta\rangle - \lambda\|\widehat{\x} - \x\|\right\} & = & \left\{\begin{array}{cl}
\sup\limits_{|\theta| \leq L+1}\inf\limits_{\|q\| \leq \lambda} & - f^\star(\theta) - \langle q, \widehat{\x}\rangle \\
\st & \theta\beta + q = 0. \end{array}\right. \\
& & \hspace{-10em} = \ \left\{\begin{array}{cl}
\sup\limits_{|\theta| \leq L+1} - f^\star(\theta) + \langle \theta\beta, \widehat{\x}\rangle, & \textnormal{if } \sup\limits_{|\theta| \leq L+1} \|\theta\beta\| \leq \lambda, \\
+\infty, & \textnormal{otherwise}. \end{array}\right.  
\end{eqnarray*}
In addition, we have
\begin{eqnarray*}
\sup\limits_{|\theta| \leq L+1} - f^\star(\theta) + \langle \theta\beta, \widehat{\x}\rangle & = & \Psi(\langle\widehat{\x}, \beta\rangle) - \widehat{y}\langle\widehat{\x}, \beta\rangle, \\
\sup\limits_{|\theta| \leq L+1} \|\theta\beta\| \leq \lambda & \Leftrightarrow & \|\beta\| \leq \frac{\lambda}{L+1}. 
\end{eqnarray*}
Putting these pieces together yields the desired claim. 
\end{proof}

\subsection{Proof of Min-max Reformulation of WDRSL Problems}
{
We first present the formal statement of the min-max reformulation of WDRSL problems in Theorem~\ref{Theorem:DROGL-min-max}.
\begin{theorem}\label{Theorem:DROGL-min-max}
Under Assumption~\ref{Assumption:main} and for an empirical distribution $\widehat{\PP}_n = (1/n)\sum_{i=1}^n \delta_{(\widehat{\x}_i, \widehat{y}_i)}$, the WDRSL problem in Eq.~\eqref{prob:WDRSL-abstract} is equivalent to the following structured min-max optimization model:
\begin{equation}\label{prob:WDRSL-min-max-appendix}
\begin{aligned} 
\min_{(\lambda, \beta) \in \br^d\times \br} \max_{\gamma \in \mb{R}^n}  \quad &\left\{\lambda(\delta - \kappa) +  \frac{1}{n}\sum_{i=1}^n \Psi(\langle\widehat{\x}_i, \beta\rangle) +  \frac{1}{n}\sum_{i=1}^n \gamma_i\left(\widehat{y}_i\langle\widehat{\x}_i, \beta\rangle - \lambda\kappa\right)\right\}  \\
\st &\quad \|\beta\| \leq \lambda/(L+1), \ \|\gamma\|_{\infty} \leq 1,
\end{aligned}
\end{equation}
where $\kappa>0$ is a positive constant associated with the metric $c(\z, \z')=\|\x-\x'\| + \kappa|y-y'|$, for points $\z=(\x, y)$ and $\z'=(\x', y')$ in  $ \mathbb{R}^d$.
\end{theorem}
}

The first part of the proof is the same as that of~\citet[Theorem~1]{Abadeh-2015-Distributionally} and we provide the details for the sake of completeness. Let us denote $\z = (\x, y) \in \ZCal = \br^d \times \{0, 1\}$ and use the shorthand $h_\beta(\z) = \Psi(\langle\x, \beta\rangle) - y\langle\x, \beta\rangle$. By the definition of the Wasserstein distance and the metric $c(\cdot, \cdot)$, we have
\begin{equation*}
\sup_{\PP \in \BB_\delta(\widehat{\PP}_n)} \EE^\PP\left[\Psi(\langle\x, \beta\rangle) - y\langle\x, \beta\rangle\right] = \sup_{\PP \in \BB_\delta(\widehat{\PP}_n)} \int_\ZCal h_\beta(\z)\; \PP(d\z) = \left\{
\begin{array}{cl}
\sup\limits_{\pi \in \Pi(\PP, \widehat{\PP}_n)} & \int_\ZCal h_\beta(\z)\; \pi(d\z, \ZCal) \\
\st & \int_{\ZCal \times \ZCal} \|\z - \z'\| \; \pi(d\z, d\z') \leq \delta. 
\end{array}\right. 
\end{equation*}
Since the marginal distribution $\widehat{\PP}_n$ of $\z'$ is discrete, the coupling $\pi$ is completely determined by the conditional distribution $\PP^i$ of $\z$ given $\z'=\widehat{\z}_i=(\widehat{\x}_i, \widehat{y}_i)$ for all $i \in [n]$. Thus, we have
\begin{equation*}
\pi(d\z, d\z') = \frac{1}{n}\sum_{i=1}^n \delta_{(\widehat{\x}_i, \widehat{y}_i)}(d\z')\PP^i(d\z). 
\end{equation*}
Putting these pieces together yields that 
\begin{equation*}
\sup_{\PP \in \BB_\delta(\widehat{\PP}_n)} \EE^\PP\left[\Psi(\langle\x, \beta\rangle) - y\langle\x, \beta\rangle\right] = \left\{
\begin{array}{cl}
\sup\limits_{\PP^i} & \frac{1}{n}\sum_{i=1}^n \int_\ZCal h_\beta(\z)\; \PP^i(d\z) \\
\st & \frac{1}{n}\sum_{i=1}^n \int_\ZCal \|\z-\widehat{\z}_i\| \; \PP^i(d\z) \leq \delta, \\
& \int_\ZCal \PP^i(d\z) = 1.  
\end{array}\right. 
\end{equation*}
By replacing $\z$ with $(\x, y)$ and decomposing each distribution $\PP^i$ into unnormalized measures $\PP^i_{\pm 1}(d\x) = \PP^i(d\x, \{y = \pm 1\})$ supported on $\br^d$, the right-hand expression simplifies to
\begin{equation*}
\left\{
\begin{array}{cl}
\sup\limits_{\PP_{\pm 1}^i} & \frac{1}{n}\sum_{i=1}^n \int_{\br^d} h_\beta(\x, +1)\; \PP_{+1}^i(d\x) + \frac{1}{n}\sum_{i=1}^n \int_{\br^d} h_\beta(\x, -1)\; \PP_{-1}^i(d\x) \\
\st & \frac{1}{n}\sum_{i=1}^n \int_{\br^d} \|(\x, +1) - (\widehat{\x}_i, \widehat{y}_i)\| \; \PP_{+1}^i(d\x) + \frac{1}{n}\sum_{i=1}^n \int_{\br^d} \|(\x, -1) - (\widehat{\x}_i, \widehat{y}_i)\| \; \PP_{-1}^i(d\x) \leq \delta, \\
& \int_{\br^d} \PP_{+1}^i(d\x) + \int_{\br^d} \PP_{-1}^i(d\x) = 1.  
\end{array}\right. 
\end{equation*}
For the inequality constraint, we further split the sum over $i \in [n]$ into two groups: $\{\widehat{y}_i=+1\}$ and $\{\widehat{y}_i=-1\}$. Then we have
\begin{eqnarray*}
\delta & \geq & \frac{1}{n}\sum_{i=1}^n \int_{\br^d} \|(\x, +1) - (\widehat{\x}_i, \widehat{y}_i)\| \; \PP_{+1}^i(d\x) + \frac{1}{n}\sum_{i=1}^n \int_{\br^d} \|(\x, -1) - (\widehat{\x}_i, \widehat{y}_i)\| \; \PP_{-1}^i(d\x) \\ 
& = & \frac{1}{n} \int_{\br^d} \sum_{\widehat{y}_i=+1} \left[\|\x - \widehat{\x}_i\| \; \PP_{+1}^i(d\x) + \|\x - \widehat{\x}_i\| \; \PP_{-1}^i(d\x) + 2\kappa\PP_{-1}^i(d\x)\right] \\ 
& & + \frac{1}{n} \int_{\br^d} \sum_{\widehat{y}_i=-1} \left[\|\x - \widehat{\x}_i\| \; \PP_{-1}^i(d\x) + \|\x - \widehat{\x}_i\| \; \PP_{+1}^i(d\x) + 2\kappa\PP_{+1}^i(d\x)\right] \\
& = & \frac{2\kappa}{n} \int_{\br^d} \sum_{\widehat{y}_i=+1} \PP_{-1}^i(d\x) + \sum_{\widehat{y}_i=-1} \PP_{+1}^i(d\x) + \frac{1}{n} \int_{\br^d} \sum_{i=1}^n \|\x - \widehat{\x}_i\|\left(\PP_{-1}^i(d\x) + \PP_{+1}^i(d\x)\right).  
\end{eqnarray*}
To this end, we conclude that 
\begin{eqnarray*}
& & \sup_{\PP \in \BB_\delta(\widehat{\PP}_n)} \EE^\PP\left[\Psi(\langle\x, \beta\rangle) - y\langle\x, \beta\rangle\right] \\
& = & \left\{
\begin{array}{cl}
\sup\limits_{\PP_{\pm 1}^i} & \frac{1}{n}\sum_{i=1}^n \int_{\br^d} h_\beta(\x, +1)\; \PP_{+1}^i(d\x) + \frac{1}{n}\sum_{i=1}^n \int_{\br^d} h_\beta(\x, -1)\; \PP_{-1}^i(d\x) \\
\st & \frac{2\kappa}{n} \int_{\br^d} \sum_{\widehat{y}_i=+1} \PP_{-1}^i(d\x) + \sum_{\widehat{y}_i=-1} \PP_{+1}^i(d\x) + \frac{1}{n} \int_{\br^d} \sum_{i=1}^n \|\x - \widehat{\x}_i\|(\PP_{-1}^i(d\x) + \PP_{+1}^i(d\x)) \leq \delta, \\
& \int_{\br^d} \PP_{+1}^i(d\x) + \int_{\br^d} \PP_{-1}^i(d\x) = 1.  
\end{array}\right..
\end{eqnarray*}
The above infinite-dimensional optimization problem over the measures $\PP_{\pm 1}^i$ admits the semi-infinite dual and strong duality holds for any $\delta > 0$ due to~\citet[Proposition~3.4]{Shapiro-2001-Duality}. That is to say, the following statement holds true, 
\begin{equation*}
\sup_{\PP \in \BB_\delta(\widehat{\PP}_n)} \EE^\PP\left[\Psi(\langle\x, \beta\rangle) - y\langle\x, \beta\rangle\right] = \left\{
\begin{array}{cl}
\inf\limits_{\lambda, s_i} & \lambda\delta + \frac{1}{n}\sum_{i=1}^n s_i \\
\st & \sup_{\x \in \br^d} h_\beta(\x, +1) - \lambda\|\widehat{\x}_i - \x\| - \lambda\kappa(1-\widehat{y}_i) \leq s_i, \quad \forall i \in [n], \\
& \sup_{\x \in \br^d} h_\beta(\x, -1) - \lambda\|\widehat{\x}_i - \x\| - \lambda\kappa(1+\widehat{y}_i) \leq s_i, \quad \forall i \in [n], \\
& \lambda \geq 0.  
\end{array}\right.
\end{equation*}
Recall that $h_\beta(\x, y) = \Psi(\langle\x, \beta\rangle) - y\langle\x, \beta\rangle$. Then Lemma~\ref{Lemma:formulation-constraint} with $\widehat{y}=\pm 1$ implies that 
\begin{equation*}
\sup_{\x \in \br^d} h_\beta(\x, \pm 1) - \lambda\|\widehat{\x}_i - \x\| = h_\beta(\widehat{\x}_i, \pm 1) = \left\{ 
\begin{array}{ll}
\Psi(\langle\widehat{\x}_i, \beta\rangle) \mp \langle\widehat{\x}_i, \beta\rangle, & \textnormal{if } \|\beta\| \leq \frac{\lambda}{L+1}, \\
+\infty & \textnormal{otherwise}.
\end{array}\right. 
\end{equation*}
Putting these pieces together yields that  
\begin{equation}\label{inequality:DROGL-minimax-first}
\inf_{\beta \in \br^d} \sup_{\PP \in \BB_\delta(\widehat{\PP}_n)} \EE^\PP\left[\Psi(\langle\x, \beta\rangle) - y\langle\x, \beta\rangle\right] = \left\{
\begin{array}{cl}
\min\limits_{\beta, \lambda, s_i} & \lambda\delta + \frac{1}{n}\sum_{i=1}^n s_i \\
\st & \Psi(\langle\widehat{\x}_i, \beta\rangle) - \widehat{y}_i\langle\widehat{\x}_i, \beta\rangle \leq s_i, \quad \forall i \in [n], \\
& \Psi(\langle\widehat{\x}_i, \beta\rangle) + \widehat{y}_i\langle\widehat{\x}_i, \beta\rangle - 2\lambda\kappa \leq s_i, \quad \forall i \in [n], \\
& \|\beta\| \leq \frac{\lambda}{L+1}.  
\end{array}\right.
\end{equation}
This is a generalization of the optimization problem in~\citet[Eq. (7)]{Abadeh-2015-Distributionally}. It is equivalent to an unconstrained convex but nonsmooth optimization over the second-order cone as follows, 
\begin{eqnarray}\label{inequality:DROGL-minimax-second}
& \min\limits_{(\lambda, \beta) \in \br \times \br^d} & \lambda\delta + \frac{1}{n}\sum_{i=1}^n \Psi(\langle\widehat{\x}_i, \beta\rangle) - \widehat{y}_i\langle\widehat{\x}_i, \beta\rangle + \frac{1}{n}\sum_{i=1}^n \max\left\{0, 2\widehat{y}_i\langle\widehat{\x}_i, \beta\rangle - 2\lambda\kappa\right\}, \\
& \st & \|\beta\| \leq \lambda/(L+1). \nonumber
\end{eqnarray} 
Given a simple transformation given by 
\begin{equation*}
\max\{0, a\} = \frac{1}{2}(a + |a|) = \frac{1}{2}\left(a + \max_{|b| \leq 1} ab\right), 
\end{equation*}
we have
\begin{equation}\label{inequality:DROGL-minimax-third}
\max\left\{0, 2\widehat{y}_i\langle\widehat{\x}_i, \beta\rangle - 2\lambda\kappa\right\} = \widehat{y}_i\langle\widehat{\x}_i, \beta\rangle - \lambda\kappa + \max_{|\gamma_i| \leq 1} \gamma_i\left(\widehat{y}_i\langle\widehat{\x}_i, \beta\rangle - \lambda\kappa\right). 
\end{equation}
Therefore, we conclude that the WDRSL problem in Eq.~\eqref{prob:WDRSL-abstract} is equivalent to the structured minimax optimization model in Eq.~\eqref{prob:WDRSL-min-max}. This completes the proof.

\subsection{Proof of Proposition~\ref{Prop:equivalence}}  
We first prove that, for sufficiently small $\epsilon>0$, a point $\widehat{\beta}$ is an $\epsilon$-optimal solution of the WDRSL in Eq.~\eqref{prob:WDRSL-abstract} if a point $\widehat{\su} = (\widehat{\lambda}, \widehat{\beta}, \widehat{\gamma})$ is an $\epsilon$-optimal saddle-point solution of the WDRSL in Eq.~\eqref{prob:WDRSL-min-max}. 

Without loss of generality, let $\epsilon \rightarrow 0$, we have $\widehat{\su} \rightarrow \su^\star$. Note that the term $\widehat{y}_i\langle\widehat{\x}_i, \beta\rangle - \lambda\kappa$ is a continuous function of $(\lambda, \beta)$. Thus, we have 
\begin{equation*}
\widehat{y}_i\langle\widehat{\x}_i, \widehat{\beta}\rangle - \widehat{\lambda}\kappa \ \rightarrow \ \widehat{y}_i\langle\widehat{\x}_i, \beta^\star\rangle - \lambda^\star\kappa. 
\end{equation*}
This implies that, for sufficiently small $\epsilon>0$, we have
\begin{equation}\label{inequality:equivalence-first}
\sign(\widehat{y}_i\langle\widehat{\x}_i, \widehat{\beta}\rangle - \widehat{\lambda}\kappa) \ = \ \sign\left(\widehat{y}_i\langle\widehat{\x}_i, \beta^\star\rangle - \lambda^\star\kappa\right). 
\end{equation}
Recall that the function $L(\su)$ is defined by 
\begin{equation*}
L(\su) = \lambda(\delta - \kappa) + \frac{1}{n}\sum_{i=1}^n \Psi(\langle\widehat{\x}_i, \beta\rangle) + \frac{1}{n}\sum_{i=1}^n \gamma_i\left(\widehat{y}_i\langle\widehat{\x}_i, \beta\rangle - \lambda\kappa\right). 
\end{equation*}
This implies that $\gamma^\star(\lambda, \beta)$ with the $i$-th entry $\gamma_i^\star(\lambda, \beta) = \sign(\widehat{y}_i\langle\widehat{\x}_i, \beta\rangle - \lambda\kappa)$ for $\forall i \in [n]$ is an unique solution that maximizes the concave function $L(\lambda, \beta, \cdot)$. Putting these pieces together with Eq.~\eqref{inequality:equivalence-first} implies that $\gamma^\star(\widehat{\lambda}, \widehat{\beta}) = \gamma^\star(\lambda^\star, \beta^\star)=\gamma^\star$. Therefore, we have
\begin{equation}
\Delta(\widehat{\su}) \leq \epsilon \ \Longrightarrow \ L(\widehat{\lambda}, \widehat{\beta}, \gamma^\star(\widehat{\lambda}, \widehat{\beta})) - L(\lambda^\star, \beta^\star, \gamma^\star) \leq \epsilon. 
\end{equation}
After some simple calculations (cf. Eq.~\eqref{inequality:DROGL-minimax-second} and Eq.~\eqref{inequality:DROGL-minimax-third}), we have $\|\widehat{\beta}\| \leq \widehat{\lambda}/(L+1)$ and 
\begin{eqnarray*}
& & \widehat{\lambda}\delta + \frac{1}{n}\sum_{i=1}^n \Psi(\langle\widehat{\x}_i, \widehat{\beta}\rangle) - \widehat{y}_i\langle\widehat{\x}_i, \widehat{\beta}\rangle + \frac{1}{n}\sum_{i=1}^n \max\left\{0, 2\widehat{y}_i\langle\widehat{\x}_i, \widehat{\beta}\rangle - 2\widehat{\lambda}\kappa\right\} \\
& \leq & \lambda^\star\delta + \frac{1}{n}\sum_{i=1}^n \Psi(\langle\widehat{\x}_i, \beta^\star\rangle) - \widehat{y}_i\langle\widehat{\x}_i, \beta^\star\rangle + \frac{1}{n}\sum_{i=1}^n \max\left\{0, 2\widehat{y}_i\langle\widehat{\x}_i, \beta^\star\rangle - 2\lambda^\star\kappa\right\} + \epsilon.   
\end{eqnarray*}
The primal-dual reformulation (cf. Eq.~\eqref{inequality:DROGL-minimax-first}) further implies that 
\begin{equation*}
\sup_{\PP \in \BB_\delta(\widehat{\PP}_n)} \EE^\PP\left[\Psi(\langle\x, \widehat{\beta}\rangle) - y\langle\x, \widehat{\beta}\rangle\right] \ \leq \ \widehat{\lambda}\delta + \frac{1}{n}\sum_{i=1}^n \Psi(\langle\widehat{\x}_i, \widehat{\beta}\rangle) - \widehat{y}_i\langle\widehat{\x}_i, \widehat{\beta}\rangle + \frac{1}{n}\sum_{i=1}^n \max\left\{0, 2\widehat{y}_i\langle\widehat{\x}_i, \widehat{\beta}\rangle - 2\widehat{\lambda}\kappa\right\},  
\end{equation*}
and 
\begin{equation*}
\inf_{\beta \in \br^d} \sup_{\PP \in \BB_\delta(\widehat{\PP}_n)} \EE^\PP\left[\Psi(\langle\x, \beta\rangle) - y\langle\x, \beta\rangle\right] \ = \ \lambda^\star\delta + \frac{1}{n}\sum_{i=1}^n \Psi(\langle\widehat{\x}_i, \beta^\star\rangle) - \widehat{y}_i\langle\widehat{\x}_i, \beta^\star\rangle + \frac{1}{n}\sum_{i=1}^n \max\left\{0, 2\widehat{y}_i\langle\widehat{\x}_i, \beta^\star\rangle - 2\lambda^\star\kappa\right\}. 
\end{equation*}
Putting these pieces together yields that 
\begin{equation*}
\sup_{\PP \in \BB_\delta(\widehat{\PP}_n)} \EE^\PP\left[\Psi(\langle\x, \widehat{\beta}\rangle) - y\langle\x, \widehat{\beta}\rangle\right] - \inf_{\beta \in \br^d} \sup_{\PP \in \BB_\delta(\widehat{\PP}_n)} \EE^\PP\left[\Psi(\langle\x, \beta\rangle) - y\langle\x, \beta\rangle\right] \leq \epsilon,  
\end{equation*}
which implies the desired result. 

Then we prove that, there exists $\widehat{\lambda}$ and $\widehat{\gamma}$ such that a point $\widehat{\su} = (\widehat{\lambda}, \widehat{\beta}, \widehat{\gamma})$ is an $\epsilon$-optimal saddle-point solution of the WDRSL in Eq.~\eqref{prob:WDRSL-min-max} if a point $\widehat{\beta}$ is an an $\epsilon$-optimal solution of the WDRSL in Eq.~\eqref{prob:WDRSL-abstract}. Indeed, there exists $\widehat{\lambda} > 0$ such that 
\begin{equation*}
\sup_{\PP \in \BB_\delta(\widehat{\PP}_n)} \EE^\PP\left[\Psi(\langle\x, \widehat{\beta}\rangle) - y\langle\x, \widehat{\beta}\rangle\right] \ = \ \widehat{\lambda}\delta + \frac{1}{n}\sum_{i=1}^n \Psi(\langle\widehat{\x}_i, \widehat{\beta}\rangle) - \widehat{y}_i\langle\widehat{\x}_i, \widehat{\beta}\rangle + \frac{1}{n}\sum_{i=1}^n \max\left\{0, 2\widehat{y}_i\langle\widehat{\x}_i, \widehat{\beta}\rangle - 2\widehat{\lambda}\kappa\right\},  
\end{equation*}
In addition, we let $\widehat{\gamma} = \gamma^\star$. Then we have
\begin{equation*}
\Delta(\widehat{\su}) \ = \ L(\widehat{\lambda}, \widehat{\beta}, \gamma^\star) - L(\lambda^\star, \beta^\star, \widehat{\gamma}) \ = \ L(\widehat{\lambda}, \widehat{\beta}, \gamma^\star) - L(\lambda^\star, \beta^\star, \gamma^\star) \ \leq \ L(\widehat{\lambda}, \widehat{\beta}, \gamma^\star(\widehat{\lambda}, \widehat{\beta})) - L(\lambda^\star, \beta^\star, \gamma^\star(\lambda^\star, \beta^\star)). 
\end{equation*}
Using the previous calculations, we have
\begin{eqnarray*} 
L(\widehat{\lambda}, \widehat{\beta}, \gamma^\star(\widehat{\lambda}, \widehat{\beta})) & = & \widehat{\lambda}\delta + \frac{1}{n}\sum_{i=1}^n \Psi(\langle\widehat{\x}_i, \widehat{\beta}\rangle) - \widehat{y}_i\langle\widehat{\x}_i, \widehat{\beta}\rangle + \frac{1}{n}\sum_{i=1}^n \max\left\{0, 2\widehat{y}_i\langle\widehat{\x}_i, \widehat{\beta}\rangle - 2\widehat{\lambda}\kappa\right\} \\ 
& = & \sup_{\PP \in \BB_\delta(\widehat{\PP}_n)} \EE^\PP\left[\Psi(\langle\x, \widehat{\beta}\rangle) - y\langle\x, \widehat{\beta}\rangle\right], \\
L(\lambda^\star, \beta^\star, \gamma^\star(\lambda^\star, \beta^\star)) & = & \lambda^\star\delta + \frac{1}{n}\sum_{i=1}^n \Psi(\langle\widehat{\x}_i, \beta^\star\rangle) - \widehat{y}_i\langle\widehat{\x}_i, \beta^\star\rangle + \frac{1}{n}\sum_{i=1}^n \max\left\{0, 2\widehat{y}_i\langle\widehat{\x}_i, \beta^\star\rangle - 2\lambda^\star\kappa\right\} \\ 
& = & \inf_{\beta \in \br^d} \sup_{\PP \in \BB_\delta(\widehat{\PP}_n)} \EE^\PP\left[\Psi(\langle\x, \beta\rangle) - y\langle\x, \beta\rangle\right]. 
\end{eqnarray*}
Putting these pieces together yields that 
\begin{equation*}
\Delta(\widehat{\su}) \ \leq \ \sup_{\PP \in \BB_\delta(\widehat{\PP}_n)} \EE^\PP\left[\Psi(\langle\x, \widehat{\beta}\rangle) - y\langle\x, \widehat{\beta}\rangle\right] - \inf_{\beta \in \br^d} \sup_{\PP \in \BB_\delta(\widehat{\PP}_n)} \EE^\PP\left[\Psi(\langle\x, \beta\rangle) - y\langle\x, \beta\rangle\right] \ \leq \ \epsilon,
\end{equation*}
which implies the desired result. 

\section{Postponed Proofs in Subsection~\ref{subsec:alg_SEVR}}\label{app:alg_SEVR}
In this section, we provide the detailed proofs for Theorem~\ref{Theorem:SEVR}. Our derivation is based on a nontrivial combination of the analysis in~\citet{Allen-2016-Variance} and~\citet{Chavdarova-2019-Reducing}. For the simplicity, we denote $\su = (\lambda, \beta, \gamma)$. The convex-concave function $L: \br \times \br^d \times \br^n \rightarrow \br$ and its component function $L_i: \br \times \br^d \times \br^n \rightarrow \br$ are defined as follows: 
\begin{equation*}
L(\su) = \lambda(\delta - \kappa) + \frac{1}{n}\sum_{i=1}^n \Psi(\langle\widehat{\x}_i, \beta\rangle) + \frac{1}{n}\sum_{i=1}^n \gamma_i\left(\widehat{y}_i\langle\widehat{\x}_i, \beta\rangle - \lambda\kappa\right), 
\end{equation*}
and 
\begin{equation*}
L_i(\su) = \lambda(\delta - \kappa) + \Psi(\langle\widehat{\x}_i, \beta\rangle) + \gamma_i\left(\widehat{y}_i\langle\widehat{\x}_i, \beta\rangle - \lambda\kappa\right). 
\end{equation*}
By the definition of $F$ and $F_i$ in Eq.~\eqref{Def:operator-main} and Eq.~\eqref{Def:operator-component}, they can be expressed as follows: 
\begin{equation*}
F(\su) = \begin{pmatrix}
\nabla_{(\lambda, \beta)} L(\su) \\ - \nabla_\gamma L(\su)
\end{pmatrix}, \quad 
F_i(\su) = \begin{pmatrix}
\nabla_{(\lambda, \beta)} L_i(\su) \\ - \nabla_\gamma L_i(\su)
\end{pmatrix}.    
\end{equation*}
We also denote the constraint sets by 
\begin{eqnarray*}
\Lambda & = & \{(\lambda, \beta) \in \br \times \br^d \mid \|\beta\| \leq \lambda/(L+1)\}, \\ 
\Gamma & = & \{\gamma \in \br^n \mid \|\gamma\|_\infty \leq 1\}. 
\end{eqnarray*}
For the reference, we rewrite each iteration of the SEVR algorithm as follows: 
\begin{equation*}
\begin{array}{ll}
\g_t^s = F(\tilde{\su}^s) + (F_{i_t}(\su_t^s) - F_{i_t}(\tilde{\su}^s)), & \quad \bar{\su}_{t+1}^s = \PCal_{\Lambda \times \Gamma}(\su_t^s - \eta_{t+1}^s\g_t^s), \\
\bar{\g}_t^s = F(\tilde{\su}^s) + (F_{j_t}(\bar{\su}_{t+1}^s) - F_{j_t}(\tilde{\su}^s)), & \quad \su_{t+1}^s = \PCal_{\Lambda \times \Gamma}(\su_t^s - \eta_{t+1}^s\bar{\g}_t^s).  
\end{array}
\end{equation*}
Finally, we denote $(\lambda^\star, \beta^\star, \gamma^\star)$ as an optimal saddle point of the smooth minimax optimization model in Eq.~\eqref{prob:WDRSL-min-max} and denote the duality gap function $\Delta(\su)$ by 
\begin{equation*}
\Delta(\su) = L(\lambda, \beta, \gamma^\star) - L(\lambda^\star, \beta^\star, \gamma). 
\end{equation*} 
It is worth noting that the function $\Delta(\su)$ is convex in $\su$ since $L$ is a convex-concave function. 

\subsection{Technical Lemmas}
Our first lemma is to provide an upper bound for the variance of the gradient estimators $\g_t^s$ and $\bar{\g}_t^s$. 
\begin{lemma}\label{Lemma:SEVR-variance}
Under Assumption~\ref{Assumption:main} and let $\su^\star=(\lambda^\star, \beta^\star, \gamma^\star) \in \Lambda \times \Gamma$ be an optimal saddle point of the smooth minimax optimization model in Eq.~\eqref{prob:WDRSL-min-max}. Then, the following statement holds true, 
\begin{eqnarray*}
\EE[\|\g_t^s - F(\su_t^s)\|^2 \mid \su_t^s, \tilde{\su}^s] & \leq & 2(\ell+\kappa+1)^2\left(\|\su_t^s-\su^\star\|^2 + \|\tilde{\su}^s-\su^\star\|^2\right), \\
\EE[\|\bar{\g}_t^s - F(\bar{\su}_{t+1}^s)\|^2 \mid \bar{\su}_{t+1}^s, \tilde{\su}^s] & \leq & 16(\ell+\kappa+1)\left(\Delta(\bar{\su}_{t+1}^s) + \Delta(\tilde{\su}^s)\right) \\ 
& & \hspace*{-8em} + 2(\ell+\kappa+1)^2\left(\|\bar{\su}_{t+1}^s-\su^\star\|^2 + \|\tilde{\su}^s-\su^\star\|^2\right).
\end{eqnarray*}
\end{lemma}
\begin{proof}
Since the indices $i_t$ and $j_t$ are both uniformly sampled from the set $[n]$, the gradient estimator $\g_t^s$ and $\bar{\g}_t^s$ are both unbiased, meaning that 
\begin{equation*}
\EE[\g_t^s \mid \su_t^s, \tilde{\su}^s] = F(\su_t^s), \qquad \EE[\bar{\g}_t^s \mid \bar{\su}_{t+1}^s, \tilde{\su}^s] = F(\bar{\su}_{t+1}^s). 
\end{equation*}
Since $\EE[\|\xi-\EE[\xi]\|^2] \leq \EE[\|\xi\|^2]$ holds true for any random variable $\xi \in \br^d$, we have
\begin{eqnarray*}
\EE[\|\g_t^s - F(\su_t^s)\|^2 \mid \su_t^s, \tilde{\su}^s] & = & \EE[\|(F_{i_t}(\su_t^s) - F_{i_t}(\tilde{\su}^s)) - (F(\su_t^s) - F(\tilde{\su}^s))\|^2 \mid \su_t^s, \tilde{\su}^s] \\
& & \hspace{-8em} \leq \ \EE[\|F_{i_t}(\su_t^s) - F_{i_t}(\tilde{\su}^s)\|^2 \mid \su_t^s, \tilde{\su}^s] \nonumber
\end{eqnarray*}
Note that Assumption~\ref{Assumption:main} implies that $\Psi'$ is $\ell$-Lipschitz and $\|\x_i\|\leq 1$ for all $i \in [n]$. By the definition of the operator $F_i$, it is easy to verify that $\|F_i(\su) - F_i(\su')\| \leq (\ell+\kappa+1)\|\su-\su'\|$. Thus, we have
\begin{equation*}
\EE[\|\g_t^s - F(\su_t^s)\|^2 \mid \su_t^s, \tilde{\su}^s] \leq (\ell+\kappa+1)^2\|\su_t^s - \tilde{\su}^s\|^2.  
\end{equation*}
In addition, it follows from the Cauchy-Schwarz inequality that $\|\su_t^s - \tilde{\su}^s\|^2 \leq 2\|\su_t^s - \su^\star\|^2 + 2\|\tilde{\su}^s - \su^\star\|^2$. Putting these pieces together yields the first desired inequality. The remaining step is to prove the second desired inequality. By the similar argument, we have
\begin{eqnarray}\label{inequality:variance-SEVR-first}
\EE[\|\bar{\g}_t^s - F(\bar{\su}_{t+1}^s)\|^2 \mid \bar{\su}_{t+1}^s, \tilde{\su}^s] & = & \EE[\|(F_{j_t}(\bar{\su}_{t+1}^s) - F_{j_t}(\tilde{\su}^s)) - (F(\bar{\su}_{t+1}^s) - F(\tilde{\su}^s))\|^2 \mid \bar{\su}_{t+1}^s, \tilde{\su}^s] \nonumber \\
& & \hspace{-10em} \leq \ \EE[\|F_{j_t}(\bar{\su}_{t+1}^s) - F_{j_t}(\tilde{\su}^s)\|^2 \mid \bar{\su}_{t+1}^s, \tilde{\su}^s] \nonumber \\
& & \hspace{-10em} \leq \ 2\EE[\|F_{j_t}(\bar{\su}_{t+1}^s) - F_{j_t}(\su^\star)\|^2 \mid \bar{\su}_{t+1}^s, \tilde{\su}^s] + 2\EE[\|F_{j_t}(\tilde{\su}^s) - F_{i_t}(\su^\star)\|^2 \mid \bar{\su}_{t+1}^s, \tilde{\su}^s]. 
\end{eqnarray}
Again, by the definition of the functions $L_i$ and $F_i$, it is clear that $L_i$ is $(\ell+\kappa)$-smooth and thus $\|F_i(\su) - F_i(\su')\| \leq (\ell+\kappa+1)\|\su-\su'\|$. We define a regularized version of $L_i$ by $\varphi_i$ over the constraint set $\Lambda \times \Gamma$ as follows, 
\begin{equation*}
\varphi_i(\su) = L_i(\su) - \nabla L_i(\su^\star)^\top (\su - \su^\star) + \frac{\ell+\kappa+1}{2}\left(\|\lambda-\lambda^\star\|^2 + \|\beta-\beta^\star\|^2\right) - \frac{\ell+\kappa+1}{2}\|\gamma-\gamma^\star\|^2. 
\end{equation*}
This function $\varphi_i$ is strongly convex-concave with the module $(\ell+\kappa+1)/2$ and $2(\ell+\kappa+1)$-smooth, and the unique optimal saddle point is $\su^\star$. This implies that 
\begin{eqnarray*}
\|\nabla \varphi_i(\su)\|^2 & = & \left\|\begin{pmatrix}
\nabla_{(\lambda, \beta)} \varphi_i(\su) \\ - \nabla_\gamma \varphi_i(\su)
\end{pmatrix} - \begin{pmatrix}
\nabla_{(\lambda, \beta)} \varphi_i(\su^\star) \\ - \nabla_\gamma \varphi_i(\su^\star)
\end{pmatrix}\right\|^2 \ \leq \ 4(\ell+\kappa+1)^2\|\su - \su^\star\|^2   \\
& \leq & 8(\ell+\kappa+1)\left(\varphi_i(\lambda, \beta, \gamma^\star) - \varphi_i(\lambda^\star, \beta^\star, \gamma)\right). 
\end{eqnarray*}
By the definition of $\varphi_i$ and $F_i$, we have
\begin{equation*}
\varphi_i(\lambda, \beta, \gamma^\star) - \varphi_i(\lambda^\star, \beta^\star, \gamma) = L_i(\lambda, \beta, \gamma^\star) - L_i(\lambda^\star, \beta^\star, \gamma) - F_i(\su^\star)^\top(\su - \su^\star) + \frac{\ell+\kappa+1}{2}\|\su-\su^\star\|^2.   
\end{equation*}
Furthermore, we have 
\begin{equation*}
\nabla \varphi_i(\su) = \nabla L_i(\su) - \nabla L_i(\su^\star) + (\ell+\kappa+1)\begin{pmatrix} \lambda-\lambda^\star \\ \beta-\beta^\star \\ -(\gamma-\gamma^\star) \end{pmatrix}. 
\end{equation*}
By the Cauchy-Schwarz inequality, we have
\begin{equation*}
\|F_i(\su) - F_i(\su^\star)\|^2 = \|\nabla L_i(\su) - \nabla L_i(\su^\star)\|^2 \leq 2\|\nabla\varphi_i(\su)\|^2 + 2(\ell+\kappa+1)^2\|\su - \su^\star\|^2. 
\end{equation*}
Putting these pieces together with the fact that $j_t$ is uniformly sampled from the set $[n]$ yields that 
\begin{eqnarray*}
& & \EE[\|F_{j_t}(\bar{\su}_{t+1}^s) - F_{j_t}(\su^\star)\|^2 \mid \bar{\su}_{t+1}^s, \tilde{\su}^s] \\ 
& \leq & 16(\ell+\kappa+1)\EE\left[L_{j_t}(\bar{\lambda}_{t+1}^s, \bar{\beta}_{t+1}^s, \gamma^\star) - L_{j_t}(\lambda^\star, \beta^\star, \bar{\gamma}_{t+1}^s) - F_{j_t}(\su^\star)^\top(\bar{\su}_{t+1}^s - \su^\star) \mid \bar{\su}_{t+1}^s, \tilde{\su}^s\right] \\ 
& & + 2(\ell+\kappa+1)^2\|\bar{\su}_{t+1}^s - \su^\star\|^2 \\
& = & 16(\ell+\kappa+1)\left(\Delta(\bar{\su}_{t+1}^s) - F(\su^\star)^\top(\bar{\su}_{t+1}^s - \su^\star)\right) + 2(\ell+\kappa+1)^2\|\bar{\su}_{t+1}^s-\su^\star\|^2. 
\end{eqnarray*}
Since $\su^\star$ is an optimal saddle point of the smooth minimax optimization model in Eq.~\eqref{prob:WDRSL-min-max}, we have $F(\su^\star)^\top(\bar{\su}_{t+1}^s - \su^\star) \geq 0$. Therefore, we conclude that 
\begin{equation}\label{inequality:variance-SEVR-second}
\EE[\|F_{j_t}(\bar{\su}_{t+1}^s) - F_{j_t}(\su^\star)\|^2 \mid \bar{\su}_{t+1}^s, \tilde{\su}^s] \leq 16(\ell+\kappa+1)\Delta(\bar{\su}_{t+1}^s) + 2(\ell+\kappa+1)^2\|\bar{\su}_{t+1}^s - \su^\star\|^2. 
\end{equation}
Similarly, we have
\begin{equation}\label{inequality:variance-SEVR-third}
\EE[\|F_{j_t}(\tilde{\su}^s) - F_{j_t}(\su^\star)\|^2 \mid \su_t^s, \tilde{\su}^s] \leq 16(\ell+\kappa+1)\Delta(\tilde{\su}^s) + 2(\ell+\kappa+1)^2\|\tilde{\su}^s-\su^\star\|^2. 
\end{equation}
Plugging Eq.~\eqref{inequality:variance-SEVR-second} and Eq.~\eqref{inequality:variance-SEVR-third} into Eq.~\eqref{inequality:variance-SEVR-first} yields that
\begin{equation*}
\EE[\|\bar{\g}_t^s - F(\bar{\su}_{t+1}^s)\|^2 \mid \bar{\su}_{t+1}^s, \tilde{\su}^s] \ \leq \ 16(\ell+\kappa+1)\left(\Delta(\bar{\su}_{t+1}^s) + \Delta(\tilde{\su}^s)\right) + 2(\ell+\kappa+1)^2\left(\|\bar{\su}_{t+1}^s-\su^\star\|^2 + \|\tilde{\su}^s-\su^\star\|^2\right). 
\end{equation*}
This completes the proof. 
\end{proof}
Then, we provide a descent lemma for the iterates generated by the SEVR algorithm. 
\begin{lemma}\label{Lemma:SEVR-descent}
Under Assumption~\ref{Assumption:main} and let $\su^\star=(\lambda^\star, \beta^\star, \gamma^\star) \in \Lambda \times \Gamma$ be an optimal saddle point of the smooth minimax optimization model in Eq.~\eqref{prob:WDRSL-min-max}. Then, the following statement holds true,
\begin{eqnarray*}
& & \left(1-24\eta_{t+1}^s(\ell+\kappa+1)\right)\EE[\Delta(\bar{\su}_{t+1}^s) \mid \su_t^s, \tilde{\su}^s] \\
& \leq & \frac{1}{2\eta_{t+1}^s}\left(\|\su_t^s - \su^\star\|^2 - \EE[\|\su_{t+1}^s - \su^\star\|^2 \mid \su_t^s, \tilde{\su}^s]\right) + 24\eta_{t+1}^s(\ell+\kappa+1)\Delta(\tilde{\su}^s) \\
& & - \left(\frac{1}{2\eta_{t+1}^s} - \frac{3\eta_{t+1}^s(\ell+\kappa+1)^2}{2}\right)\EE[\|\bar{\su}_{t+1}^s - \su_t^s\|^2 \mid \su_t^s, \tilde{\su}^s] + 6\eta_{t+1}^s(\ell+\kappa+1)^2\|\tilde{\su}^s-\su^\star\|^2 \\
& & + 3\eta_{t+1}^s(\ell+\kappa+1)^2\|\bar{\su}_{t+1}^s-\su^\star\|^2 + 3\eta_{t+1}^s(\ell+\kappa+1)^2\|\su_t^s-\su^\star\|^2. 
\end{eqnarray*}
\end{lemma}
\begin{proof}
By the update formula for the iterates $\bar{\su}_{t+1}^s$ and $\su_{t+1}^s$, we have
\begin{eqnarray*}
0 & \leq & (\su - \bar{\su}_{t+1}^s)^\top\left(\bar{\su}_{t+1}^s - \su_t^s + \eta_{t+1}^s\g_t^s\right), \quad \forall\su \in \Lambda \times \Gamma, \\ 
0 & \leq & (\su - \su_{t+1}^s)^\top\left(\su_{t+1}^s - \su_t^s + \eta_{t+1}^s\bar{\g}_t^s\right), \quad \forall\su \in \Lambda \times \Gamma. 
\end{eqnarray*}
Letting $\su=\su_{t+1}^s$ in the first inequality and $\su=\su^\star$ in the second inequality and rearranging the resulting inequalities yields that 
\begin{eqnarray*}
0 & \leq & \frac{1}{2\eta_{t+1}^s}\left(\|\su_t^s - \su_{t+1}^s\|^2 - \|\su_{t+1}^s-\bar{\su}_{t+1}^s\|^2 - \|\bar{\su}_{t+1}^s - \su_t^s\|^2\right) + (\su_{t+1}^s - \bar{\su}_{t+1}^s)^\top\g_t^s, \\
0 & \leq & \frac{1}{2\eta_{t+1}^s}\left(\|\su_t^s - \su^\star\|^2 - \|\su_{t+1}^s - \su^\star\|^2 - \|\su_{t+1}^s - \su_t^s\|^2\right) + (\su^\star - \su_{t+1}^s)^\top\bar{\g}_t^s. 
\end{eqnarray*}
Summing up the above two inequalities yields that
\begin{eqnarray}\label{inequality:descent-SEVR-first}
0 & \leq & \frac{1}{2\eta_{t+1}^s}\left(\|\su_t^s - \su^\star\|^2 - \|\su_{t+1}^s - \su^\star\|^2 - \|\su_{t+1}^s-\bar{\su}_{t+1}^s\|^2 - \|\bar{\su}_{t+1}^s - \su_t^s\|^2\right) \\
& & + (\su_{t+1}^s - \bar{\su}_{t+1}^s)^\top\g_t^s + (\su^\star - \su_{t+1}^s)^\top\bar{\g}_t^s \nonumber \\
& = & \frac{1}{2\eta_{t+1}^s}\left(\|\su_t^s - \su^\star\|^2 - \|\su_{t+1}^s - \su^\star\|^2 - \|\su_{t+1}^s-\bar{\su}_{t+1}^s\|^2 - \|\bar{\su}_{t+1}^s - \su_t^s\|^2\right) \nonumber \\
& & + (\su_{t+1}^s - \bar{\su}_{t+1}^s)^\top\left(\g_t^s - \bar{\g}_t^s\right) + (\su^\star - \bar{\su}_{t+1}^s)^\top\bar{\g}_t^s. \nonumber
\end{eqnarray} 
Using the Young's inequality, we have
\begin{equation}\label{inequality:descent-SEVR-second}
(\su_{t+1}^s - \bar{\su}_{t+1}^s)^\top\left(\g_t^s - \bar{\g}_t^s\right) \ \leq \ \frac{\|\su_{t+1}^s-\bar{\su}_{t+1}^s\|^2}{2\eta_{t+1}^s} + \frac{\eta_{t+1}^s\|\g_t^s - \bar{\g}_t^s\|^2}{2}. 
\end{equation}
Using the Cauchy-Schwarz inequality and the fact that $F$ is $(\ell+\kappa+1)$-Lipschitz, we have
\begin{eqnarray}\label{inequality:descent-SEVR-third}
\|\g_t^s - \bar{\g}_t^s\|^2 & \leq & 3\|\g_t^s - F(\su_t^s)\|^2 + 3\|F(\su_t^s) - F(\bar{\su}_{t+1}^s)\|^2 + 3\|\bar{\g}_t^s - F(\bar{\su}_{t+1}^s)\|^2 \\
& \leq & 3\|\g_t^s - F(\su_t^s)\|^2 + 3\|\bar{\g}_t^s - F(\bar{\su}_{t+1}^s)\|^2 + 3(\ell+\kappa+1)^2\|\bar{\su}_{t+1}^s - \su_t^s\|^2. \nonumber
\end{eqnarray}
Combining Eq.~\eqref{inequality:descent-SEVR-first}, Eq.~\eqref{inequality:descent-SEVR-second} and Eq.~\eqref{inequality:descent-SEVR-third} yields that
\begin{eqnarray*}
(\bar{\su}_{t+1}^s - \su^\star)^\top\bar{\g}_t^s & \leq & \frac{1}{2\eta_{t+1}^s}\left(\|\su_t^s - \su^\star\|^2 - \|\su_{t+1}^s - \su^\star\|^2\right) - \left(\frac{1}{2\eta_{t+1}^s} - \frac{3\eta_{t+1}^s(\ell+\kappa+1)^2}{2}\right)\|\bar{\su}_{t+1}^s - \su_t^s\|^2 \\ 
& & + \frac{3\eta_{t+1}^s}{2}\left(\|\g_t^s - F(\su_t^s)\|^2 + \|\bar{\g}_t^s - F(\bar{\su}_{t+1}^s)\|^2\right). 
\end{eqnarray*}
Taking the expectation of both sides of the above inequalities conditioned on $(\bar{\su}_{t+1}^s, \tilde{\su}^s)$, we have
\begin{eqnarray*}
& & (\bar{\su}_{t+1}^s - \su^\star)^\top F(\bar{\su}_{t+1}^s) \\
& \leq & \frac{1}{2\eta_{t+1}^s}\left(\|\su_t^s - \su^\star\|^2 - \EE[\|\su_{t+1}^s - \su^\star\|^2 \mid \bar{\su}_{t+1}^s, \tilde{\su}^s]\right) - \left(\frac{1}{2\eta_{t+1}^s} - \frac{3\eta_{t+1}^s(\ell+\kappa+1)^2}{2}\right)\|\bar{\su}_{t+1}^s - \su_t^s\|^2 \\ 
& & + \frac{3\eta_{t+1}^s}{2}\left(\|\g_t^s - F(\su_t^s)\|^2 + \EE[\|\bar{\g}_t^s - F(\bar{\su}_{t+1}^s)\|^2 \mid \bar{\su}_{t+1}^s, \tilde{\su}^s] \right). 
\end{eqnarray*}
Plugging the second inequality of Lemma~\ref{Lemma:SEVR-variance} into the above inequality, we have 
\begin{eqnarray*}
& & (\bar{\su}_{t+1}^s - \su^\star)^\top F(\bar{\su}_{t+1}^s) \\
& \leq & \frac{1}{2\eta_{t+1}^s}\left(\|\su_t^s - \su^\star\|^2 - \EE[\|\su_{t+1}^s - \su^\star\|^2 \mid \bar{\su}_{t+1}^s, \tilde{\su}^s]\right) - \left(\frac{1}{2\eta_{t+1}^s} - \frac{3\eta_{t+1}^s(\ell+\kappa+1)^2}{2}\right)\|\bar{\su}_{t+1}^s - \su_t^s\|^2 \\
& & + \frac{3\eta_{t+1}^s}{2}\|\g_t^s - F(\su_t^s)\|^2 + 24\eta_{t+1}^s(\ell+\kappa+1)\left(\Delta(\bar{\su}_{t+1}^s) + \Delta(\tilde{\su}^s)\right) + 3\eta_{t+1}^s(\ell+\kappa+1)^2\|\tilde{\su}^s-\su^\star\|^2 \\ 
& & + 3\eta_{t+1}^s(\ell+\kappa+1)^2\|\bar{\su}_{t+1}^s-\su^\star\|^2. 
\end{eqnarray*}
By the definition of $F$ and the fact that $L$ is a convex-concave function, we have
\begin{equation*}
(\bar{\su}_{t+1}^s - \su^\star)^\top F(\bar{\su}_{t+1}^s) \ \geq \ \Delta(\bar{\su}_{t+1}^s). 
\end{equation*}
Putting these pieces together yields that
\begin{eqnarray}\label{inequality:descent-SEVR-fourth}
& & \left(1-24\eta_{t+1}^s(\ell+\kappa+1)\right)\Delta(\bar{\su}_{t+1}^s) \\
& \leq & \frac{1}{2\eta_{t+1}^s}\left(\|\su_t^s - \su^\star\|^2 - \EE[\|\su_{t+1}^s - \su^\star\|^2 \mid \bar{\su}_{t+1}^s, \tilde{\su}^s]\right) + \frac{3\eta_{t+1}^s}{2}\|\g_t^s - F(\su_t^s)\|^2 + 24\eta_{t+1}^s(\ell+\kappa+1)\Delta(\tilde{\su}^s) \nonumber \\
& & - \left(\frac{1}{2\eta_{t+1}^s} - \frac{3\eta_{t+1}^s(\ell+\kappa+1)^2}{2}\right)\|\bar{\su}_{t+1}^s - \su_t^s\|^2 + 3\eta_{t+1}^s(\ell+\kappa+1)^2\|\tilde{\su}^s-\su^\star\|^2 \nonumber \\
& & + 3\eta_{t+1}^s(\ell+\kappa+1)^2\|\bar{\su}_{t+1}^s-\su^\star\|^2. \nonumber
\end{eqnarray}
Taking the expectation of both sides of Eq.~\eqref{inequality:descent-SEVR-fourth} conditioned on $(\su_t^s, \tilde{\su}^s)$ together with the tower property and the first inequality of Lemma~\ref{Lemma:SEVR-variance}, we have
\begin{eqnarray*}
& & \left(1-24\eta_{t+1}^s(\ell+\kappa+1)\right)\EE[\Delta(\bar{\su}_{t+1}^s) \mid \su_t^s, \tilde{\su}^s] \\
& \leq & \frac{1}{2\eta_{t+1}^s}\left(\|\su_t^s - \su^\star\|^2 - \EE[\|\su_{t+1}^s - \su^\star\|^2 \mid \su_t^s, \tilde{\su}^s]\right) + 24\eta_{t+1}^s(\ell+\kappa+1)\Delta(\tilde{\su}^s) \\
& & - \left(\frac{1}{2\eta_{t+1}^s} - \frac{3\eta_{t+1}^s(\ell+\kappa+1)^2}{2}\right)\EE[\|\bar{\su}_{t+1}^s - \su_t^s\|^2 \mid \su_t^s, \tilde{\su}^s] + 6\eta_{t+1}^s(\ell+\kappa+1)^2\|\tilde{\su}^s-\su^\star\|^2 \\
& & + 3\eta_{t+1}^s(\ell+\kappa+1)^2\EE[\|\bar{\su}_{t+1}^s-\su^\star\|^2 \mid \su_t^s, \tilde{\su}^s] + 3\eta_{t+1}^s(\ell+\kappa+1)^2\|\su_t^s-\su^\star\|^2. 
\end{eqnarray*}
This completes the proof. 
\end{proof}
Now we are ready to prove a key technical lemma which is crucial to our subsequent analysis.  
\begin{lemma}\label{Lemma:Key-descent-SEVR}
Under Assumption~\ref{Assumption:main} and let $\su^\star=(\lambda^\star, \beta^\star, \gamma^\star) \in \Lambda \times \Gamma$ be an optimal saddle point of the smooth minimax optimization model in Eq.~\eqref{prob:WDRSL-min-max}. If $k_0 \geq 1$, $0 < \eta \leq \frac{1}{100(\ell+\kappa+1)}$, and $\frac{1}{2\sqrt{2}\eta T} \geq 100\eta(\ell+\kappa+1)^2$, the following statement holds true, 
\begin{equation*}
\EE[\Delta(\tilde{\su}^S)] \leq \frac{1}{2^S}\left(\Delta(\tilde{\su}^0) + 18\eta(\ell+\kappa+1)^2\|\tilde{\su}^0-\su^\star\|^2 + \frac{3\|\su_0^0 - \su^\star\|^2}{2k_s\eta_0^s}\right). 
\end{equation*}
\end{lemma}
\begin{proof}
By the Cauchy-Schwarz inequality, we have
\begin{equation*}
\|\bar{\su}_{t+1}^s-\su^\star\|^2 \ \leq \ 2\|\bar{\su}_{t+1}^s - \su_t^s\|^2 + 2\|\su_t^s-\su^\star\|^2.  
\end{equation*}
Rearranging the inequality in Lemma~\ref{Lemma:SEVR-descent} with the above inequality and $0 < \eta_{t+1}^s \leq \eta$, we have
\begin{eqnarray}\label{inequality:Key-descent-SEVR-first}
& & \left(1-24\eta(\ell+\kappa+1)\right)\EE[\Delta(\bar{\su}_{t+1}^s) \mid \su_t^s, \tilde{\su}^s] \\
& \leq & \frac{1}{2\eta_{t+1}^s}\left(\|\su_t^s - \su^\star\|^2 - \EE[\|\su_{t+1}^s - \su^\star\|^2 \mid \su_t^s, \tilde{\su}^s]\right) + 24\eta(\ell+\kappa+1)\Delta(\tilde{\su}^s) \nonumber \\
& & - \left(\frac{1}{2\eta} - \frac{63\eta(\ell+\kappa+1)^2}{2}\right)\EE[\|\bar{\su}_{t+1}^s - \su_t^s\|^2 \mid \su_t^s, \tilde{\su}^s] + 6\eta(\ell+\kappa+1)^2\|\tilde{\su}^s-\su^\star\|^2 \nonumber \\
& & - 12\eta(\ell+\kappa+1)^2\EE[\|\bar{\su}_{t+1}^s-\su^\star\|^2 \mid \su_t^s, \tilde{\su}^s] + 33\eta(\ell+\kappa+1)^2\|\su_t^s-\su^\star\|^2. \nonumber  
\end{eqnarray}
Since $\eta \leq \frac{1}{100(\ell+\kappa+1)}$, we have
\begin{eqnarray*}
1 - 24\eta(\ell+\kappa+1) & \geq & \frac{2}{3}, \\
24\eta(\ell+\kappa+1) & \leq & \frac{1}{3}, \\
\frac{1}{2\eta} - \frac{63\eta(\ell+\kappa+1)^2}{2} & \geq & 0. 
\end{eqnarray*}
Putting these pieces together with Eq.~\eqref{inequality:Key-descent-SEVR-first} and taking the expectation on all the randomness yields that 
\begin{eqnarray*}
\EE[\Delta(\bar{\su}_{t+1}^s)] & \leq & \frac{3}{4\eta_{t+1}^s}\left(\EE[\|\su_t^s - \su^\star\|^2] - \EE[\|\su_{t+1}^s - \su^\star\|^2]\right) + \frac{\EE[\Delta(\tilde{\su}^s)]}{2} + 50\eta(\ell+\kappa+1)^2\EE[\|\su_t^s-\su^\star\|^2] \\
& & + 9\eta(\ell+\kappa+1)^2\EE[\|\tilde{\su}^s-\su^\star\|^2] - 18\eta(\ell+\kappa+1)^2\EE[\|\bar{\su}_{t+1}^s-\su^\star\|^2].  
\end{eqnarray*}
Note that, by the definition of $\eta_t^s$, we have
\begin{equation*}
\frac{1}{\eta_t^s} - \frac{1}{\eta_{t+1}^s} \geq \frac{1}{2\eta\sqrt{T}\sqrt{2T}} = \frac{1}{2\sqrt{2}\eta T} \geq 100\eta(\ell+\kappa+1)^2. 
\end{equation*}
Therefore, we conclude that
\begin{eqnarray}\label{inequality:Key-descent-SEVR-second}
& & \EE[\Delta(\bar{\su}_{t+1}^s)] + 18\eta(\ell+\kappa+1)^2\EE[\|\bar{\su}_{t+1}^s-\su^\star\|^2] \\ 
& \leq & \frac{3}{4}\left(\frac{\EE[\|\su_t^s - \su^\star\|^2]}{\eta_t^s} - \frac{\EE[\|\su_{t+1}^s - \su^\star\|^2]}{\eta_{t+1}^s}\right) + \frac{1}{2}\left(\EE[\Delta(\tilde{\su}^s)] + 18\eta(\ell+\kappa+1)^2\EE[\|\tilde{\su}^s-\su^\star\|^2]\right). \nonumber  
\end{eqnarray}
Summing Eq.~\eqref{inequality:Key-descent-SEVR-second} up over $t=0,1,2,\ldots,k_s-1$ and dividing both sides by $k_s$, we have
\begin{eqnarray*}
& & \EE\left[\sum_{t=0}^{k_s-1} \frac{\Delta(\bar{\su}_{t+1}^s) + 18\eta(\ell+\kappa+1)^2\|\bar{\su}_{t+1}^s-\su^\star\|^2}{k_s}\right] \\ 
& \leq & \frac{3}{4}\left(\frac{\EE[\|\su_0^s - \su^\star\|^2]}{k_s\eta_0^s} - \frac{\EE[\|\su_{k_s}^s - \su^\star\|^2]}{k_s\eta_{k_s}^s}\right) + \frac{1}{2}\left(\EE[\Delta(\tilde{\su}^s)] + 18\eta(\ell+\kappa+1)^2\EE[\|\tilde{\su}^s-\su^\star\|^2]\right).  
\end{eqnarray*}
Note that $\Delta(\su)$ and $\|\su - \su^\star\|^2$ are both convex in $\su$. Since $\tilde{\su}^{s+1} = (1/k_s)\sum_{t=1}^{k_s} \bar{\su}_t^s$, we have
\begin{equation*}
\Delta(\tilde{\su}^{s+1}) + 18\eta(\ell+\kappa+1)^2\|\tilde{\su}^{s+1}-\su^\star\|^2 \ \leq \ \sum_{t=0}^{k_s-1} \frac{\Delta(\bar{\su}_{t+1}^s) + 18\eta(\ell+\kappa+1)^2\|\bar{\su}_{t+1}^s-\su^\star\|^2}{k_s}. 
\end{equation*}
In addition, we have $\su_{k_s}^s = \su_0^{s+1}$, $\eta_{k_s}^s = \eta_0^{s+1}$ and $k_{s+1}=2k_s$. Putting these pieces together yields that 
\begin{eqnarray*}
& & \EE\left[\Delta(\tilde{\su}^{s+1}) + 18\eta(\ell+\kappa+1)^2\|\tilde{\su}^{s+1}-\su^\star\|^2 + \frac{3\|\su_0^{s+1} - \su^\star\|^2}{2k_{s+1}\eta_0^{s+1}}\right] \\ 
& \leq & \frac{1}{2}\left(\EE\left[\Delta(\tilde{\su}^s) + 18\eta(\ell+\kappa+1)^2\|\tilde{\su}^s-\su^\star\|^2 + \frac{3\|\su_0^s - \su^\star\|^2}{2k_s\eta_0^s}\right]\right).  
\end{eqnarray*}
After telescoping for $s=0,1,2,\ldots,S-1$ and using $k_0 \geq 1$, we have
\begin{equation*}
\EE[\Delta(\tilde{\su}^S)] \leq \frac{1}{2^S}\left(\Delta(\tilde{\su}^0) + 18\eta(\ell+\kappa+1)^2\|\tilde{\su}^0-\su^\star\|^2 + \frac{3\|\su_0^0 - \su^\star\|^2}{\eta k_0\sqrt{2}}\right). 
\end{equation*}
This completes the proof. 
\end{proof}
\subsection{Proof of Theorem~\ref{Theorem:SEVR}}
Recall that the given parameter choices are
\begin{eqnarray*}
S & = & 1 + \left\lfloor \log_2\left(\frac{10D_L}{\epsilon}\right)\right\rfloor, \\
\eta & = & \min\left\{\frac{1}{100(\ell+\kappa+1)}, \frac{\epsilon}{2000\sqrt{2}(\ell+\kappa+1)^2D_\su^2}, \frac{D_\su^2}{D_L}\right\}, \\
k_0 & = & \frac{D_\su^2}{\eta D_L} \ \geq \ 1.  
\end{eqnarray*} 
Then, we have
\begin{equation*}
\frac{1}{2\sqrt{2}\eta T} \geq \frac{1}{2\sqrt{2}\eta k_0 \cdot 2^S} = \frac{1}{2\sqrt{2}\eta k_0}  \frac{\epsilon}{10D_L} = \frac{\epsilon}{20\sqrt{2}D_\su^2}. 
\end{equation*}
This together with 
$$0 < \eta \leq \frac{\epsilon}{2000\sqrt{2}(\ell+\kappa+1)^2D_\su^2}$$
yields that 
\begin{equation*}
\frac{1}{2\sqrt{2}\eta T} \geq \frac{\epsilon}{20\sqrt{2}D_\su^2} \geq 100\eta(\ell+\kappa+1)^2. 
\end{equation*}
Thus, Lemma~\ref{Lemma:Key-descent-SEVR} holds true and we have
\begin{equation*}
\EE[\Delta(\tilde{\su}^S)] \leq \frac{1}{2^S}\left(\Delta(\tilde{\su}^0) + 18\eta(\ell+\kappa+1)^2\|\tilde{\su}^0-\su^\star\|^2 + \frac{3\|\su_0^0 - \su^\star\|^2}{\eta k_0\sqrt{2}}\right). 
\end{equation*}
Now we bound the three terms on the right hand side of the above inequality. Indeed, we have
\begin{eqnarray*}
\frac{\Delta(\tilde{\su}^0)}{2^S} & \leq & \left(\frac{\Delta(\tilde{\su}^0)}{10D_L}\right)\epsilon \ \leq \ \frac{\epsilon}{3}, \\
\frac{18\eta(\ell+\kappa+1)^2\|\tilde{\su}^0-\su^\star\|^2}{2^S} & \leq & 18\eta(\ell+\kappa+1)^2\|\tilde{\su}^0-\su^\star\|^2 \ \leq \ \left(\frac{18\|\tilde{\su}^0-\su^\star\|^2}{400\sqrt{2}D_\su^2}\right)\epsilon \ \leq \ \frac{\epsilon}{3}, \\ 
\frac{1}{2^S}\left(\frac{3\|\su_0^0 - \su^\star\|^2}{\eta k_0\sqrt{2}}\right) & \leq & \frac{\epsilon}{10D_L}\frac{3\|\su_0^0 - \su^\star\|^2}{\sqrt{2}}\frac{D_L}{D_\su^2} \ \leq \ \left(\frac{3\|\su_0^0 - \su^\star\|^2}{10\sqrt{2}D_\su^2}\right)\epsilon \ \leq \ \frac{\epsilon}{3}. 
\end{eqnarray*}
This completes the proof. 

\section{Postponed Proofs in Subsection~\ref{subsec:alg_SPPRR}}\label{app:alg_SPPRR}
In this section, we provide the detailed proofs for Theorem~\ref{Theorem:SPPRR}. Our derivation extends the analysis in~\citet{Nagaraj-2019-Sgd} from convex optimization to convex-concave minimax optimization. For the simplicity, we also denote $\su = (\lambda, \beta, \gamma)$ and the functions $L$ and $L_i$ by 
\begin{eqnarray*}
L(\su) & = & \lambda(\delta - \kappa) + \frac{1}{n}\sum_{i=1}^n \Psi(\langle\widehat{\x}_i, \beta\rangle) + \frac{1}{n}\sum_{i=1}^n \gamma_i\left(\widehat{y}_i\langle\widehat{\x}_i, \beta\rangle - \lambda\kappa\right), \\
L_i(\su) & = & \lambda(\delta - \kappa) + \Psi(\langle\widehat{\x}_i, \beta\rangle) + \gamma_i\left(\widehat{y}_i\langle\widehat{\x}_i, \beta\rangle - \lambda\kappa\right).
\end{eqnarray*}
We also have
\begin{equation*}
F(\su) = \begin{pmatrix}
\nabla_{(\lambda, \beta)} L(\su) \\ - \nabla_\gamma L(\su)
\end{pmatrix}, \quad 
F_i(\su) = \begin{pmatrix}
\nabla_{(\lambda, \beta)} L_i(\su) \\ - \nabla_\gamma L_i(\su)
\end{pmatrix}.    
\end{equation*}
and 
\begin{eqnarray*}
\Lambda & = & \{(\lambda, \beta) \in \br \times \br^d \mid \|\beta\| \leq \lambda/(L+1)\}, \\ 
\Gamma & = & \{\gamma \in \br^n \mid \|\gamma\|_\infty \leq 1\}. 
\end{eqnarray*}
For the reference, we rewrite each iteration of the SPPRR algorithm as follows: 
\begin{equation*}
\su_{i+1, t}^s \ = \ \PCal_{\Lambda \times \Gamma}(\su_t^s - \eta F_{\sigma^{s}_t}(\su_{i, t}^s)).
\end{equation*}
Finally, we denote $(\lambda^\star, \beta^\star, \gamma^\star)$ as an optimal saddle point of the smooth minimax optimization model in Eq.~\eqref{prob:WDRSL-min-max} and recall that the duality gap function $\Delta(\su)$ is defined by 
\begin{equation*}
\Delta(\su) = L(\lambda, \beta, \gamma^\star) - L(\lambda^\star, \beta^\star, \gamma). 
\end{equation*} 
It is worth noting that the function $\Delta(\su)$ is convex in $\su$ since $L$ is a convex-concave function.

\subsection{Technical Lemmas}
Before presenting our technical lemmas, we define the exchangeable pair and some other key notations which are first introduced by~\citet{Nagaraj-2019-Sgd} for convex optimization and become a standard machinery for analyzing stochastic algorithm with random reshuffling. 

Suppose that we run the SPPRR algorithm for $s$ epochs using the random permutation $\sigma^0, \sigma^1, \ldots, \sigma^{s-1}$ and obtain $\su_0^s = \su_n^{s-1}$. If $s=0$, we start with the initial point $\su_0^0$. The exchange pair is defined by two different $s$-th epoch iterates $\{\su_t(\sigma^s)\}_{1 \leq t \leq n}$ and $\{\su_t(\tilde{\sigma}^s)\}_{1 \leq t \leq n}$ obtained by running the $s$-th epoch with independent uniform permutations $\sigma^s$ and $\tilde{\sigma}^s$. Then, it is obvious that $\{\su_t(\sigma^s)\}_{1 \leq t \leq n}$ and $\{\su_t(\tilde{\sigma}^s)\}_{1 \leq t \leq n}$ are independent and identically distributed. This further implies that  
\begin{equation}\label{Def:Key-exchange-pair}
\EE[L_{\sigma_{t-1}^s}(\lambda_t(\tilde{\sigma}^s), \beta_t(\tilde{\sigma}^s), \gamma^\star) - L_{\sigma_{t-1}^s}(\lambda^\star, \beta^\star, \gamma_t(\tilde{\sigma}^s))] = \EE[\Delta(\su_t(\tilde{\sigma}^s))] = \EE[\Delta(\su_t(\sigma^s))] = \EE[\Delta(\su_t^s)],
\end{equation}
Let $\DCal_{t, s}$ be the distribution of the iterate $\su_t^s$ under the random shuffling $\sigma^s$ and $\DCal_{t, s}^{(r)}$ be the distribution of the iterate $\su_t^s$ under the random shuffling $\sigma^s$ conditioned on the event $\{\sigma_{t-1}^s=r\}$. This is different from the event $\{\sigma_{t+1}^s=r\}$ used in~\citet{Nagaraj-2019-Sgd} since we analyze inexact proximal point update instead of projected gradient update. Nonetheless, our proof techniques are also based on the Kantorovich duality and the 1-Wasserstein and 2-Wasserstein distances between $\DCal_{t, s}$ and $\DCal_{t, s}^{(r)}$.   
\begin{definition}[1-Wasserstein and 2-Wasserstein distance]
Suppose that $\mu$ and $\nu$ be two probability distributions over $\br^N$ such that $\EE_\mu[\|X\|^2] < +\infty$ and $\EE_\nu[\|Y\|^2] < +\infty$. Let $X \sim \mu$ and $Y \sim \nu$ be random vectors defined on a common measure space (i.e., they are coupled). The 1-Wasserstein and 2-Wasserstein distances between $\mu$ and $\nu$ are defined by 
\begin{eqnarray*}
\WCal_1(\mu, \nu) & = & \inf_{\pi \in \Pi(\mu, \nu)} \EE_\pi[\|X-Y\|], \\
\WCal_2(\mu, \nu) & = & \inf_{\pi \in \Pi(\mu, \nu)} \sqrt{\EE_\pi[\|X-Y\|^2]}. 
\end{eqnarray*}
where $\pi \in \Pi(\mu, \nu)$ denotes a coupling (or a joint distribution) over $(X, Y)$ with marginals $\mu$ and $\nu$. 
\end{definition}
By the above definition and Jensen's inequality, we have $\WCal_1(\mu, \nu) \leq \WCal_2(\mu, \nu)$. The following lemma summarizes a key characterization of 1-Wasserstein distance~\citep[Theorem~5.10]{Villani-2008-Optimal}. 
\begin{lemma}[Kantorovich Duality] Suppose that $\mu$ and $\nu$ be two probability distributions over $\br^N$ such that $\EE_\mu[\|X\|^2] < +\infty$ and $\EE_\nu[\|Y\|^2] < +\infty$. Then, we have
\begin{equation*}
\WCal_1(\mu, \nu) \ = \ \sup_{g \textnormal{ is 1-Lipschitz}} \EE_\mu[g(X)] - \EE_\nu[g(Y)]. 
\end{equation*}
\end{lemma}
The second lemma is to provide an upper bound for the difference between the unbiased gap $\EE[\Delta(\su_t^s)]$ and the gap estimator $\EE[L_{\sigma_{t-1}^s}(\lambda_t^s, \beta_t^s, \gamma^\star) - L(\lambda^\star, \beta^\star, \gamma_t^s)]$ using the 2-Wasserstein distance.  
\begin{lemma}\label{Lemma:SPPRR-bias}
Let $\su^\star=(\lambda^\star, \beta^\star, \gamma^\star) \in \Lambda \times \Gamma$ be an optimal saddle point of the smooth minimax optimization model in Eq.~\eqref{prob:WDRSL-min-max}. Under Assumption~\ref{Assumption:main} and let $F_i$ be bounded over $\Lambda \times \Gamma$ for all $i \in [n]$. Then, the following statement holds true,
\begin{equation*}
\left|\EE[\Delta(\su_t^s)] - \EE[L_{\sigma_{t-1}^s}(\lambda_t^s, \beta_t^s, \gamma^\star) - L_{\sigma_{t-1}^s}(\lambda^\star, \beta^\star, \gamma_t^s)]\right| \ \leq \ \frac{G}{n}\sum_{r=1}^n \WCal_2(\DCal_{t, s}, \DCal_{t, s}^{(r)}). 
\end{equation*}
\end{lemma}
\begin{proof}
By the definition of the exchangeable pair, we derive from Eq.~\eqref{Def:Key-exchange-pair} that 
\begin{eqnarray*}
& & \left|\EE[\Delta(\su_t^s)] - \EE[L_{\sigma_{t-1}^s}(\lambda_t^s, \beta_{t}^s, \gamma^\star) - L_{\sigma_{t-1}^s}(\lambda^\star, \beta^\star, \gamma_t^s)]\right| \\ 
& = & \left|\EE[L_{\sigma_{t-1}^s}(\lambda_t(\tilde{\sigma}^s), \beta_t(\tilde{\sigma}^s), \gamma^\star) - L_{\sigma_{t-1}^s}(\lambda^\star, \beta^\star, \gamma_t(\tilde{\sigma}^s))] - \EE[L_{\sigma_{t-1}^s}(\lambda_t^s, \beta_t^s, \gamma^\star) - L_{\sigma_{t-1}^s}(\lambda^\star, \beta^\star, \gamma_t^s)]\right|. 
\end{eqnarray*}
By using the conditional expectation on $\{\sigma_{t-1}^s=r\}$ and the fact that $\sigma^s$ and $\tilde{\sigma}^s$ are independent, we have
\begin{eqnarray*}
\EE[L_{\sigma_{t-1}^s}(\lambda_t(\tilde{\sigma}^s), \beta_t(\tilde{\sigma}^s), \gamma^\star) - L_{\sigma_{t-1}^s}(\lambda^\star, \beta^\star, \gamma_t(\tilde{\sigma}^s))] & = & \frac{1}{n}\sum_{r=1}^n \EE[L_r(\lambda_t(\tilde{\sigma}^s), \beta_t(\tilde{\sigma}^s), \gamma^\star) - L_r(\lambda^\star, \beta^\star, \gamma_t(\tilde{\sigma}^s))] \\ 
\EE[L_{\sigma_{t-1}^s}(\lambda_t^s, \beta_t^s, \gamma^\star) - L_{\sigma_{t-1}^s}(\lambda^\star, \beta^\star, \gamma_t^s)] & = & \frac{1}{n}\sum_{r=1}^n  \EE[L_r(\lambda_t^s, \beta_t^s, \gamma^\star) - L_r(\lambda^\star, \beta^\star, \gamma_t^s) \mid \sigma_{t-1}^s=r]. 
\end{eqnarray*}
Putting these pieces together with the triangle inequality yields that 
\begin{eqnarray*}
& & \left|\EE[\Delta(\su_t^s)] - \EE[L_{\sigma_{t-1}^s}(\lambda_t^s, \beta_t^s, \gamma^\star) - L_{\sigma_{t-1}^s}(\lambda^\star, \beta^\star, \gamma_t^s)]\right| \\ 
& \leq & \frac{1}{n}\sum_{r=1}^n \left|\EE[L_r(\lambda_t(\tilde{\sigma}^s), \beta_t(\tilde{\sigma}^s), \gamma^\star) - L_r(\lambda^\star, \beta^\star, \gamma_t(\tilde{\sigma}^s))] - \EE[L_r(\lambda_t^s, \beta_t^s, \gamma^\star) - L_r(\lambda^\star, \beta^\star, \gamma_t^s) \mid \sigma_{t-1}^s=r]\right| \\
& \leq & \frac{1}{n}\sum_{r=1}^n \sup_{g \textnormal{ is $G$-Lipschitz}} \left(\EE[g(\lambda_t(\tilde{\sigma}^s), \beta_t(\tilde{\sigma}^s), \gamma_t(\tilde{\sigma}^s))] - \EE[g(\lambda_t^s, \beta_t^s, \gamma_t^s) \mid \sigma_{t-1}^s=r]\right) \\
& = & \frac{1}{n}\sum_{r=1}^n G \cdot \WCal_1(\DCal_{t, s}, \DCal_{t, s}^{(r)}), 
\end{eqnarray*} 
where $G$ is the Lipschitz parameter. This together with the fact that $\WCal_1(\mu, \nu) \leq \WCal_2(\mu, \nu)$ yields the desired result.  
\end{proof}
Recall that the inexact oracle returns a point $\su_{M, t}^s$ that approximates the fixed-point problem $\su = \PCal_{\Lambda \times \Gamma}(\su_t^s - \eta F_{\sigma_t^s}(\su))$. For the simplicity of notation, we denote the fixed point by $\bar{\su}_t^s \in \Lambda \times \Gamma$ given the point $\su_t^s$. The following lemma provides an upper bound based on the coupling characterization of 2-Wasserstein distance under the certain error criterion.  
\begin{lemma}\label{Lemma:SPPRR-coupling}
Let $\su^\star=(\lambda^\star, \beta^\star, \gamma^\star) \in \Lambda \times \Gamma$ be an optimal saddle point of the smooth minimax optimization model in Eq.~\eqref{prob:WDRSL-min-max}. Under Assumption~\ref{Assumption:main} and let $F_i$ be bounded over $\Lambda \times \Gamma$ for all $i \in [n]$. Suppose that $\epsilon \in (0, 1)$ denotes the desired tolerance. Given any fixed random permutation $\sigma_t^s$, the following error criterion holds true,   
\begin{equation*}
\|\su_t^s - \bar{\su}_{t-1}^s\| = \|\su_{M, t-1}^s - \bar{\su}_{t-1}^s\| \ \leq \ \frac{D_\su}{10(nS)^{3/2}}. 
\end{equation*}
Then, the following statement holds true,
\begin{equation*}
\left|\EE[\Delta(\su_t^s)] - \EE[L_{\sigma_{t-1}^s}(\lambda_t^s, \beta_t^s, \gamma^\star) - L_{\sigma_{t-1}^s}(\lambda^\star, \beta^\star, \gamma_t^s)]\right| \ \leq \ 2\eta G^2 + \frac{D_\su}{5\sqrt{nS}}. 
\end{equation*} 
\end{lemma}
\begin{proof}
By Lemma~\ref{Lemma:SPPRR-bias}, it suffices to upper bound the term $\WCal_2(\DCal_{t, s}, \DCal_{t, s}^{(r)})$. While the coupling used here has been established in~\citet{Nagaraj-2019-Sgd}, we hope to remark that their argument can not be directly applied since the iterates generated by the SPPRR algorithm does not necessarily satisfy~\citet[Lemma~2]{Nagaraj-2019-Sgd}. Nonetheless, we prove that this issue can be fixed by controlling the error if $\|\su_{M, t}^s - \bar{\su}_t^s\| \leq \frac{\epsilon}{10GnS}$. To facilitate the readers, we divide our proof into three steps.  
\paragraph{Coupling construction.} By the definition of 2-Wasserstein distance, the tight upper bound is based on the construction of a particular coupling between $\DCal_{t, s}$ and $\DCal_{t, s}^{(r)}$. We use the one established in~\citet{Nagaraj-2019-Sgd} and present the details as follows.  

Let $\RCal_n$ be the set of all random permutations over the set $[n]$. For $i, j \in [n]$, we define the exchange function $E_{i, j}: \RCal_n \rightarrow \RCal_n$. That is to say, for any $\sigma \in \RCal_n$, the permutation $E_{i, j}(\sigma)$ stands for a new one where $i$-th and $j$-th entries of $\sigma$ are exchangeable and it keeps everything else the same. We construct the operator $\omega_{r, t}: \RCal_n \rightarrow \RCal_n$ as follows: 
\begin{equation*}
\omega_{r, t}(\sigma) = \left\{
\begin{array}{ll}
\sigma, & \textnormal{if } \sigma_{t-1} = r, \\
E_{t-1, j}(\sigma), & \textnormal{if } \sigma_j = r \textnormal{ and } j \neq t-1. 
\end{array}
\right. 
\end{equation*}
Intuitively, $\omega_{r, t}$ performs a single swap such that the $(t-1)$-th position of the permutation is $r$. Clearly, if $\sigma^s$ is a random permutation at uniform, $\omega_{r, t}(\sigma^s)$ has the same distribution as $\sigma^s \mid \sigma^s_{t-1} = r$. Based on this construction, we have $\su_t(\sigma^s) \sim \DCal_{t, s}$ and $\su_{t}(\omega_{r, t}(\sigma^s)) \sim \DCal_{t, s}^{(r)}$. This gives a coupling between $\DCal_{t, s}$ and $\DCal_{t, s}^{(r)}$. Therefore, we conclude that 
\begin{equation}\label{inequality-SPPRR-coupling-main}
\WCal_2(\DCal_{t, s}, \DCal_{t, s}^{(r)}) \ \leq \ \sqrt{\EE[\|\su_t(\sigma^s) - \su_i(\omega_{r, t}(\sigma^s))\|^2]}. 
\end{equation}
\paragraph{Coupling upper bound.} Let $\sigma^s$ and $\tilde{\sigma}^s$ be two random permutations of the set $[n]$, we denote $\su_t(\sigma^s)$ and $\su_t(\tilde{\sigma}^s)$ by $\sv_t$ and $\w_t$ respectively. It is clear that $\|\sv_0-\w_0\|=0$ almost surely. Let $j \leq t$, we consider two cases. For the first case, we suppose that $\sigma_j^s = r \neq \tilde{r} = \tilde{\sigma}_j^s$. Then, due to the error criterion, we have  
\begin{equation*}
\|\sv_j - \w_j\| \ \leq \ \|\sv_j - \bar{\sv}_{j-1}\| + \|\bar{\sv}_{j-1} - \bar{\w}_{j-1}\| + \|\bar{\w}_{j-1} - \w_j\| \ \leq \ \|\bar{\sv}_{j-1} - \bar{\w}_{j-1}\| + \frac{D_\su}{5(nS)^{3/2}},  
\end{equation*}
where $\bar{\sv}_{j-1}$ and $\bar{\w}_{j-1}$ are defined by
\begin{eqnarray*}
\bar{\sv}_{j-1} & = & \PCal_{\Lambda \times \Gamma}(\sv_{j-1} - \eta F_r(\bar{\sv}_{j-1})), \\
\bar{\w}_{j-1} & = & \PCal_{\Lambda \times \Gamma}(\w_{j-1} - \eta F_{\tilde{r}}(\bar{\w}_{j-1})). 
\end{eqnarray*}
Therefore, we have
\begin{eqnarray}\label{inequality-SPPRR-coupling-first}
\|\sv_j - \w_j\| & \leq & \|(\sv_{j-1} - \eta F_r(\bar{\sv}_{j-1})) - (\w_{j-1} - \eta F_{\tilde{r}}(\bar{\w}_{j-1}))\| + \frac{D_\su}{5(nS)^{3/2}}  \\
& \leq & \|\sv_{j-1} - \w_{j-1}\| + 2\eta G + \frac{D_\su}{5(nS)^{3/2}}. \nonumber  
\end{eqnarray}
For the second case, we suppose that $\sigma_j^s = r = \tilde{\sigma}_j^s$. Then, by the same argument, we have
\begin{equation*}
\|\sv_j - \w_j\| \ \leq \ \|\bar{\sv}_{j-1} - \bar{\w}_{j-1}\| + \frac{D_\su}{5(nS)^{3/2}},  
\end{equation*}
where $\bar{\sv}_{j-1}$ and $\bar{\w}_{j-1}$ are defined by
\begin{eqnarray*}
\bar{\sv}_{j-1} & = & \PCal_{\Lambda \times \Gamma}(\sv_{j-1} - \eta F_r(\bar{\sv}_{j-1})), \\
\bar{\w}_{j-1} & = & \PCal_{\Lambda \times \Gamma}(\w_{j-1} - \eta F_r(\bar{\w}_{j-1})). 
\end{eqnarray*}
Since $F_r$ is monotone, it is easy to verify that $\|\bar{\sv}_{j-1} - \bar{\w}_{j-1}\| \leq \|\sv_{j-1} - \w_{j-1}\|$. Therefore, we have
\begin{equation}\label{inequality-SPPRR-coupling-second}
\|\sv_j - \w_j\| \ \leq \ \|\sv_{j-1} - \w_{j-1}\| + \frac{D_\su}{5(nS)^{3/2}}. 
\end{equation}
Combining Eq.~\eqref{inequality-SPPRR-coupling-first} and Eq.~\eqref{inequality-SPPRR-coupling-second} yields that 
\begin{equation}\label{inequality-SPPRR-coupling-third}
\|\su_t(\sigma^s) - \su_t(\tilde{\sigma}^s)\| \ \leq \ 2\eta G \cdot |\{j \leq t: \sigma_j^s \neq \tilde{\sigma}_j^s\}| + \frac{t D_\su}{5(nS)^{3/2}}. 
\end{equation}
\paragraph{Main proof.} By the definition of $\omega_{r, t}(\cdot)$, we have $|\{j \leq t: \sigma_j^s \neq [\omega_{r, t}(\sigma^s)]_j^s\}| \leq 1$. Plugging this into Eq.~\eqref{inequality-SPPRR-coupling-third} and using the fact that $0 \leq t < n$ and $S \geq 1$ yields that 
\begin{equation*}
\|\su_t(\sigma^s) - \su_i(\omega_{r, t}(\sigma^s))\| \ \leq \  2\eta G + \frac{D_\su}{5\sqrt{nS}}. 
\end{equation*}
Plugging the above inequality into Eq.~\eqref{inequality-SPPRR-coupling-main} yields that 
\begin{equation*}
\WCal_2(\DCal_{t, s}, \DCal_{t, s}^{(r)}) \ \leq \ 2\eta G + \frac{D_\su}{5\sqrt{nS}}. 
\end{equation*}
This together with Lemma~\ref{Lemma:SPPRR-bias} yields the desired result. 
\end{proof}
Finally, we provide a upper bound on the term $\|\su_t^s - \su^\star\|$ and $\|\bar{\su}_{t-1}^s - \su^\star\|$ for $\forall t \in [n]$ and $\forall s \in \{0, 1, \ldots, S-1\}$. This bound is universal and depends on $\eta$, $G$, $n$, $S$, $D_\su$ and $\epsilon$ 
\begin{lemma}\label{Lemma:SPPRR-iterate-bound}
Let $\su^\star=(\lambda^\star, \beta^\star, \gamma^\star) \in \Lambda \times \Gamma$ be an optimal saddle point of the smooth minimax optimization model in Eq.~\eqref{prob:WDRSL-min-max}. Under Assumption~\ref{Assumption:main} and let $F_i$ be bounded over $\Lambda \times \Gamma$ for all $i \in [n]$. Suppose that the initial vector $\su_0^0$ satisfies that $\|\su_0^0 - \su^\star\| \leq D_\su$ for parameters $D_\su \geq 0$ and $\epsilon \in (0, 1)$ denotes the desired tolerance. Given any fixed random permutation $\sigma_t^s$, the following error criterion holds true,   
\begin{equation*}
\|\su_t^s - \bar{\su}_{t-1}^s\| = \|\su_{M, t-1}^s - \bar{\su}_{t-1}^s\| \ \leq \ \frac{D_\su}{10(nS)^{3/2}}. 
\end{equation*}
Then, the following statement holds true,
\begin{equation*}
\max\left\{\|\bar{\su}_{t-1}^s - \su^\star\|, \|\su_t^s - \su^\star\|\right\} \ \leq \ 2D_\su + \eta GnS.  
\end{equation*}
\end{lemma}
\begin{proof}
By the definition, we have $\|\su_0^0 - \su^\star\| \leq D_\su$. Then we consider $\|\su_1^0 - \su^\star\|$. Indeed, regardless of any random permutation, we derive from the error criterion that 
\begin{equation*}
\|\su_1^0 - \su^\star\| \ \leq \ \|\su_1^0 - \bar{\su}_0^0\| + \|\bar{\su}_0^0 - \su^\star\| \ \leq \ \|\bar{\su}_0^0 - \su^\star\| + \frac{D_\su}{10(nS)^{3/2}}. 
\end{equation*}
By the definition of $\bar{\su}_0^0$ and the boundness of $F_i$ for $\forall i \in [n]$, we have
\begin{equation*}
\|\bar{\su}_0^0 - \su^\star\| = \|\PCal_{\Lambda \times \Gamma}(\su_0^0 - \eta F_{\sigma_0^0}(\bar{\su}_0^0)) - \su^\star\| \ \leq \ \|\su_0^0 - \su^\star\| + \eta G.
\end{equation*}
Putting these pieces together yields that 
\begin{equation*}
\|\su_1^0 - \su^\star\| \ \leq \ \|\su_0^0 - \su^\star\| + \eta G + \frac{D_\su}{10(nS)^{3/2}} \ \leq \ D_\su + \eta G + \frac{D_\su}{10(nS)^{3/2}}.    
\end{equation*}
Repeating the above argument, we have
\begin{eqnarray*}
\|\su_t^s - \su^\star\| & \leq & \|\su_{t-1}^s - \su^\star\| + \eta G + \frac{D_\su}{10(nS)^{3/2}} \ \leq \ \|\su_0^s - \su^\star\| + t\left(\eta G + \frac{D_\su}{10(nS)^{3/2}}\right) \\
& \leq & \|\su_0^0 - \su^\star\| + (ns + t)\left(\eta G + \frac{D_\su}{10(nS)^{3/2}}\right),   
\end{eqnarray*}
and 
\begin{equation*}
\|\bar{\su}_{t-1}^s - \su^\star\| \ \leq \ \|\su_{t-1}^s - \su^\star\| + \eta G \ \leq \ \|\su_0^0 - \su^\star\| + (ns + t)\left(\eta G + \frac{D_\su}{10(nS)^{3/2}}\right). 
\end{equation*}
Since $1 \leq t \leq n$ and $0 \leq s \leq S-1$, we have $ns + t \leq nS$. This together with the previous two inequalities, $\|\su_0^0 - \su^\star\| \leq D_\su$ and $nS \geq 1$ yields the desired result. 
\end{proof}

\subsection{Proof of Theorem~\ref{Theorem:SPPRR}}
Recall that the iterate $\bar{\su}_t^s$ is defined as the fixed point of the operator $T(\su) := \PCal_{\Lambda \times \Gamma}(\su_t^s - \eta F_{\sigma_t^s}(\su))$. By using the triangle inequality, we have
\begin{equation*}
\|\su_t^s - \su^\star\|^2 \ \leq \ \|\su_t^s - \bar{\su}_{t-1}^s\|^2 + 2\|\su_t^s - \bar{\su}_{t-1}^s\|\|\bar{\su}_{t-1}^s - \su^\star\| + \|\bar{\su}_{t-1}^s - \su^\star\|^2.  
\end{equation*}
We claim that the following error criterion holds true,   
\begin{equation}\label{Def:SPPRR-error-main}
\|\su_t^s - \bar{\su}_{t-1}^s\| = \|\su_{M, t-1}^s - \bar{\su}_{t-1}^s\| \ \leq \ \frac{D_\su}{10(nS)^{3/2}}. 
\end{equation}
By Lemma~\ref{Lemma:SPPRR-iterate-bound}, we have $\|\bar{\su}_{t-1}^s - \su^\star\| \leq 2D_\su + \eta GnS$. Putting these pieces together yields that 
\begin{equation}\label{inequality-SPPRR-main-first}
\|\su_t^s - \su^\star\|^2 \ \leq \ \|\bar{\su}_{t-1}^s - \su^\star\|^2 + \frac{2D_\su^2 + \eta GD_\su nS}{5(nS)^{3/2}} + \frac{D_\su^2}{100(nS)^3}.  
\end{equation}
By definition of $\bar{\su}_t^s$, we have
\begin{equation*}
\|\bar{\su}_{t-1}^s - \su^\star\|^2 \ \leq \ \|\su_{t-1}^s - \su^\star\|^2 - 2\eta\langle F_{\sigma^{s}_{t-1}}(\bar{\su}_{t-1}^s), \bar{\su}_{t-1}^s - \su^\star\rangle. \end{equation*}
By the definition of $F_i$, we have
\begin{equation}\label{inequality-SPPRR-main-second}
\|\bar{\su}_{t-1}^s - \su^\star\|^2 \ \leq \ \|\su_{t-1}^s - \su^\star\|^2 - 2\eta\left(L_{\sigma_{t-1}^s}(\bar{\lambda}_{t-1}^s, \bar{\beta}_{t-1}^s, \gamma^\star) - L_{\sigma_{t-1}^s}(\lambda^\star, \beta^\star, \bar{\gamma}_{t-1}^s)\right). 
\end{equation}
Since $F_i$ is bounded by $G$ for all $i \in [n]$, the function $L$ is $G$-Lipschitz over $\Lambda \times \Gamma$. Thus, we have
\begin{eqnarray}\label{inequality-SPPRR-main-third}
& & \left|(L_{\sigma_{t-1}^s}(\bar{\lambda}_{t-1}^s, \bar{\beta}_{t-1}^s, \gamma^\star) - L_{\sigma_{t-1}^s}(\lambda^\star, \beta^\star, \bar{\gamma}_{t-1}^s)) - (L_{\sigma_{t-1}^s}(\lambda_t^s, \beta_t^s, \gamma^\star) - L_{\sigma_{t-1}^s}(\lambda^\star, \beta^\star, \gamma_t^s))\right| \\ 
& \leq & G\|\su_t^s - \bar{\su}_{t-1}^s\| \ \overset{\textnormal{Eq.~\eqref{Def:SPPRR-error-main}}}{\leq} \frac{GD_\su}{10(nS)^{3/2}}. \nonumber
\end{eqnarray} 
Combining Eq.~\eqref{inequality-SPPRR-main-first}, Eq.~\eqref{inequality-SPPRR-main-second} and Eq.~\eqref{inequality-SPPRR-main-third} yields that 
\begin{eqnarray*}
\|\su_t^s - \su^\star\|^2 & \leq & \|\su_{t-1}^s - \su^\star\|^2 - 2\eta\left(L_{\sigma_{t-1}^s}(\lambda_t^s, \beta_t^s, \gamma^\star) - L_{\sigma_{t-1}^s}(\lambda^\star, \beta^\star, \gamma_t^s) - \frac{GD_\su}{10(nS)^{3/2}}\right) \\
& & + \frac{2D_\su^2 + \eta GD_\su nS}{5(nS)^{3/2}} + \frac{D_\su^2}{100(nS)^3}.  
\end{eqnarray*}
Taking the expectation of both sides and using Lemma~\ref{Lemma:SPPRR-coupling}, we have
\begin{eqnarray*}
\EE[\|\su_t^s - \su^\star\|^2] & \leq & \EE[\|\su_{t-1}^s - \su^\star\|^2] - 2\eta\left(\EE[\Delta(\su_t^s)] - 2\eta G^2 - \frac{GD_\su}{5\sqrt{nS}} - \frac{GD_\su}{10(nS)^{3/2}}\right) \\
& & + \frac{2D_\su^2 + \eta GD_\su nS}{5(nS)^{3/2}} + \frac{D_\su^2}{100(nS)^3}.  
\end{eqnarray*}
Summing the above inequality up over $t = 1, 2, \ldots, n$ and $s = 0, 1, \ldots, S-1$ and using the fact that $nS \geq 1$, we conclude that 
\begin{eqnarray*}
\frac{1}{nS}\sum_{s=0}^{S-1}\sum_{t=1}^n \EE[\Delta(\su_t^s)] & \leq & \frac{\|\su_0^0 - \su^\star\|^2}{2\eta nS} + 2\eta G^2 + \frac{GD_\su}{5\sqrt{nS}} + \frac{GD_\su}{10(nS)^{3/2}} + \frac{D_\su^2}{5\eta(nS)^{3/2}} + \frac{GD_\su}{10\sqrt{nS}} + \frac{D_\su^2}{200\eta(nS)^3} \\
& \leq & \frac{D_\su^2}{2\eta nS} + 2\eta G^2 + \frac{2GD_\su}{5\sqrt{nS}} + \frac{D_\su^2}{4\eta(nS)^{3/2}} \\
& \leq & \frac{3D_\su^2}{4\eta nS} + 2\eta G^2 + \frac{2GD_\su}{5\sqrt{nS}}.  
\end{eqnarray*}
By the convexity of $\Delta(\cdot)$ and the definition of $\tilde{\su}^S$, we have
\begin{equation*}
\EE[\Delta(\tilde{\su}^S)] \ \leq \ \frac{1}{nS}\sum_{s=0}^{S-1}\sum_{t=1}^n \EE[\Delta(\su_t^s)]. 
\end{equation*}
By the definition of $\eta$, we have
\begin{eqnarray*}
\frac{3D_\su^2}{4\eta nS} + 2\eta G^2 & \leq & \frac{3D_\su^2}{4nS}\max\left\{2(\ell+\kappa+1), \frac{4G^2}{\epsilon}\right\} + \frac{\epsilon}{2} \\ 
& \leq & \frac{3D_\su^2}{4nS}\left(2(\ell+\kappa+1) + \frac{4G^2}{\epsilon}\right) + \frac{\epsilon}{2} \\ 
& = & \frac{2(\ell+\kappa+1)D_\su^2}{nS} + \frac{3G^2D_\su^2}{\epsilon nS} + \frac{\epsilon}{2}. 
\end{eqnarray*}
Putting these pieces together yields that 
\begin{equation*}
\EE[\Delta(\tilde{\su}^S)] \ \leq \ \frac{2(\ell+\kappa+1)D_\su^2}{nS} + + \frac{3G^2D_\su^2}{\epsilon nS} + \frac{\epsilon}{2} + \frac{2GD_\su}{5\sqrt{nS}}. 
\end{equation*}
By the definition of $S$, we have $\EE[\Delta(\tilde{\su}^S)] \leq \epsilon$. Then it remains to prove the claim for Eq.~\eqref{Def:SPPRR-error-main}. 

\paragraph{Proof of Eq.~\eqref{Def:SPPRR-error-main}.} From the scheme of the SPPRR algorithm, we have $\su_{0, t-1}^s = \su_{t-1}^s$ and $\bar{\su}_{t-1}^s = \PCal_{\Lambda \times \Gamma}(\su_{t-1}^s - \eta F_{\sigma_{t-1}^s}(\bar{\su}_{t-1}^s))$. Thus, we have
\begin{equation*}
\|\su_{0, t-1}^s - \bar{\su}_{t-1}^s\| \ \leq \ \|\su_{t-1}^s - (\su_{t-1}^s - \eta F_{\sigma_{t-1}^s}(\bar{\su}_{t-1}^s))\| \ \leq \ \eta G \ \leq \ \frac{\epsilon}{4G}.
\end{equation*}
Since $0 < \eta < \frac{1}{2(\ell+\kappa+1)}$ and the operator $F_i$ is $(\ell+\kappa+1)$-Lipschitz, we have
\begin{equation*}
\|\PCal_{\Lambda \times \Gamma}(\su_{t-1}^s - \eta F_{\sigma_{t-1}^s}(\su)) - \PCal_{\Lambda \times \Gamma}(\su_{t-1}^s - \eta F_{\sigma_{t-1}^s}(\sv))\| \ \leq \ \eta\|F_{\sigma_{t-1}^s}(\su) - F_{\sigma_{t-1}^s}(\sv)\| \ \leq \ \frac{1}{2}\|\su - \sv\|. 
\end{equation*} 
This implies that the operator $T(\cdot) = \PCal_{\Lambda \times \Gamma}(\su_{t-1}^s - \eta F_{\sigma_{t-1}^s}(\cdot))$ is a $\frac{1}{2}$-contraction. Then, the fixed-point iteration scheme implies that 
\begin{equation*}
\|\su_{M, t-1}^s - \bar{\su}_{t-1}^s\| \ \leq \ \frac{\|\su_{0, t-1}^s - \bar{\su}_{t-1}^s\|}{2^M} \ \leq \ \frac{\epsilon}{2^{M+2}G}. 
\end{equation*}
By the definition of $M$, we have $2^M \geq 10nS$. This together with the definition of $S$ yields that 
\begin{equation*}
\|\su_{M, t-1}^s - \bar{\su}_{t-1}^s\| \ \leq \ \frac{\epsilon}{40GnS} \ \leq \ \frac{D_\su}{10(nS)^{3/2}},    
\end{equation*}
which implies Eq.~\eqref{Def:SPPRR-error-main}. 

\section{Additional Experimental Results}\label{sec:appendix-exp}
In this section, we first provide details on our experimental setup and more experimental results on SEVR and SPPRR on synthetic and LIBSVM datasets. We then present a performance comparison of ERM, DRSL with $f$-divergences, and DRSL with Wasserstein metrics on  the standard datasets from LIBSVM Binary\footnote{\url{https://www.csie.ntu.edu.tw/~cjlin/libsvmtools/datasets/binary.html}.}~\citep{Chang-2011-Libsvm}.

\subsection{Datasets Setup}
For the synthetic data experiments, we first generate $\beta^\star \in \br^d$ where $\beta^\star \sim \NCal(0, I_d)$. We then generate inputs $\widehat{\x}_i \sim \NCal(0, I_d)$. 
The label $\widehat{y}_i$ is defined as $\widehat{y}_i = \text{sign}(\langle \widehat{\x}_i, \beta^{\star} \rangle + {\epsilon}_i)$, where ${\epsilon}_i \sim  \NCal(0, 0.2 \cdot I_d)$. We consider $n \in \{5\times 10^3, 1 \times 10^5, 5 \times 10^4\}$ and $d=100$ for the synthetic datasets. For real data, we use the standard datasets from LIBSVM Binary~\citep{Chang-2011-Libsvm}.

\subsection{Parameter Setup} 
We set the parameter $\{\kappa, \delta\}$ for distributionally robust logistic regression as $\kappa = 1$, $\delta = 0.1$ for our experiments. For SSG, SGDA, and ExtraSGDA, we use a diminishing step size: $\eta_{t} \propto 1/t^{1/2}$ following the guidelines of their theoretical results. We use a constant step size for SEVR and SPPRR following our theoretical results. We compare the performance of different methods in terms of the number of data passes. The default mini-batch size for SEVR is $B=32$, and the default number of fixed-point iterations for SPPRR is $M=2$.

\subsection{Convergence Results of SEVR and SPPRR} 
(1). In Figure~\ref{fig:exp-comparison-with-min-baseline-appendix}, We compare our proposed algorithms with existing algorithms for solving convex reformulations of WDRSL on two synthetic datasets ($n=5,000$ and $n=10,000$) and the w8a dataset from LIBSVM. (2). In Figure~\ref{fig:exp-comparison-with-minmax-baseline-appendix}, We compare our proposed algorithms with existing gradient-based algorithms for min-max optimization  on two synthetic datasets ($n=5,000$ and $n=10,000$) and the w8a dataset from LIBSVM. (3). In Figure~\ref{fig:exp-comparison-batchsize-fix-point-iteration-M-appendix}, We study the effect of batch size in SEVR and the effect of number of fixed-point iterations in SPPRR  on the synthetic dataset ($n=10,000$) and the w8a dataset from LIBSVM.. 

\begin{figure*}[!ht]
\centering
\includegraphics[width=0.99\textwidth]{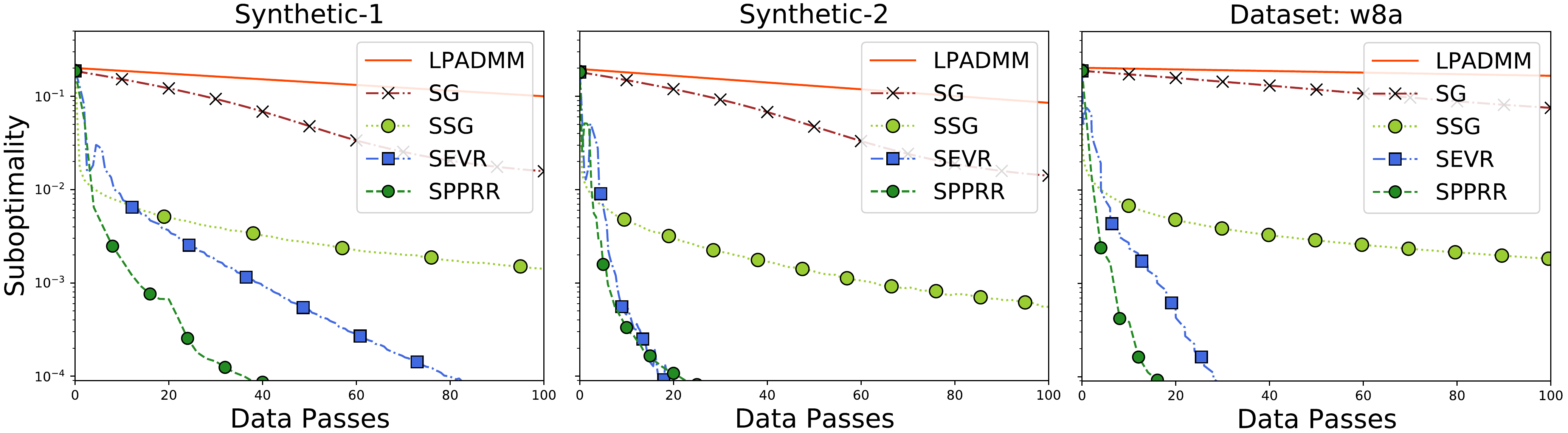}
\caption{\footnotesize{Comparison of SEVR and SPPRR with (minimization) baseline methods for solving distributionally robust logistic regression on synthetic datasets and the w8a dataset from LIBSVM. The horizontal axis represents the number of data passes, and the vertical axis represents the suboptimality $(f(\lambda^{(t)}, \beta^{(t)}) - f^{\star})$. We remark that LPADMM takes multiple data passes in every iteration as it needs to solve a nonlinear constrained problem in each iteration, which accounts for its poor performance.}}\label{fig:exp-comparison-with-min-baseline-appendix}
\end{figure*}

\begin{figure*}[!ht]
\centering
\includegraphics[width=0.99\textwidth]{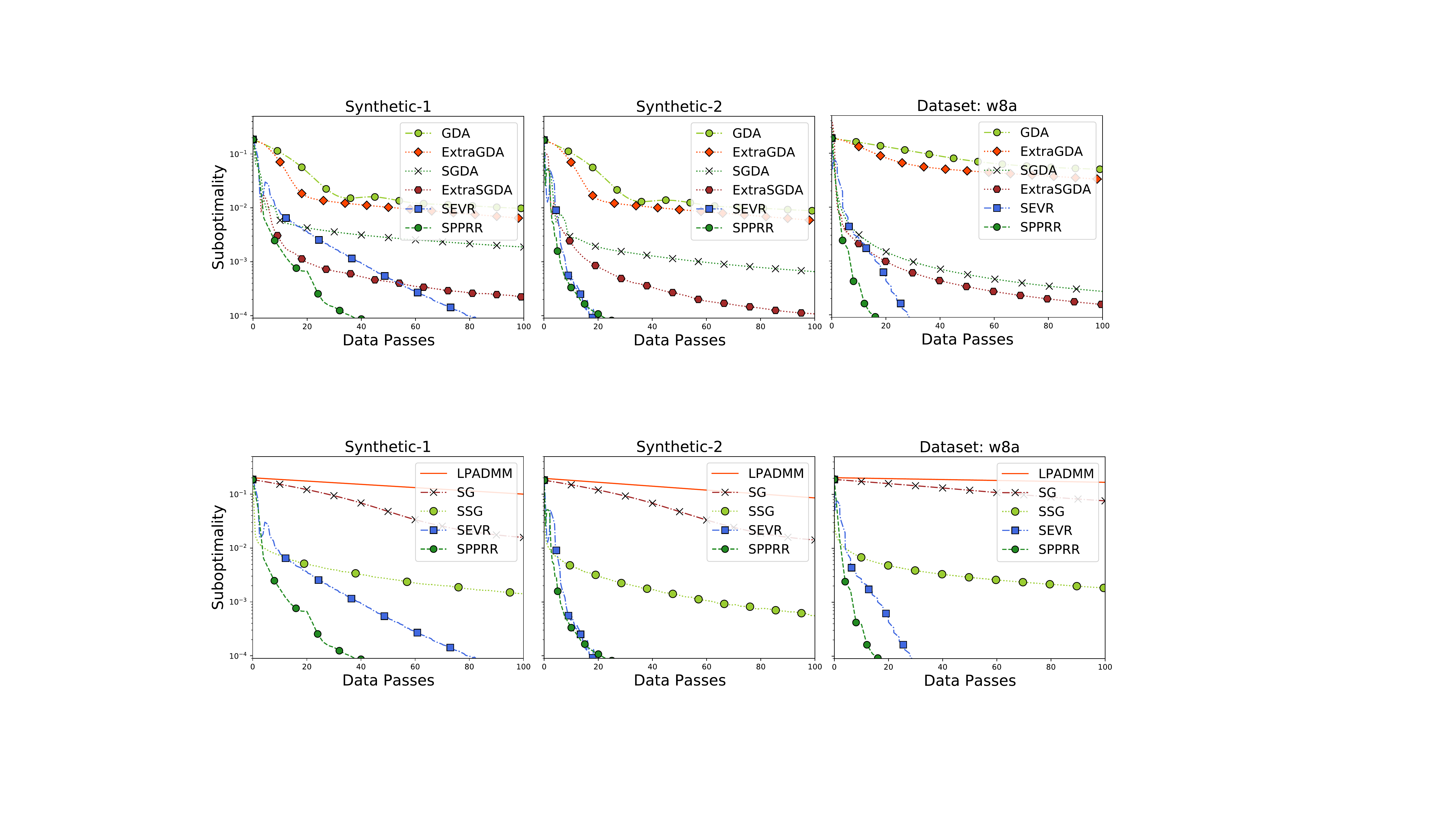}
\caption{\footnotesize{Comparison of SEVR and SPPRR with (minmax) baseline methods for solving distributionally robust logistic regression on synthetic datasets and the w8a dataset from LIBSVM. The horizontal axis represents the number of data passes, and the vertical axis represents the suboptimality $(f(\lambda^{(t)}, \beta^{(t)}) - f^{\star})$.}}
\label{fig:exp-comparison-with-minmax-baseline-appendix}
\end{figure*}

\begin{figure*}[ht!]
\centering
\subfigure[Batch size for SEVR ]{\includegraphics[width=0.75\textwidth]{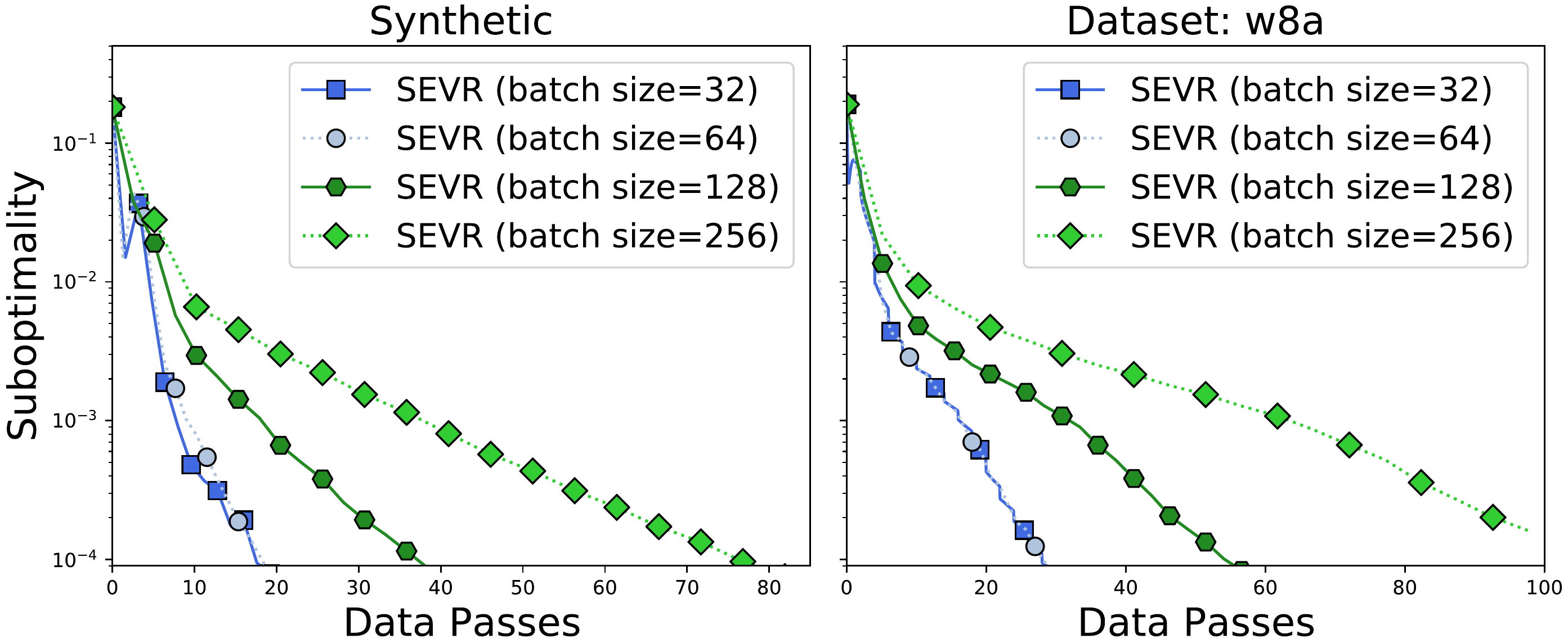}}
\subfigure[Number of fixed-point iteration for SPPRR]{\includegraphics[width=0.75\textwidth]{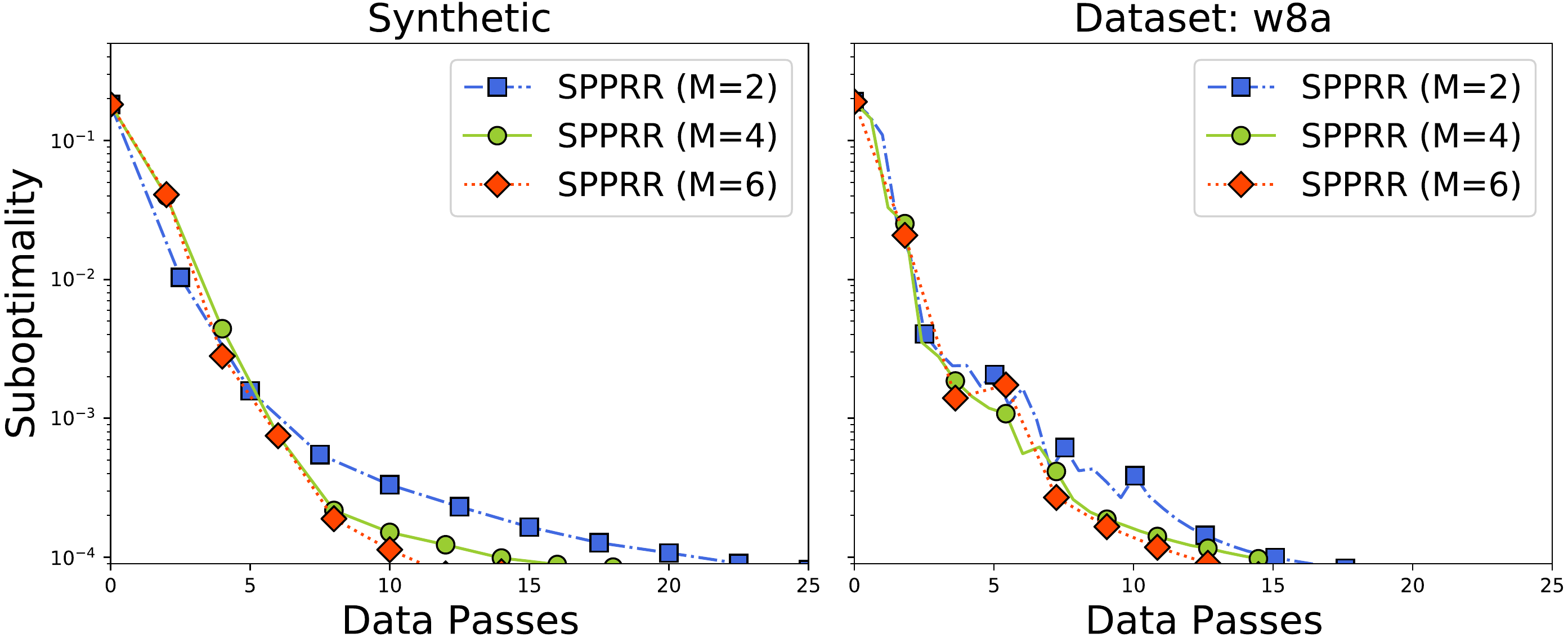}}
\caption{(\textbf{a})\,Comparing SEVR across batch sizes $B \in \{32, 64, 128, 256\}$ for solving WDRSL with logistic losses. (\textbf{b})\,Comparing SPPRR with a different number of fixed-point iterations $M \in \{2, 4, 6\}$ for solving WDRSL with logistic losses.}\label{fig:exp-comparison-batchsize-fix-point-iteration-M-appendix}
\end{figure*}

\newpage
\subsection{Performance Comparison under Perturbations}\label{sec:robust-loss-compare-appendix}
We compare the performance of three approaches in terms of robust loss during the test-time, including (1). standard ERM with logloss loss; (2). DRSL with $f$-divergences using logloss loss~\cite{Namkoong-2017-Variance}; (3). DRSL with Wasserstein distance  using logloss loss studied in this paper. We consider two types of robust loss for evaluation, including the worst-case expected loss under Wasserstein distance and robust loss under $f$-divergence. We first define the logloss function as $\ell_{\beta}(\x, y) = \log(1 + \exp(-{y}\langle {\x}, \beta\rangle))$. We evaluate the performance on test dataset $\{\x_i, y_i\}_{i=1}^{n}$. 
Then the robust loss with Wasserstein distance is defined as
\begin{equation}\label{eq:appendix-robust-loss-WDRSL}
R^{W}(\beta, \widehat{\PP}_n, \delta) = \sup_{\PP \in \BB_\delta(\widehat{\PP}_n)}  \EE^\PP\left[\ell_{\beta}(\x, y)\right],
\end{equation}
where $\widehat{\PP}_n = (1/n)\sum_{i=1}^n \delta_{(\widehat{\x}_i, \widehat{y}_i)}$, $\BB_\delta(\widehat{\PP}_n)=\{\PP \mid \WCal(\PP, \widehat{\PP}_n) \leq \delta\}$, and $\WCal(\mu, \nu) = \inf_{\pi \in \Pi(\mu, \nu)} \int_{\ZCal \times \ZCal} c(\z, \z') \pi(d\z, d\z')$ with the metric $c(\z, \z')=\|\x-\x'\| +  \kappa|y-y'|$. 

The robust loss under $f$-divergence is defined as  
\begin{equation}\label{eq:appendix-robust-loss-f-div-DRSL}
R^{f}(\beta, \widehat{Q}_n, \rho) = \sup_{Q \in D_{\phi}({Q}\parallel\widehat{Q}_n)\leq \rho/n }  \EE^{Q}\left[\ell_{\beta}(\x, y)\right],
\end{equation}
where $\widehat{Q}_n$ denotes the empirical distribution of sample $\{\x_1, \dots, \x_n\}$ and $Q$ has the same support as $\widehat{Q}_n$. We consider the $\chi^{2}$-divergence studied in \cite{Namkoong-2017-Variance}, where $D_{\phi}(P\parallel Q)$ $D_{\phi}(P\parallel Q) = \int_{\XCal} \phi(dP/dQ)dQ$ and $\phi(t) = (1/2)\cdot(t-1)^{2}$. 

\paragraph{Parameter setup.} We use similar synthetic datasets which are introduced in Section~\ref{sec:experiment}: we first generate $\beta^\star \in \br^d$ where $\beta^\star \sim \NCal(0, I_d)$. We then generate training/test inputs and labels as follows: $\widehat{\x}_i \sim \NCal(0, I_d)$, $\widehat{y}_i = \text{sign}(\langle \widehat{\x}_i, \beta^{\star} \rangle + {\epsilon}_i)$, where ${\epsilon}_i \sim  \NCal(0, 0.2 \cdot I_d)$, $i=1, \dots, n_{\text{train}}$, ${\x}_j \sim \NCal(0, I_d)$, ${y}_j = \text{sign}(\langle \widehat{\x}_i, \beta^{\star} \rangle)$,  $j=1, \dots, n_{\text{test}}$. We set $n_{\text{train}}\in\{1,000, 5,000\}$ and $n_{\text{test}}=10,000$.

For standard ERM, we optimize $\beta$ using standard logistic regression. We set $\rho=10$ for $f$-divergence DSRL and set  $\delta=0.1, \kappa=1.0$ for DSRL with Wasserstein distance. We evaluate the robust losses of three models by varying parameters $\delta$ and $\rho$, and the results are summarized in Figure~\ref{fig:exp-comparison-robust-loss-f-appendix} and Figure~\ref{fig:exp-comparison-robust-loss-W-appendix}. 

\begin{figure*}[!ht]
\centering
\subfigure[$n=1,000$]{\includegraphics[width=0.45\textwidth]{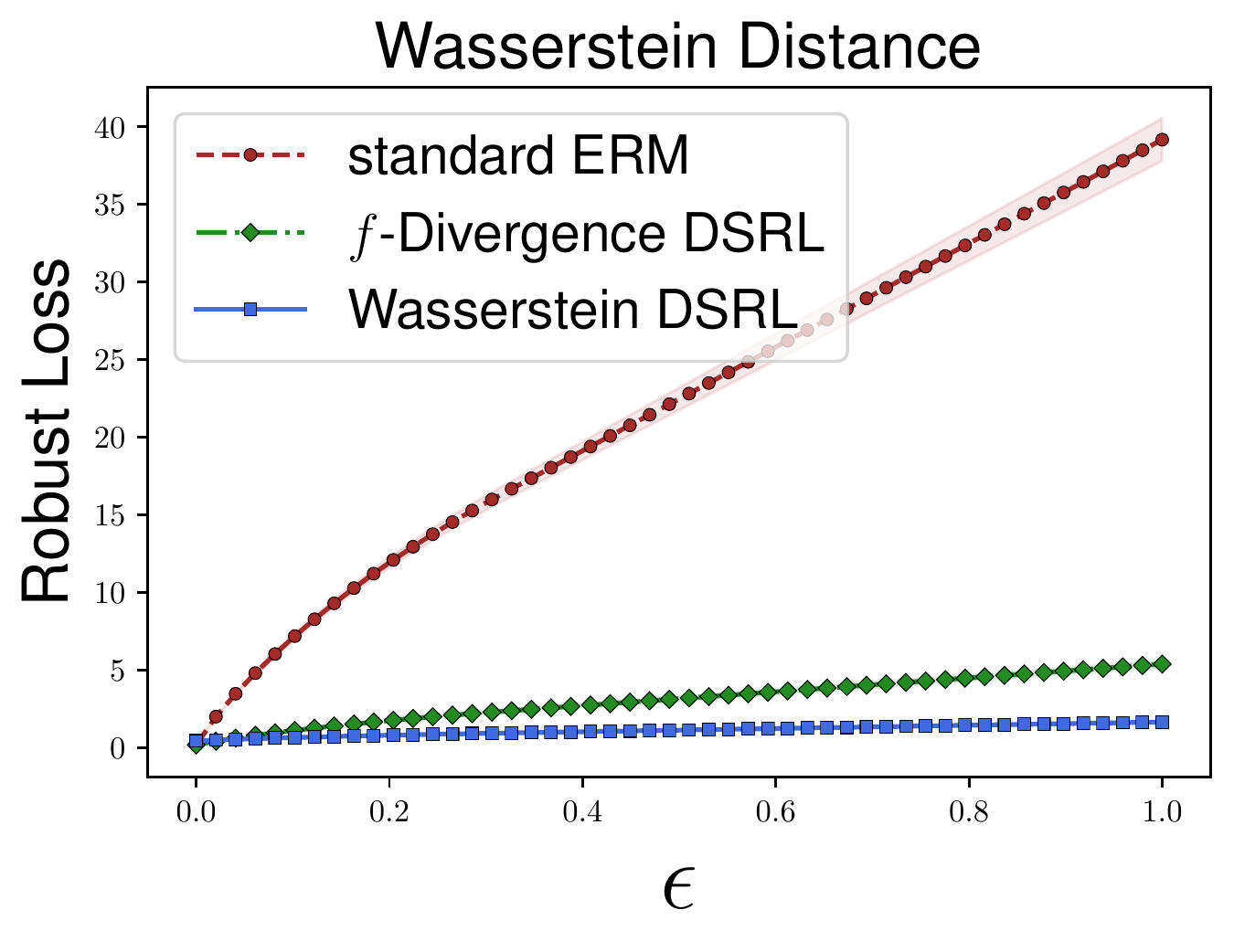}
\includegraphics[width=0.45\textwidth]{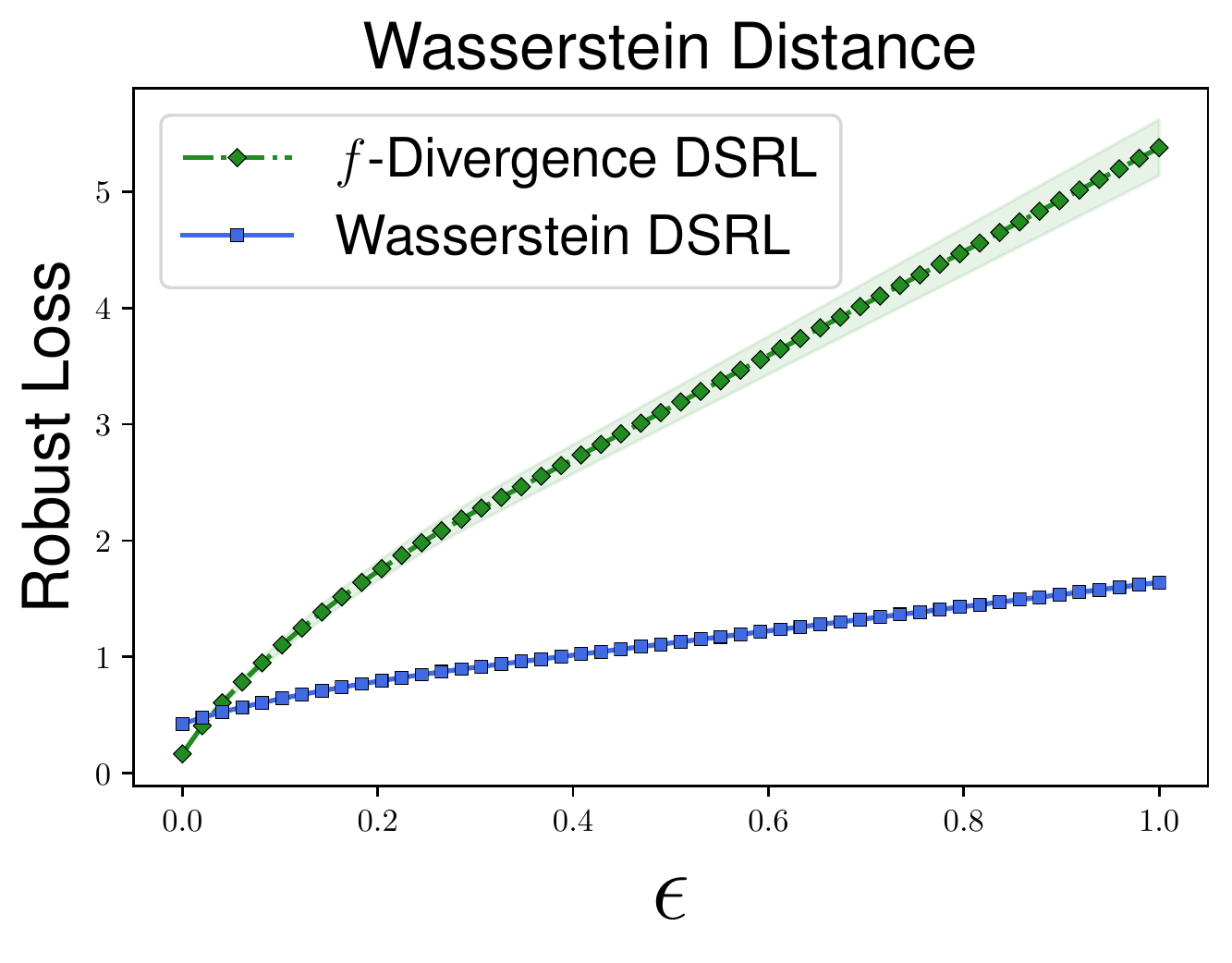}}
\subfigure[$n=2,000$]{\includegraphics[width=0.45\textwidth]{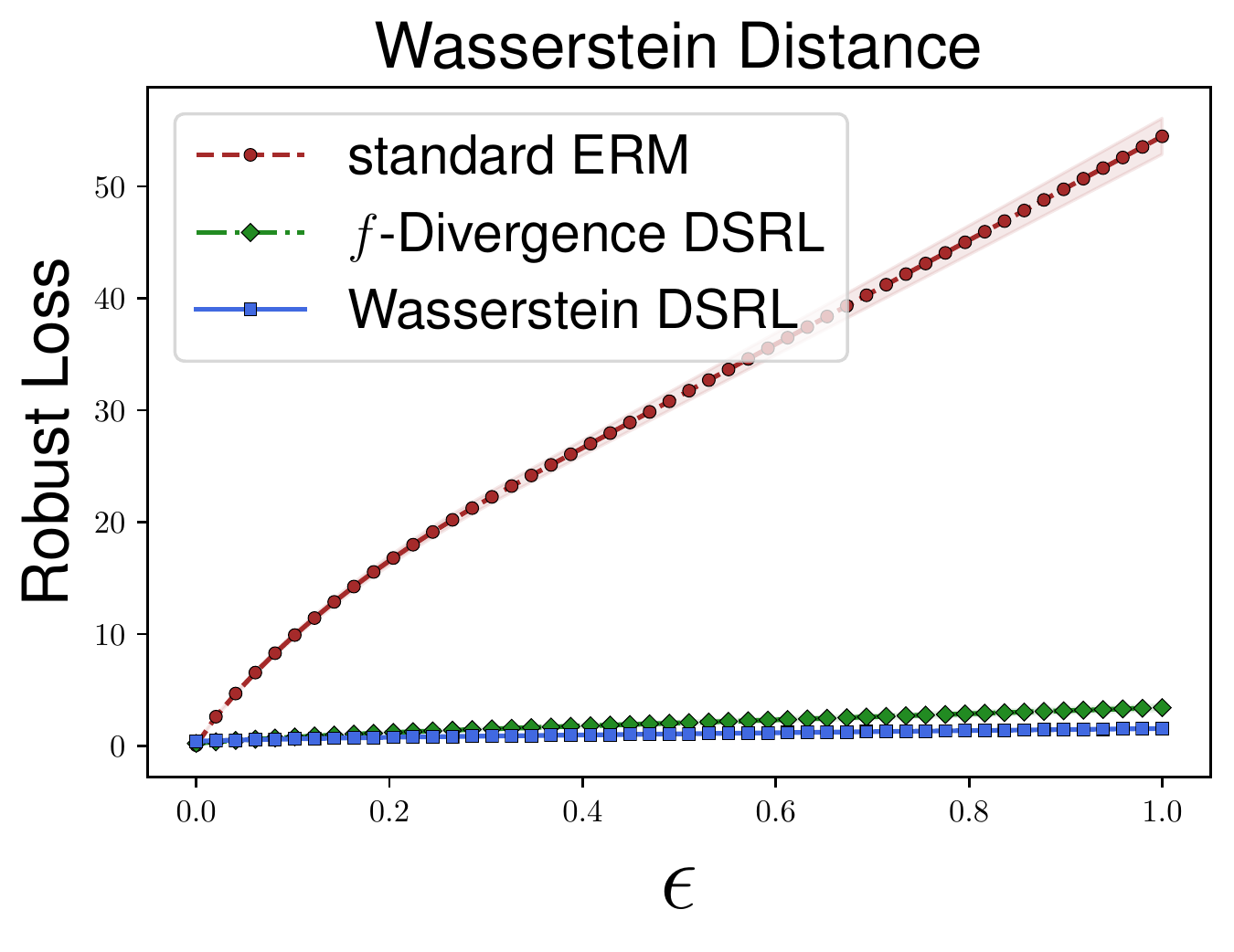}
\includegraphics[width=0.45\textwidth]{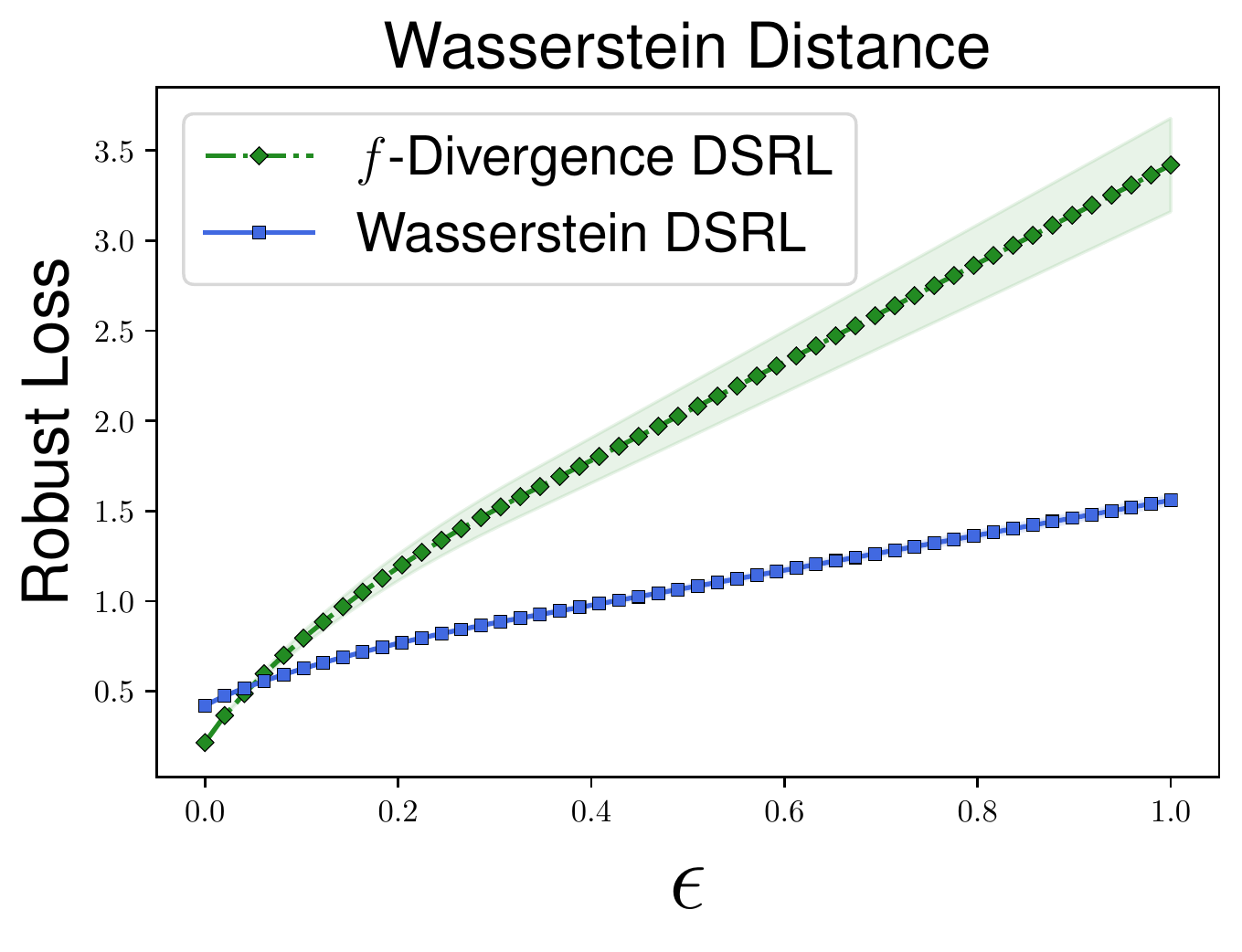}}
\caption{Comparison of the robustness of the solution to attacks in a Wasserstein ball of radius $\epsilon$, where $\epsilon$ is defined in Eq.~\eqref{eq:appendix-robust-loss-WDRSL}. (\textbf{a})\, $n_{\text{train}}=1,000$. (\textbf{b})\, $n_{\text{train}}=2,000$.}\label{fig:exp-comparison-robust-loss-f-appendix}
\end{figure*}

\begin{figure*}[ht!]
\centering
\subfigure[$n=1,000$]{\includegraphics[width=0.45\textwidth]{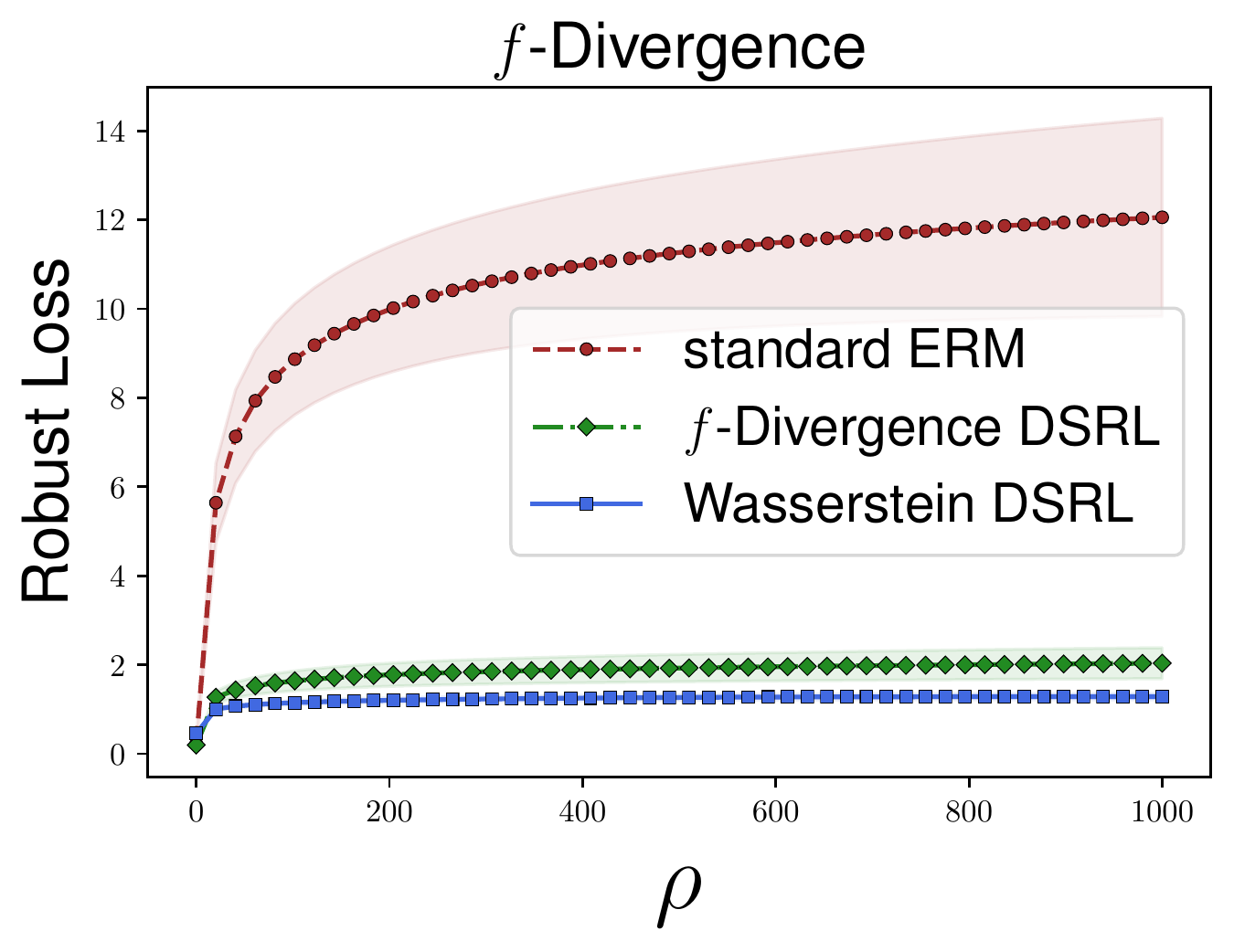}
\includegraphics[width=0.45\textwidth]{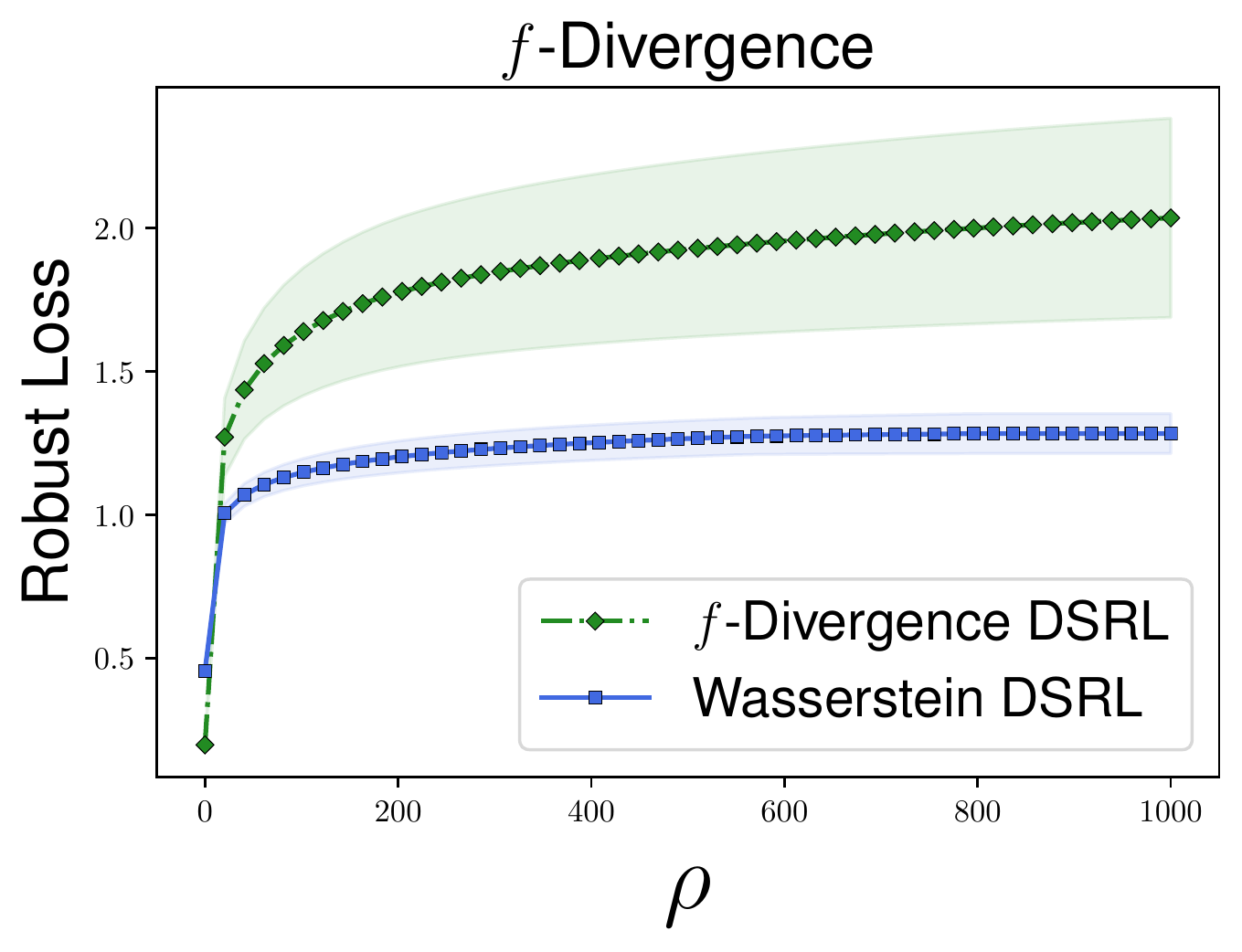}}
\subfigure[$n=2,000$]{\includegraphics[width=0.45\textwidth]{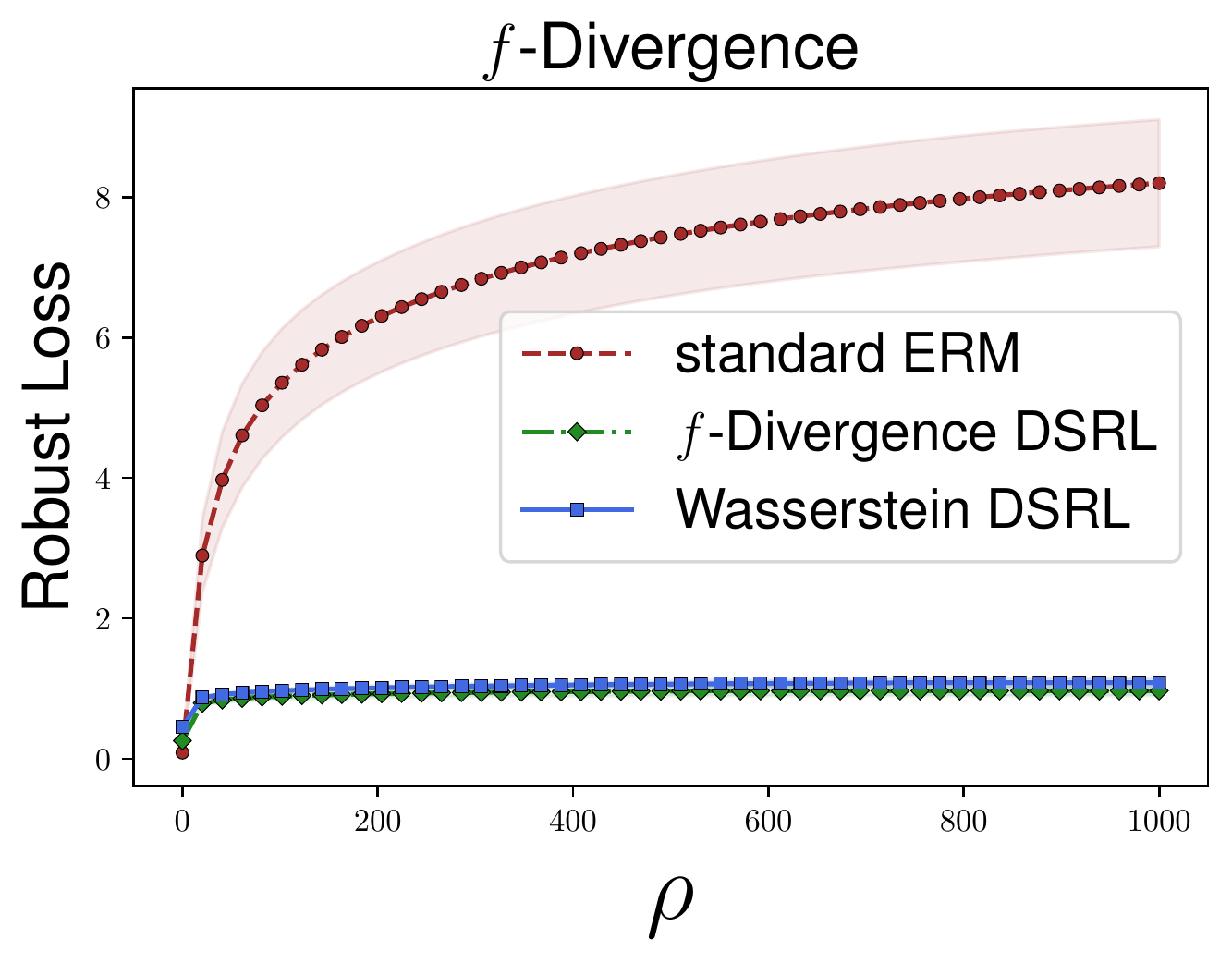}
\includegraphics[width=0.45\textwidth]{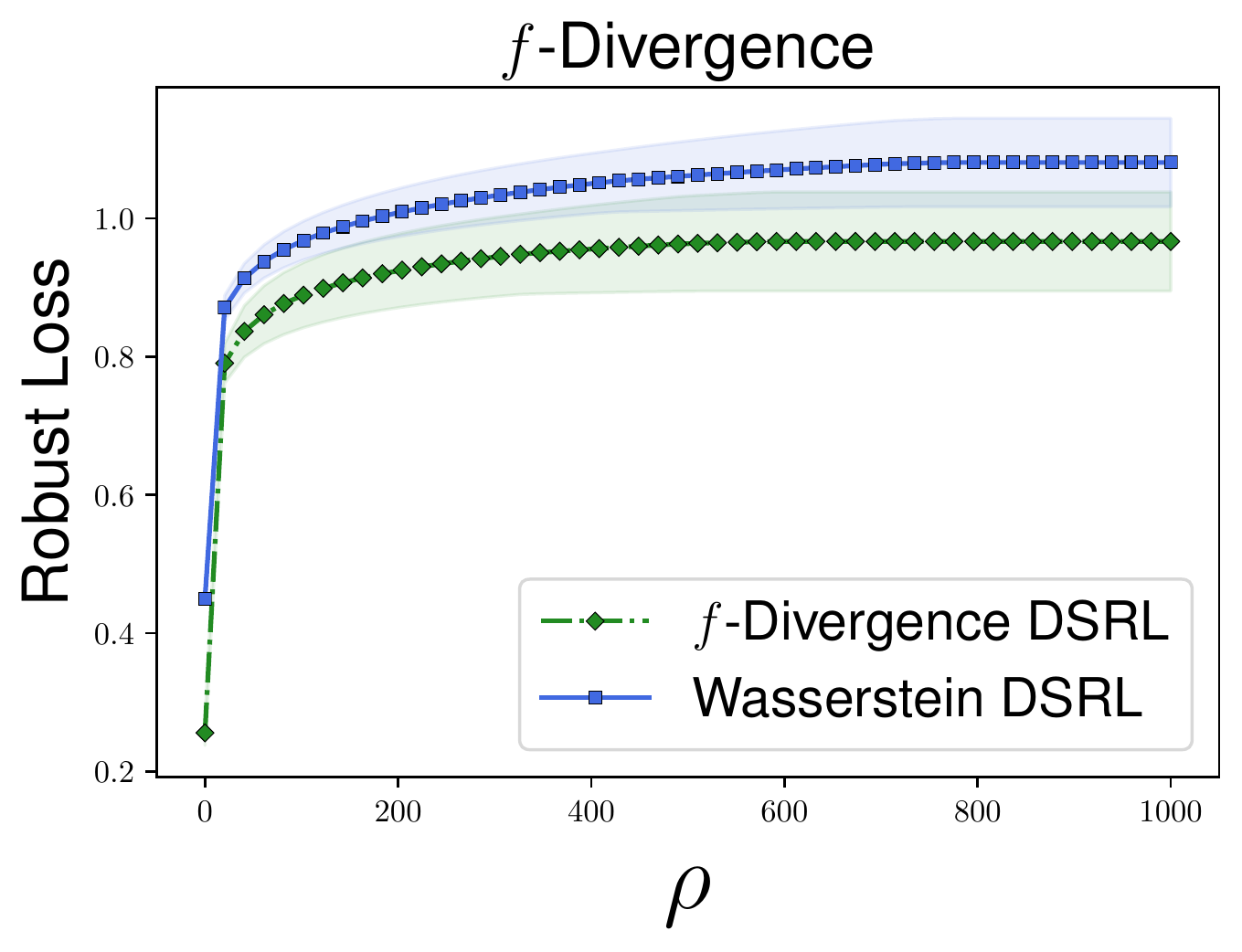}}
\caption{Comparison of the robustness of the solution to attacks in $f$-divergence, where $\rho$ (defined in Eq.~\eqref{eq:appendix-robust-loss-f-div-DRSL}) is the parameter to define the uncertainty set $D_{\phi}({P}\parallel\mathbf{1}_n/n)\leq \rho/n$. (\textbf{a})\, $n_{\text{train}}=1,000$. (\textbf{b})\, $n_{\text{train}}=2,000$.}\label{fig:exp-comparison-robust-loss-W-appendix}
\end{figure*}

For both robust loss with Wasserstein distance and $f$-divergences, WDSRL and DSRL with  $f$-divergences achieve significant better performance compared with ERM. For the Wasserstein distance robust loss, we can find that WDSRL performs much better than DSRL with  $f$-divergences when perturbation radius $\epsilon$ is large. As shown in Figure~\ref{fig:exp-comparison-robust-loss-W-appendix},  we observe that when the number of training data points is small (i.e., $n=1,000$),  WDSRL achieves slightly better performance than DSRL with  $f$-divergences, and DSRL with  $f$-divergences performs slightly better than WDSRL when training size is large (i.e., $n=2,000$).

\end{document}